%% file: main.tex
\documentclass{article}

\input{header}

\usepackage{microtype}
\usepackage{graphicx}
\usepackage{subcaption}

\usepackage{booktabs} 

\usepackage{hyperref}

\usepackage[accepted]{icml2026}

\usepackage{amsmath}
\usepackage{amssymb}
\usepackage{mathtools}
\usepackage{amsthm}

\usepackage[capitalize,noabbrev]{cleveref}

\theoremstyle{plain}
\theoremstyle{definition}
\theoremstyle{remark}

\usepackage[textsize=tiny]{todonotes}

\newcommand{\DefinedAs}[0]{\mathrel{\mathop:}=}

\usepackage{xcolor}         
\definecolor{mycyan}{RGB}{0,204,204}
\definecolor{myred}{RGB}{183,65,20}
\definecolor{mygreen}{RGB}{70,180,86}

\icmltitlerunning{Unlearning in Diffusion Models: A Unified Framework with KL Divergence and Likelihood Constraints}

\begin{document}

\twocolumn[
\icmltitle{Unlearning in Diffusion Models: A Unified Framework with KL Divergence and Likelihood Constraints}



  \begin{icmlauthorlist}

    \icmlauthor{Shervin Khalafi}{sch}
    \icmlauthor{Alejandro Ribeiro}{sch}
    \icmlauthor{Dongsheng Ding}{yyy}

  \end{icmlauthorlist}

  \icmlaffiliation{sch}{University of Pennsylvania}
  \icmlaffiliation{yyy}{University of Tennessee, Knoxville}

  \icmlcorrespondingauthor{Shervin Khalafi}{shervink@seas.upenn.edu}
  \icmlcorrespondingauthor{Dongsheng Ding}{dongshed@utk.edu}

  \icmlkeywords{Diffusion models; machine unlearning; constrained learning; KL divergence; likelihood constraints; strong duality; primal–dual algorithms; concept unlearning; data unlearning; retention–unlearning tradeoff.}
  \vskip 0.3in
]



\printAffiliationsAndNotice{}  

\begin{abstract}
Unlearning in diffusion models aims to remove undesirable data or concepts while preserving the utility of pretrained models---two fundamentally conflicting objectives. We propose a principled constrained optimization framework that formulates unlearning as minimizing the deviation from a pretrained model, subject to explicit separation constraints from the unlearning distributions. Specifically, we formulate three constrained optimization problems based on reverse and forward KL divergences, and likelihood constraints. The first two generalize existing approaches for concept and data unlearning, while the third offers a novel and natural formulation for unlearning. Despite the nonconvexity of the KL constraints, we establish strong duality for all three problems, enabling us to explicitly characterize their optimal solutions as unlearning targets and develop primal–dual algorithms for each formulation. Experimental results demonstrate that our KL-constrained approach achieves superior retention-unlearning tradeoffs compared to weight-based baselines for concept and data unlearning, and that our likelihood-based approach matches unlearning effectiveness while better preserving retained concepts compared to baselines.
\end{abstract}

\input{sec_introduction_camready}

\input{sec_distspace_camready}
\input{sec_diffusion_camready}
\input{sec_experiments_camready}

\newpage
\section*{Impact Statement}

This paper aims to improve the responsible behavior of deep generative models by ensuring compliance with unlearning requirements. By establishing a unified optimization framework, our work provides practical guidance for developing more reliable training algorithms, with potential impact across unlearning applications such as language, image, speech, and audio generation.

\bibliography{dd-bib}
\bibliographystyle{icml2026}

\newpage
\appendix
\onecolumn

\input{appendices}

\end{document}

%% file: header.tex
\usepackage{amsthm}
\usepackage{amsmath}
\usepackage{amssymb}
\usepackage{graphicx}
\usepackage{mathtools}
\usepackage{enumerate}
\usepackage{enumitem}
\usepackage{footnote}
\usepackage{float}
\usepackage{xspace}
\usepackage{multirow}
\usepackage{nicefrac}
\usepackage{wrapfig}
\usepackage{framed}
\usepackage{url}
\usepackage[colorlinks=true, linkcolor=blue, citecolor=blue]{hyperref}
\usepackage{blkarray}
\PassOptionsToPackage{round}{natbib}

\usepackage{accents}

\usepackage{tikz}
\usetikzlibrary{arrows,chains,matrix,positioning,scopes}

\let\hat\widehat
\let\tilde\widetilde

\newtheorem{lemma}{Lemma}
\newtheorem{theorem}{Theorem}
\newtheorem{cor}{Corollary}
\newtheorem{remark}{Remark}

\newtheorem{assumption}{Assumption}

\DeclareMathOperator*{\minimize}{minimize}
\DeclareMathOperator*{\maximize}{maximize}
\DeclareMathOperator*{\subject}{subject~to}

\newcommand{\one}{\boldsymbol{1}}

\DeclareMathOperator*{\argmin}{argmin}
\DeclareMathOperator*{\argmax}{argmax}

\newcommand{\norm}[1]{\left\|{#1}\right\|}

\DeclareFontFamily{OMX}{MnSymbolE}{}
\DeclareFontShape{OMX}{MnSymbolE}{m}{n}{
    <-6>  MnSymbolE5
   <6-7>  MnSymbolE6
   <7-8>  MnSymbolE7
   <8-9>  MnSymbolE8
   <9-10> MnSymbolE9
  <10-12> MnSymbolE10
  <12->   MnSymbolE12}{}
\DeclareSymbolFont{mnlargesymbols}{OMX}{MnSymbolE}{m}{n}
\SetSymbolFont{mnlargesymbols}{bold}{OMX}{MnSymbolE}{b}{n}
\DeclareMathDelimiter{\llangle}{\mathopen}{mnlargesymbols}{'164}{mnlargesymbols}{'164}
\DeclareMathDelimiter{\rrangle}{\mathclose}{mnlargesymbols}{'171}{mnlargesymbols}{'171}

\usepackage{prettyref}

\newcommand{\savehyperref}[2]{\texorpdfstring{\hyperref[#1]{#2}}{#2}}
\newrefformat{eq}{\savehyperref{#1}{\textup{(\ref*{#1})}}}
\newrefformat{eqn}{\savehyperref{#1}{Equation~\eqref{#1}}}
\newrefformat{lem}{\savehyperref{#1}{Lemma~\ref*{#1}}}
\newrefformat{def}{\savehyperref{#1}{Definition~\ref*{#1}}}
\newrefformat{line}{\savehyperref{#1}{line~\ref*{#1}}}
\newrefformat{thm}{\savehyperref{#1}{Theorem~\ref*{#1}}}
\newrefformat{cor}{\savehyperref{#1}{Corollary~\ref*{#1}}}
\newrefformat{sec}{\savehyperref{#1}{Section~\ref*{#1}}}
\newrefformat{app}{\savehyperref{#1}{Appendix~\ref*{#1}}}
\newrefformat{assum}{\savehyperref{#1}{Assumption~\ref*{#1}}}
\newrefformat{ex}{\savehyperref{#1}{Example~\ref*{#1}}}
\newrefformat{fig}{\savehyperref{#1}{Figure~\ref*{#1}}}
\newrefformat{alg}{\savehyperref{#1}{Algorithm~\ref*{#1}}}
\newrefformat{rem}{\savehyperref{#1}{Remark~\ref*{#1}}}
\newrefformat{conj}{\savehyperref{#1}{Conjecture~\ref*{#1}}}
\newrefformat{prop}{\savehyperref{#1}{Proposition~\ref*{#1}}}
\newrefformat{proto}{\savehyperref{#1}{Protocol~\ref*{#1}}}
\newrefformat{prob}{\savehyperref{#1}{Problem~\ref*{#1}}}
\newrefformat{claim}{\savehyperref{#1}{Claim~\ref*{#1}}}
\newrefformat{que}{\savehyperref{#1}{Question~\ref*{#1}}}

%% file: sec_introduction_camready.tex
\section{Introduction}

Generative diffusion models have emerged as an effective approach for synthesizing high-quality images that resemble training data and capture underlying concepts. However, this expressiveness also raises significant safety and ethical concerns, as such models may generate inappropriate content, including copyrighted material~\cite{carlini2023extracting} or harmful content~\cite{schuhmann2022laion}. While policy-makers have made progress in regulating the removal of undesirable data or concepts from trained models~\cite{protection2016general,goldman2020introduction}, it remains an open problem to develop effective technical methods to achieve this goal.

Recently, the classical concept of machine unlearning has been introduced in the context of diffusion models~\cite{cao2015towards}. A standard approach aims to retain the utility of pretrained models, ensuring high-quality image generation, while simultaneously preventing the generation of specific harmful data or concepts---two inherently \emph{conflicting} objectives. This principle has motivated recent empirical advances, including concept unlearning~\cite{gandikota2023erasing,feng2024controllable} and data unlearning~\cite{park2025data,alberti2025data}. Often, these works employ the simple practice of balancing two conflicting objectives by combining them with weighted sums. However, such weight-based methods are inherently ad hoc, as the weights are treated as hyperparameters, making the approach sensitive and poorly generalizable across scenarios.

In this work, we take a principled constrained optimization approach to address the conflicting objectives in unlearning. To do so, we view retaining the utility of trained models as minimizing the distance between the trained model and the pretrained model, while preventing harmful generation is formulated as additional constraints that push the trained model away from specific data or concept distributions. As a result, the tradeoff between these objectives can be systematically characterized by a Lagrangian formulation. To measure the distance between two models, we employ reverse and forward KL divergences, which broadly correspond to concept and data unlearning, respectively. We summarize our key contributions in three aspects below. 
\begin{figure*}[t]
    \centering
    \begin{subfigure}{0.23\textwidth}
        \centering
        \includegraphics[width=\linewidth]{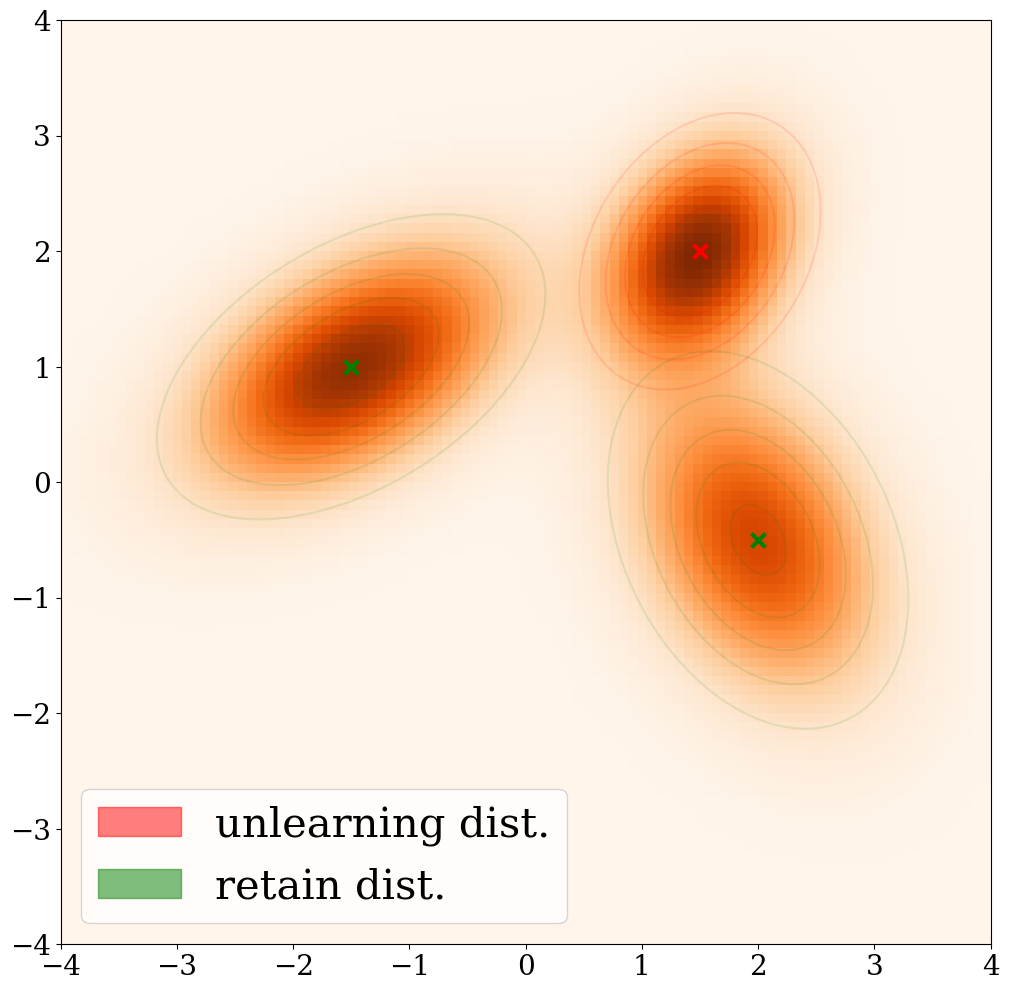}
    \end{subfigure}\hfill
    \begin{subfigure}{0.23\textwidth}
        \centering
        \includegraphics[width=\linewidth]{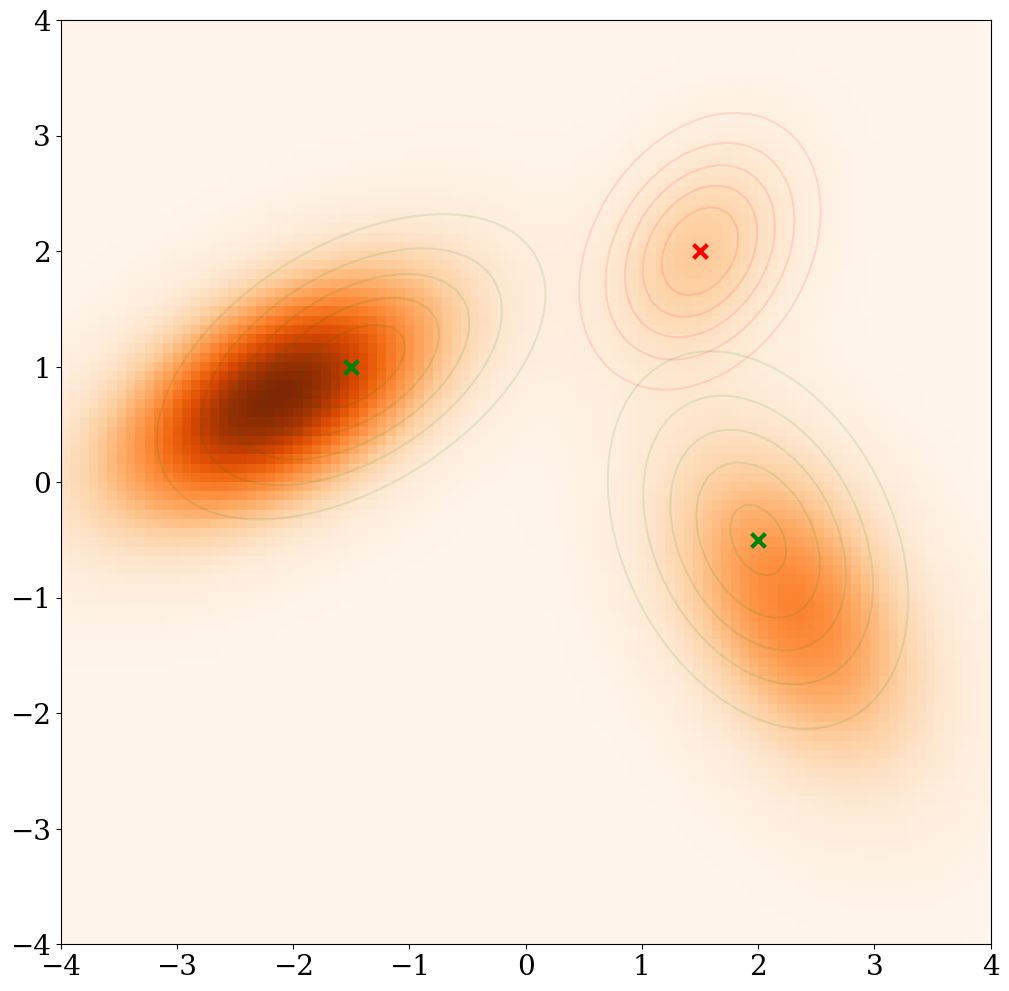}
    \end{subfigure}\hfill
    \begin{subfigure}{0.23\textwidth}
        \centering
        \includegraphics[width=\linewidth]{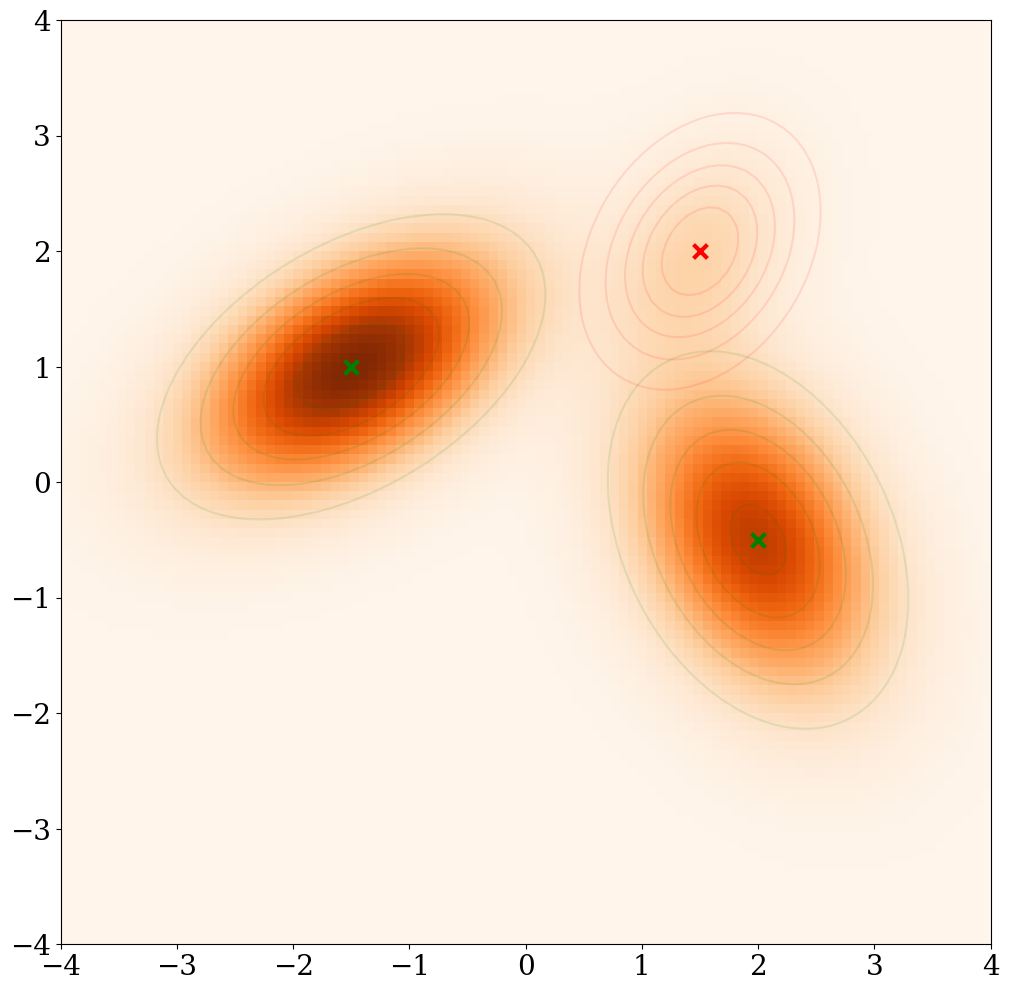}
    \end{subfigure}\hfill
    \begin{subfigure}{0.23\textwidth}
        \centering
        \includegraphics[width=\linewidth]{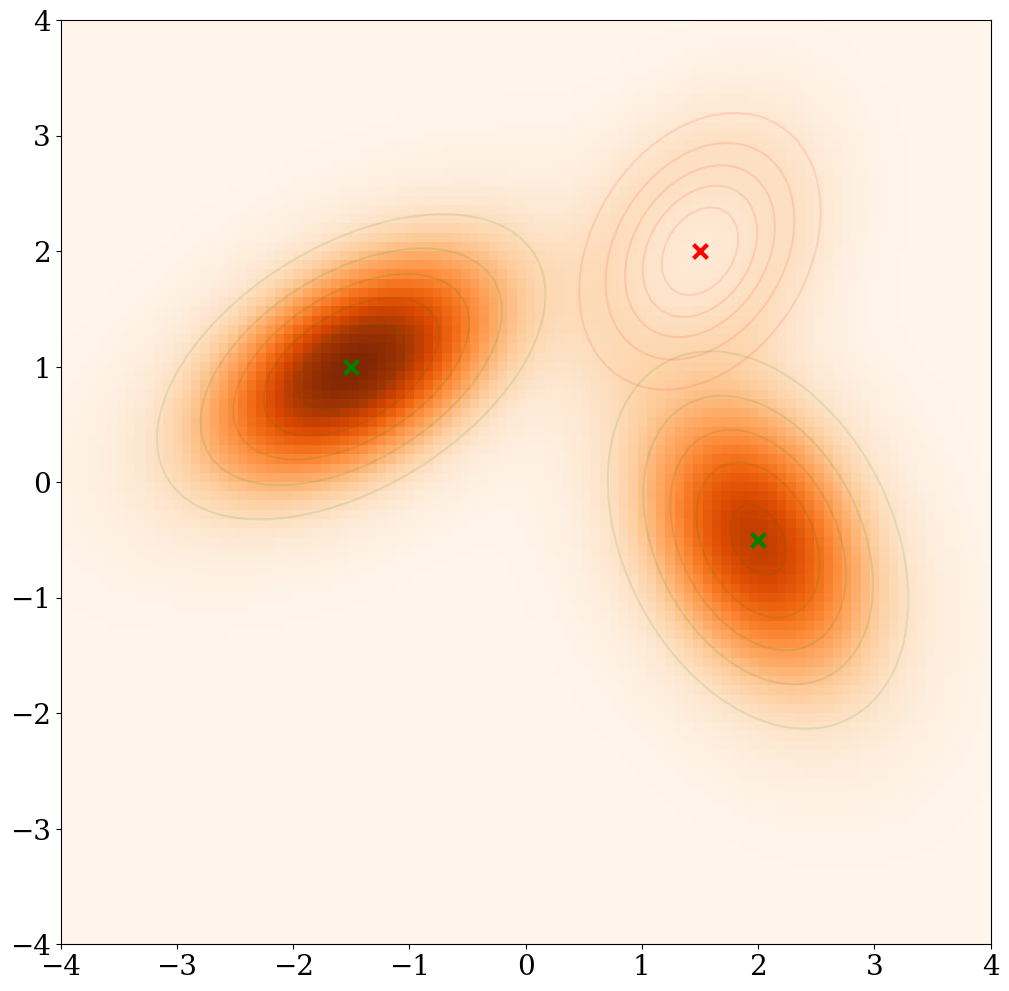}
    \end{subfigure}

    \caption{Heat maps of a three-Gaussian mixture before/after unlearning one mode (Right up). From left to right: pretrained model, reverse KL-constrained unlearning, forward KL-constrained unlearning, and likelihood-constrained unlearning.}
    \label{fig:gaussians_likelihood_vs_reverse}
\end{figure*}

\begin{enumerate}
    \item[(i)] \textbf{Problem Formulation}.
    We formulate a distribution optimization problem that minimizes the reverse KL divergence to a pretrained model, while constraining the reverse KL divergences to the distributions associated with the concepts to be unlearned to remain above user-specified thresholds. Formulation~\eqref{eq: unlearning reverse KL} corresponds to unlearning at the concept level. We then consider an analogous optimization problem based on the forward KL divergence~\eqref{eq: unlearning forward KL}, which corresponds to unlearning at the data level. Finally, we formulate a  distribution optimization problem~\eqref{eq: unlearning alignment} that constrains the expected likelihood of the concepts to be unlearned to remain below user-specified thresholds. This constrained framework extends the unlearning objectives, thereby revealing their commonalities and differences; see Section~\ref{subsec: contribution} for exact formulations and discussion. 
    
    \item[(ii)] \textbf{Theoretical Analysis}. Although the reverse and forward KL unlearning problems~\eqref{eq: unlearning reverse KL} and~\eqref{eq: unlearning forward KL} are \emph{nonconvex}, we exploit the convexity of the image of a non-atomic vector measure to prove that they are strongly dual, rather than relying on standard convex duality. In contrast, strong duality for the likelihood unlearning problem~\eqref{eq: unlearning alignment} directly follows from its convexity. In all cases, we explicitly characterize the optimal solutions, which serve as principled unlearning targets; see Figure~\ref{fig:gaussians_likelihood_vs_reverse} for visualizations and discussion.  A key implication of strong duality is that the unlearning problems can be solved equivalently in the dual domain. This enables our development of primal–dual algorithms instantiated with diffusion models for each formulation.

    \item[(iii)] \textbf{Empirical Results}.
    We demonstrate the effectiveness of our constrained optimization approach for concept and data unlearning in a text-to-image diffusion model. In concept unlearning, our method achieves the same degree of unlearning while deviating less from the pretrained model compared to equal-weights approaches. In data unlearning, assigning the same weight to every sample to be unlearned leads to excessive deviation from the pretrained model, whereas our method learns optimal weights that ensure unlearning while minimizing deviation. Furthermore, we show that our novel likelihood-based unlearning method achieves a better tradeoff between unlearning the target concept and preserving the retained concept. This allows us to perform strong unlearning without deviating significantly from the parts of the distribution we wish to retain.
\end{enumerate}

\subsection{Summary of our problem formulations}\label{subsec: contribution}

Within the optimization framework over the probability measure space \(\Delta\), we outline three principled unlearning targets and defer their exact characterizations to Section~\ref{sec: unlearning in distribution space}. Here, \(q\in\Delta\) denotes the data distribution or the distribution induced by the pretrained model to be retained, and \(q_{\mathrm u}^i\in\Delta\) denotes the \(i\)th concept or data distribution to be unlearned.

\textbf{Reverse KL-Constrained Unlearning}. We minimize the reverse KL divergence between the model $p$ and the reference $q$ to ensure that the model is \emph{close} to the reference. To remove the $m$ undesirable concepts, we impose $m$ reverse KL divergence constraints, each encouraging the model to \emph{stay away from} an undesirable concept encoded by $q_{\text{u}}^i$. To ensure this separation, we enforce a lower bound $b_i$ on each KL divergence by formulating a reverse KL-constrained distribution optimization problem,
\begin{equation}\label{eq: unlearning reverse KL}\tag{RU}
    \begin{array}{rl}
        \displaystyle\minimize_{p\,\in\,\Delta} & 
        D_{\text{KL}}(p\,\Vert\, q)
        \\[0.2cm]
        \subject &  D_{\text{KL}}(p\,\Vert\,q_{\text{u}}^i)\; \geq \; b_i \; \text{ for }\; i\in [m].
    \end{array}
\end{equation}
Problem~\eqref{eq: unlearning reverse KL} aims to minimize deviation from the pretrained model $q$---thereby maximizing model utility---while forgetting undesirable concepts. In Section~\ref{subsec: concept unlearning}, we capture a solution to Problem~\eqref{eq: unlearning reverse KL} as a ratio of distributions,
\[
p_{\text{rev}}^\star(\cdot) \;\propto\; \frac{(q(\cdot))^{\alpha_0} }{\prod_{i\,=\,1}^m ( q_{\text{u}}^i(\cdot))^{\alpha_i}}
\]
where $\{\alpha_i\}_{i\,=\,0}^m$ is a set of non-negative exponents that balance the relative importance of each model. The subscript $\text{rev}$ stands for reverse KL. The solution $p_{\text{rev}}^\star$ reflects unlearning by reducing the probability of sampling in regions where $q_{\text{u}}^i(\cdot)$ is large. This idea generalizes the concept unlearning target in concept erasing~\cite{gandikota2023erasing} and negation generation~\cite{du2023reduce}.

\textbf{Forward KL-Constrained Unlearning}. We employ the forward KL (rather than reverse KL) divergence to measure the closeness between two distributions. To do so, we obtain a variant of Problem~\eqref{eq: unlearning reverse KL} in terms of forward KL,
\begin{equation}\label{eq: unlearning forward KL}\tag{FU}
    \begin{array}{rl}
        \displaystyle\minimize_{p\,\in\,\Delta} & 
        D_{\text{KL}}(q\,\Vert\, p)
        \\[0.2cm]
        \subject &  D_{\text{KL}}(q_{\text{u}}^i\,\Vert\,p)\; \geq \; b_i \; \text{ for }\; i \in [m].
    \end{array}
\end{equation}
Since the forward KL divergence is an expectation over the data distribution, it essentially optimizes the log-likelihood of the data under the model---often approximated by the evidence lower bound in diffusion models~\cite{lai2025principles}. Therefore, Problem~\eqref{eq: unlearning forward KL} seeks to maximize the likelihood of the data to be retained while reducing the likelihood of the data to be forgotten. In Section~\ref{subsec: data unlearning}, we characterize a solution to Problem~\eqref{eq: unlearning forward KL} as a difference of distributions,
\[
    p_{\text{fw}}^\star(\cdot) 
    \;\propto\;
    q(x) 
    \, - \,
    \sum_{i\,=\,1}^m \alpha_i \, q_{\text{u}}^i(\cdot)
\]
where $\{\alpha_i\}_{i\,=\,1}^m$ is a set of non-negative weights that balance the relative importance of each distribution. The subscript fw stands for forward KL. The log-likelihood reduction of forgetting data is often referred to as data unlearning in the literature~\cite{feng2024controllable,alberti2025data}. Formulation~\eqref{eq: unlearning forward KL} is distinguished by its weaker assumption: it requires only sample access to the unlearning distributions $q_{\text{u}}^i$, whereas the reverse KL and likelihood formulations assume access to the associated diffusion models.

\textbf{Likelihood-Constrained Unlearning}. In addition to the nonconvex reverse KL divergence constraints in Problem~\eqref{eq: unlearning reverse KL}, we propose an alternative formulation that constrains the likelihood of the unlearning data under the model,
\begin{equation}\label{eq: unlearning alignment}\tag{LU}
    \begin{array}{rl}
        \displaystyle\minimize_{p\,\in\,\Delta} & 
        D_{\text{KL}}(p\,\Vert\, q)
        \\[0.2cm]
        \subject &  \mathbb{E}_{p}[\, q_{\text{u}}^i\,]\; \leq \; \epsilon_i \; \text{ for }\; i \in [m].
    \end{array}
\end{equation}
Problem~\eqref{eq: unlearning alignment} aims to minimize the deviation from the pretrained model $q$ while directly reducing the likelihood of the unlearning concepts appearing in the generated samples. In Section~\ref{subsec: likelihood unlearning} we capture a solution to Problem~\eqref{eq: unlearning alignment},
\[
    p_{\text{revl}}^\star(\cdot) 
    \; \propto \;
     \frac{q(\cdot)}{{\rm e}^{\sum_{i\,=\,1}^m \alpha_i\, q_{\text{u}}^i(\cdot)}}
\]
where $\{\alpha_i\}_{i\,=\,1}^m$ is a set of non-negative weights that balance the relative importance of each model. The subscript revl stands for reverse KL with likelihood. The exponential term downweights samples that have high likelihood under the unlearning distributions. Compared with the framework of constrained relative-entropy minimization~\cite{koyejo2013representation}, our formulation chooses the moment functions to be unlearning-distribution likelihoods, thereby turning the exponential tilt into a principled mechanism for pushing the learned distribution away from undesired concepts.

To illustrate, we use a three-Gaussian mixture example to compare the three unlearning solutions, as shown in Figure~\ref{fig:gaussians_likelihood_vs_reverse}. While the formulations~\eqref{eq: unlearning reverse KL} and~\eqref{eq: unlearning alignment} both reduce the probability of sampling from the unlearning concepts, their resulting solutions differ significantly. In particular, reverse KL-constrained unlearning achieves this by pushing the other modes we wish to retain away from the unlearned concept, whereas likelihood-constrained unlearning does not. This occurs because, for a sample $x$ where $q_{\text{u}}^i$ is small, the solution $p_{\text{rev}}^\star$ is more strongly altered by $q_{\text{u}}^i$ than $p_{\text{revl}}^\star$ (proportional vs exponential).

\textbf{Background on Diffusion Models}. A diffusion model consists of two stochastic processes: a forward process and a backward process. The forward process starts from a data point $x_0$ and gradually adds noise according to $x_{t} = \sqrt{1-\beta_t} x_{t-1} + \sqrt{\beta_t}w_t$ for $t\in [T]$, where $\beta_t\in (0,1)$ is the stepsize and $w_t\sim\mathcal{N}(0,I)$ is the Gaussian noise. The backward process starts from pure noise $x_T\sim\mathcal{N}(0,I)$ and iteratively removes noise. For instance, the DDPM sampler updates $x_{t-1} = \frac{1}{\sqrt{\alpha_t}}(x_t + \eta_t s(x_t, t) + \sigma_t z_t)$, where $\alpha_t = 1-\beta_t$, $\eta_t>0$ is the stepsize, $\sigma_t$ is the noise scaling, and $z_t\sim\mathcal{N}(0,I)$ is the Gaussian noise. A crucial component of the backward process is the score function $s(x_t,t)$: $X\times[T] \to \mathbb{R}^d$; when it is specified accurately, the backward process converges to the data distribution~\cite{chen2023sampling,li2025convergence}. Throughout the paper, we focus on score-based diffusion models, denoting each probability measure $p$ by its score function $s_p$: $X\times[T] \to \mathbb{R}^d$. We use \(p_0(\cdot; s_p)\) to denote the marginal distribution of the terminal sample \(x_0\), and \(p_{0:T}(\cdot; s_p)\) to denote the path distribution of the trajectory \(x_{0:T}\), both induced by the score function \(s_p\). We use analogous notation for the pretrained models $q$ and $q^i$.

%% file: sec_distspace_camready.tex
\section{Unlearning in Distribution Space}\label{sec: unlearning in distribution space}

In Sections~\ref{subsec: concept unlearning}--\ref{subsec: likelihood unlearning}, we formally present, in distribution space, the three formulations presented in Section~\ref{subsec: contribution}. We prove strong duality in all three cases, despite the nonconvexity of the KL-divergence constraints.

\subsection{Reverse KL-Constrained Unlearning}\label{subsec: concept unlearning}

We begin with Problem~\eqref{eq: unlearning reverse KL}. Let $q$: $\mathcal{B}(X)\to [0, 1]$ be the probability measure induced by a pretrained model, where $\mathcal{B}(X)$ is a Borel $\sigma$-algebra over a bounded domain $X \subseteq \mathbb{R}^d$. In addition to the pretrained model, we consider a set of $m$ probability measures $\{ q_{\text{u}}^i: \mathcal{B}(X)\to [0, 1] \}_{i\,=\,1}^m$, each modeling a concept that we wish the model to forget/unlearn. 

To ensure well-posedness, we assume that the probability measures of interest satisfy certain admissibility conditions.  

\begin{assumption}[Admissibility]\label{as: admissibility reverse KL}
There exists a probability measure $p$ satisfying two conditions: (i) $p$ is absolutely continuous with respect to $q$, $q_{\text{\normalfont u}}^i$, and (ii) $p$ is bounded, i.e.,
\begin{equation}\label{eq: admissible measures reverse KL}
    p \ll q,\, q_{\text{\normalfont u}}^i \, \text{ for } \, i \in [m]
    \; \text{ and } \;
    p(x) < \infty \, \text{ for } \, x \in X.
\end{equation}
\end{assumption}

The absolute continuity ensures the boundedness of reverse KL divergence. It is practical to only consider bounded models, since diffusion models are represented by conditional Gaussian distributions. Problem~\eqref{eq: unlearning reverse KL} is a \emph{nonconvex} optimization problem because the superlevel set of reverse KL divergence is not convex. As a result, we cannot invoke the convexity of the feasible set to conclude strong duality from convex duality analysis~\cite{boyd2004convex}.
\begin{assumption}[Feasibility]
\label{as: feasibility reverse KL}
    There exists an admissible probability measure $p \in \Delta$: $D_{\text{\normalfont KL}}(p\,\Vert\, q_{\text{\normalfont u}}^i) >  b_i$ for $i \in [m]$.
\end{assumption}

Since $q_{\text{u}}^i$ is focal on a specific concept or class, we can always construct $p$ such that most of its mass is located on samples with low probability under $q_{\text{u}}^i$. This makes Assumption~\ref{as: feasibility reverse KL} practically achievable. Let $p_{\text{rev}}^\star$ be a solution to Problem~\eqref{eq: unlearning reverse KL} and denote $P_{\text{rev}}^\star = D_{\text{KL}}(p_{\text{rev}}^\star\,\Vert\, q)$. 
Denote $\lambda \DefinedAs [\,\lambda_1,\ldots,\lambda_m\,]^\top$. 
Let the Lagrangian for Problem~\eqref{eq: unlearning reverse KL}~be
\[
    \begin{array}{rcl}
         L_{\text{rev}}(p, \lambda)
        & = & 
        \displaystyle
        D_{\text{KL}}(p\,\Vert\, q)
        \, + \,
        \sum_{i\,=\,1}^m \lambda_i \left(
        b_i 
        - D_{\text{KL}}(p\,\Vert\, q_{\text{u}}^i)
        \right)
    \end{array}
\]
and the associated dual function is given by $D_{\text{rev}}(\lambda) 
     =  \min_{p\,\in\,\Delta}  L_{\text{rev}}(p, \lambda)$, 
where the minimization is achieved at $p_{\text{rev}}^\star(\cdot;\lambda)$. 
Let an optimal dual variable be $\lambda_{\text{rev}}^\star \in \argmax_{\lambda\,\geq\,0} D_{\text{rev}}(\lambda)$, and the optimal value of the dual function be $D_{\text{rev}}^\star \DefinedAs D_{\text{rev}}(\lambda^\star)$.
 
However, it is unclear when $p_{\text{rev}}^\star(\cdot;\lambda)$ is a solution to Problem~\eqref{eq: unlearning reverse KL}. Due to nonconvexity, $p_{\text{rev}}^\star(\cdot;\lambda)$ is not necessarily a solution to Problem~\eqref{eq: unlearning reverse KL}, even when we fix $\lambda = \lambda_{\text{rev}}^\star$. To address this issue, we next prove that Problem~\eqref{eq: unlearning reverse KL} is strongly dual, despite its nonconvex feasible set.

\begin{theorem}[Strong duality]\label{thm: unlearning reverse KL strong duality}
    Let Assumptions~\ref{as: admissibility reverse KL} and~\ref{as: feasibility reverse KL} hold. Then, strong duality holds for Problem~\eqref{eq: unlearning reverse KL}: $P_{\text{\normalfont rev}}^\star = D_{\text{\normalfont rev}}^\star$. 
    Moreover, an optimal probability measure $p^\star$ is given by $p_{\text{\normalfont rev}}^\star(\cdot) = p_{\text{\normalfont rev}}^\star(\cdot;\lambda_{\text{\normalfont rev}}^\star)$. 
\end{theorem}

We defer the proof of Theorem~\ref{thm: unlearning reverse KL strong duality} to Appendix~\ref{app: reverse kl strong duality}, and sketch our key technical approach below. 

A sufficient condition for strong duality to hold for Problem~\eqref{eq: unlearning reverse KL} is that its epigraph is convex~\cite{chamon2022constrained}. To show this, we observe that all probability measures of interest have continuous domains. This allows us to define a non-atomic vector measure in terms of the objective and constraint functions. A key step is a new application of Lyapunov's convexity theorem~\cite{olech1968range}, which proves the image of a non-atomic vector measure is convex. We stress that it is infeasible to directly apply Lyapunov's convexity theorem, since the originally-constructed vector measure~\cite{olech1968range} is not necessarily a probability measure.

Having established strong duality, we evaluate a solution to Problem~\eqref{eq: unlearning reverse KL} given an optimal dual variable $\lambda_{\text{rev}}^\star$ in Corollary~\ref{cor: optimal distributions reverse KL}. Denote $\hat\lambda_i \DefinedAs \frac{\lambda_i}{1-\one^\top\lambda }$. When $\one^\top\lambda \neq 1$, we define an unlearning target $q_{\text{u}}^{\dagger}(\cdot;\lambda)$ as a ratio of tilted distributions,
\begin{equation}\label{eq: negation neq}
    q_{\text{\normalfont u}}^{\dagger} (\cdot;\lambda)
    \; \DefinedAs \;
    \frac{1}{ Z^\dagger_{\text{\normalfont 
    u}}(\lambda)}
    \frac{(q(\cdot))^{1/(1-\one^\top\lambda)}}{\prod_{i\,=\,1}^m(q_{\text{\normalfont u}}^i(\cdot))^{\hat\lambda_i}} 
\end{equation}
where $Z_{\text{\normalfont 
    u}}^{\dagger}(\lambda) \DefinedAs \int {(q(x))^{1/(1-\one^\top\lambda)}}/{\prod_{i\,=\,1}^m(q_{\text{\normalfont u}}^i(x))^{\hat\lambda_i}} dx$.  
    When $\one^\top\lambda = 1$,~\eqref{eq: negation neq} can be simplified in Appendix~\ref{app: optimal distributions reverse KL}; however, this is a degenerate solution, so we omit it here.
    Since $q$ and $\{q_{\text{\normalfont u}}^i\}_{i\,=\,1}^m$ share the same support, the unlearning target~\eqref{eq: negation neq} is a proper probability measure. Outside the support, $\frac{0}{0} = 0$ by convention. Viewing $q_{\text{\normalfont u}}^i$ as a concept distribution to be unlearned,~\eqref{eq: negation neq} reweighs the pretrained model $q$ by the reciprocal of the tilted distribution $q_{\text{\normalfont u}}^i$. Hence,~\eqref{eq: negation neq} excludes, or downweights, samples that are likely in $q_{\text{\normalfont u}}^i$. Similar targets have been used in concept erasure~\cite{gandikota2023erasing} and negation composition~\cite{du2023reduce}. 

\begin{cor}[Optimal models]\label{cor: optimal distributions reverse KL}
    Let Assumptions~\ref{as: admissibility reverse KL} and~\ref{as: feasibility reverse KL} hold. 
    Then, given an optimal dual variable $\lambda_{\text{\normalfont rev}}^\star$ that satisfies $0 \leq \one^\top\lambda_{\text{\normalfont rev}}^\star < 1$, Problem~\eqref{eq: unlearning reverse KL} admits an explicit solution, 
    \begin{equation}\label{eq: unlearning solutions}
    p_{\text{\normalfont rev}}^\star(\cdot) 
    \; = \;
    q_{\text{\normalfont u}}^{\dagger}(\cdot; \lambda_{\text{\normalfont rev}}^\star).
    \end{equation} 
    Moreover, it is infeasible or degenerate when $\one^\top\lambda_{\text{\normalfont rev}}^\star \geq 1$.
\end{cor}

The proof of Corollary~\ref{cor: optimal distributions reverse KL} evaluates the partial minimization of $L(p;\lambda)$ over $p\in\Delta$, and we provide it in Appendix~\ref{app: optimal distributions reverse KL}. When $0\leq \one^\top\lambda_{\text{\normalfont rev}}^\star<1$, the solution~\eqref{eq: unlearning solutions} captures the effect of the concept unlearning through the distribution ratio~\eqref{eq: negation neq}. An example of this is the concept erasing~\cite{gandikota2023erasing}. We note that when $\one^\top\lambda_{\text{\normalfont rev}}^\star > 1$, the distribution ratio~\eqref{eq: negation neq} is reversed; this leads to an infeasible unlearning target, which we therefore exclude. Hence, the reverse KL constraints effectively characterize concept unlearning.
  
\subsection{Forward KL-Constrained Unlearning}\label{subsec: data unlearning}

We next move to Problem~\eqref{eq: unlearning forward KL}. The forward KL objective ensures that the model remains \emph{likely} under the pretrained model, while the $m$ forward KL constraints \emph{reduce} the model's likelihood of generating samples from $m$ undesirable distributions. To ensure the well-posedness, we assume the following admissibility conditions.

\begin{assumption}[Admissibility]\label{as: admissible forward KL} 
    Any probability measure $p \in \Delta$ satisfies two conditions: (i) $q$ and $q_{\text{\normalfont u}}^i$ are absolutely continuous with respect to $p$; and (ii) $p$ is bounded, i.e.,
    \begin{equation}\label{eq: admissible measures forward KL}
        q,\, q_{\text{\normalfont u}}^i \ll  p \, \text{ for } \, i \in [m] 
        \; \text{ and } \;
        p(x) < \infty 
        \, \text{ for } \, 
        x \in X.
    \end{equation}
\end{assumption}

The absolute continuity ensures the boundedness of forward KL divergence. Problem~\eqref{eq: unlearning forward KL} is a \emph{non-convex} optimization problem because the superlevel set of forward KL divergence is \emph{nonconvex}. Again, strong duality is not ensured by convex duality analysis~\cite{boyd2004convex}.

\begin{assumption}[Feasibility]
\label{as: feasibility forward KL}
    There exists an admissible probability measure $p \in \Delta$: $D_{\text{\normalfont KL}}(q_{\text{\normalfont u}}^i \,\Vert\, p) >  b_i$ for $i \in [m]$.
\end{assumption}

Compared to Assumption~\ref{as: feasibility reverse KL}, we can always find a $p$ such that few of its mass is located on samples with high probability under $q_{\text{u}}^i$. Thus,
Assumption~\ref{as: feasibility reverse KL} is also practically achievable. Let $p_{\text{fw}}^\star$ be a solution to Problem~\eqref{eq: unlearning forward KL} and denote $P_{\text{fw}}^\star = D_{\text{KL}}(q \,\Vert\, p_{\text{fw}}^\star)$. Denote $\lambda \DefinedAs [\,\lambda_1,\ldots,\lambda_m\,]^\top$. Let the Lagrangian for Problem~\eqref{eq: unlearning forward KL} be
\[
    \begin{array}{rcl}
         L_{\text{fw}}(p, \lambda)
        & = & 
        \displaystyle
        D_{\text{KL}}(q \,\Vert\, p)
        \, + \,
        \sum_{i\,=\,1}^m \lambda_i \left(
        b_i 
        - D_{\text{KL}}( q_{\text{u}}^i \,\Vert\, p)
        \right)
    \end{array}
\]
and the associated dual function is given by $D_{\text{fw}}(\lambda) 
     =  \min_{p\,\in\,\Delta}  L_{\text{fw}}(p, \lambda)$, 
where the minimization is achieved at $p_{\text{fw}}^\star(\cdot;\lambda)$. 
Let an optimal dual variable be $\lambda_{\text{fw}}^\star \in \argmax_{\lambda\,\geq\,0} D_{\text{fw}}(\lambda)$, and the optimal value of the dual function be $D_{\text{fw}}^\star \DefinedAs D_{\text{fw}}(\lambda_{\text{fw}}^\star)$. Similar to Theorem~\ref{thm: unlearning reverse KL strong duality}, we next show that $p_{\text{fw}}^\star(\cdot;\lambda)$ is a solution to Problem~\eqref{eq: unlearning forward KL} if $\lambda = \lambda_{\text{fw}}^\star$. To do so, we prove that Problem~\eqref{eq: unlearning forward KL} is strongly dual.

\begin{theorem}[Strong duality]\label{thm: unlearning forward KL strong duality}
    Let Assumptions~\ref{as: admissible forward KL} and~\ref{as: feasibility forward KL} hold. Then, strong duality holds for Problem~\eqref{eq: unlearning forward KL}: $P_{\text{\normalfont fw}}^\star = D_{\text{\normalfont fw}}^\star$. Moreover, an optimal probability measure $p_{\text{\normalfont fw}}^\star$ is given by $p_{\text{\normalfont fw}}^\star(\cdot) = p_{\text{\normalfont fw}}^\star(\cdot;\lambda_{\text{\normalfont fw}}^\star)$.
\end{theorem}

Theorem~\ref{thm: unlearning forward KL strong duality} has a similar proof as the one for Theorem~\ref{thm: unlearning reverse KL strong duality}. We explain the proof of Theorem~\ref{thm: unlearning forward KL strong duality} in Appendix~\ref{app: unlearning forward KL strong duality}. 

With the strong duality in place, Corollary~\ref{cor: optimal distributions forward KL} captures an explicit solution to Problem~\eqref{eq: unlearning forward KL}. When ${q(x)} - 
        \sum_{i\,=\,1}^m \lambda_i{q_{\text{u}}^i(x)} \geq 0$ for any $x\in X$ and $\lambda\geq0$, we define an unlearning target $q_{\text{u}}^\triangleleft(\cdot;\lambda)$ as a difference between a pretrained model and a mixture distribution to be unlearned,
        \begin{equation}\label{eq: mixture minus}
            q_{\text{u}}^{\triangleleft}(\cdot;\lambda)
            \; = \; 
            \frac{1}{Z_{\text{u}}^\triangleleft(\lambda)} 
            \left(
            {q(\cdot)} 
            \, - \,
             \sum_{i\,=\,1}^m \lambda_i
             \,
             {q_{\text{u}}^i(\cdot)}
             \right)
        \end{equation}
        where $Z_{\text{u}}^\triangleleft(\lambda) \DefinedAs \int \left(
        {q(x)} - 
        \sum_{i\,=\,1}^m \lambda_i{q_{\text{u}}^i(x)}\right) dx$. The unlearning target~\eqref{eq: mixture minus} defines a valid distribution by reducing the likelihood of generating samples from a mixture. This target is similar to the data-unlearning~\cite{wu2025erasing,alberti2025data}, which runs gradient descent on the retained dataset and gradient ascent on the unlearning dataset.

\begin{cor}[Optimal models]\label{cor: optimal distributions forward KL}
    Let Assumptions~\ref{as: admissible forward KL} and~\ref{as: feasibility forward KL} hold. Then, given an optimal dual variable $\lambda_{\text{\normalfont fw}}^\star\geq0$, Problem~\eqref{eq: unlearning forward KL} admits an explicit solution,
    \begin{equation}\label{eq: unlearning solutions_1}
    p_{\text{\normalfont fw}}^\star(\cdot) 
    \; = \; 
    q_{\text{\normalfont u}}^{\triangleleft}(\cdot;\lambda_{\text{\normalfont fw}}^\star)
    \end{equation}
    where $\lambda_{\text{\normalfont fw}}^\star$ satisfies ${q(x)} - 
        \sum_{i\,=\,1}^m \lambda_i{q_{\text{\normalfont u}}^i(x)} \geq 0$ for $x\in X$.
\end{cor}

We provide the proof of Corollary~\ref{cor: optimal distributions forward KL} in Appendix~\ref{app: optimal distributions forward KL}.

\subsection{Likelihood-Constrained Unlearning}\label{subsec: likelihood unlearning}
 
We now consider Problem~\eqref{eq: unlearning alignment}. A key distinction from Problem~\eqref{eq: unlearning reverse KL} is that it constrains the likelihood of generating samples from undesirable concepts, rather than the KL divergences. The likelihood constraints force the model to generate samples from the unlearning concepts less often, with each bounded by $\epsilon_i$. This constraint is intuitive since it prevents the model from assigning high probability to regions associated with the unlearning distributions. Since the likelihood-based feasible set is convex, we can directly conclude strong duality~\cite{boyd2004convex}.  

\begin{assumption}[Feasibility]
\label{as: feasibility reverse KL likelihood}
    There exists an admissible probability measure $p \in \Delta$: $\mathbb{E}_p[ q_{\text{\normalfont u}}^i ] <  \epsilon_i$ for $i \in [m]$.
\end{assumption}

Let $p_{\text{revl}}^\star$ be a solution to Problem~\eqref{eq: unlearning alignment} and denote $P_{\text{revl}}^\star = D_{\text{KL}}(p_{\text{revl}}^\star\,\Vert\, q)$. Denote $\lambda \DefinedAs [\,\lambda_1,\ldots,\lambda_m\,]^\top$. Let the Lagrangian for Problem~\eqref{eq: unlearning alignment} be
\[
    \begin{array}{rcl}
         L_{\text{revl}}(p, \lambda)
        & = & 
        \displaystyle
        D_{\text{KL}}(p\,\Vert\, q)
        \, + \,
        \sum_{i\,=\,1}^m \lambda_i \left(\,
         \mathbb{E}_{p} [ \,q_{\text{u}}^i \,]
         -
         \epsilon_i
        \,\right)
    \end{array}
\]
and the associated dual function is given by $D_{\text{revl}}(\lambda) 
     =  \min_{p\,\in\,\Delta}  L_{\text{revl}}(p, \lambda)$, 
where the minimization is achieved at
\begin{equation}\label{eq: likelihood exp}
    p_{\text{revl}}^\star(\cdot;\lambda) \;=\; \frac{1}{Z_{\text{u}}^{\circ}(\lambda)}\, q(\cdot)\, {\rm e}^{-\sum_{i\,=\,1}^m \lambda_i \, q_{\text{u}}^i(\cdot)}
\end{equation}
where ${Z_{\text{u}}^{\circ}(\lambda)} \DefinedAs \int q(x) {\rm e}^{-\sum_{i\,=\,1}^m \lambda_i q_{\text{u}}^i(x)} dx$. Let an optimal dual variable be $\lambda_{\text{revl}}^\star \in \argmax_{\lambda\,\geq\,0} D_{\text{revl}}(\lambda)$, and the optimal value of the dual function be $D_{\text{revl}}^\star \DefinedAs D_{\text{revl}}(\lambda_{\text{revl}}^\star)$.

\begin{theorem}[Strong duality]\label{thm: unlearning likelihood}
     Let Assumptions~\ref{as: admissibility reverse KL} and~\ref{as: feasibility reverse KL likelihood} hold. Then, strong duality holds for Problem~\eqref{eq: unlearning alignment}: $P_{\text{\normalfont revl}}^\star = D_{\text{\normalfont revl}}^\star$. Moreover, an optimal probability measure $p_{\text{\normalfont revl}}^\star$ is given by 
     \begin{equation}\label{eq: unlearning solutions_2}
     p_{\text{\normalfont revl}}^\star(\cdot) \; = \; p_{\text{\normalfont revl}}^\star(\cdot;\lambda_{\text{\normalfont revl}}^\star).
     \end{equation}
\end{theorem}

The proof of Theorem~\ref{thm: unlearning likelihood} is based on the convex duality analysis~\cite{boyd2004convex} in Appendix~\ref{app: unlearning likelihood}. Given an optimal dual variable $\lambda_{\text{revl}}^\star\geq 0$, the optimal model $p_{\text{revl}}^\star(\cdot;\lambda_{\text{revl}}^\star)$ that we describe in~\eqref{eq: likelihood exp} characterizes the effect of the concept unlearning through the negative exponent $-\sum_{i\,=\,1}^m \lambda_i q_{\text{u}}^i$. Regions with high likelihood under the unlearning distributions are downweighted exponentially.

\begin{remark}[Constrained vs. Regularized Formulations]
The solutions to Problems~\eqref{eq: unlearning reverse KL},~\eqref{eq: unlearning forward KL}, and~\eqref{eq: unlearning alignment} also admit regularized interpretations: each can be viewed as promoting separation from the distributions to be unlearned via suitable penalty terms. Under proper regularity conditions, for each feasible set of constraint thresholds, there exists a set of dual or penalty weights (or exponents) such that the constrained and regularized formulations agree at the level of optimal solutions. However, this equivalence is primarily variational and does not imply that the corresponding methods are equally effective in practice. We highlight three key reasons to prefer constrained formulations.
\begin{enumerate}
    \item[\text{\normalfont(i)}] Constraint thresholds are more interpretable than penalty weights.
    One can require the learned distribution to be separated from each unlearning distribution by a prescribed amount, or put an upper bound on the expected likelihood of undesirable samples. These specifications have direct semantic meaning, whereas penalty weights are indirect and harder to interpret.

    \item[\text{\normalfont(ii)}] Penalty weights do not directly specify the desired level of unlearning.
    A penalty weight, e.g., \(\alpha_i\) in the solution \(p_{\text{fw}}^\star(\cdot) \), does not mean that the \(i\)th unlearning target is \(\alpha_i\) times more important. Rather, at the optimum, it balances the marginal change in the unlearning objective against the marginal change in the retention objective. In contrast, a constraint level \(b_i\) or \(\epsilon_i\) directly specifies a target separation or likelihood level.

    \item[\text{\normalfont(iii)}] Constraints better control multiple unlearning targets.
    With many concepts or datasets, tuning separate penalty weights is impractical, while assigning the same weight to all targets can lead to uneven unlearning because some concepts are easier to erase than others. Constraint thresholds provide a direct way to enforce balanced unlearning across targets, achieving the desired erasure effect with less deviation from the pretrained model, as shown in Figures~\ref{fig: forward pareto} and~\ref{fig: reverseKL_pareto}.
\end{enumerate}
\end{remark}

%% file: sec_diffusion_camready.tex
\section{Unlearning for Diffusion Models}\label{sec: unlearning in score function space}

We instantiate the three formulations in Section~\ref{subsec: contribution} by letting all distributions be induced by diffusion models. 

\subsection{Reverse KL-Constrained Unlearning}\label{subsec: concept unlearning diffusion models_reverse KL}

We specify the probability measures $p$, $q$, and $q_{\text{u}}^i$ in Problem~\eqref{eq: unlearning reverse KL} as the point distributions $p_{0}(\cdot; s_p)$, $q_{0}(\cdot; s_q)$, and $q_{0}^i(\cdot; s_q^i)$, respectively, where $s_p$, $s_q$, and $s_q^i$ denote the corresponding score functions. We cast Problem~\eqref{eq: unlearning reverse KL} into a constrained optimization with functional constraints,
    \begin{align}\label{eq: unlearning reverse KL diffusion models_point-wise_score}
        \displaystyle
        \minimize_{s_p\,\in\,\mathcal{S}} \;\;
        &
        D_{\text{KL}}(p_0(\cdot; s_p)\,\Vert\, q_0(\cdot; s_q))
        \\
        \subject \;\;
        & D_{\text{KL}}(p_0(\cdot; s_p)\,\Vert\, q_0^i(\cdot; s_q^i)) \geq b_i \, \text{ for }\, i\in [m].\nonumber
    \end{align}
Problem~\eqref{eq: unlearning reverse KL diffusion models_point-wise_score} is a \emph{nonconvex} optimization problem because the mapping from the score function to the KL divergence is nonlinear. Nevertheless, when viewed as a function of the underlying probability measures, it has the same structure as Problem~\eqref{eq: unlearning reverse KL}. To exploit this structure, we assume that the score function class is sufficiently expressive to represent any target point distribution within the class of interest. This is mild, as diffusion models are implemented as overparameterized deep neural networks.

\begin{assumption}[Expressivity]\label{as: admissibility reverse KL_point-wise}
    There exists a class of point distributions in which each $p_{0}(\cdot)$ is induced by a backward process associated with a score function $s_p \in \mathcal{S}$.
\end{assumption}

\begin{assumption}[Feasibility]
\label{as: feasibility reverse KL_point-wise}
    There exists an admissible score $s_p$: $D_{\text{\normalfont KL}}(p_{0}(\cdot;s_p)\,\Vert\, q_{0}^i(\cdot; s_q^i)) >  b_i$ for $i \in [m]$.
\end{assumption}

Assumption~\ref{as: admissibility reverse KL_point-wise} is reasonable, because score-based backward processes converge to a data distribution under very mild conditions~\cite{chen2023sampling,li2025convergence}. Assumption~\ref{as: feasibility reverse KL_point-wise} resembles the rationale behind Assumption~\ref{as: feasibility reverse KL}. Let $s_{\text{rev}}^\star$ be a solution to Problem~\eqref{eq: unlearning reverse KL diffusion models_point-wise_score} and denote $\hat P_{\text{rev}}^\star = D_{\text{KL}}(p_{0}(\cdot; s_{\text{rev}}^\star)\,\Vert\, q_{0}(\cdot; s_q))$. The Lagrangian for Problem~\eqref{eq: unlearning reverse KL diffusion models_point-wise_score} is given by $\hat L_{\text{rev}}(s_p, \lambda) = L_{\text{rev}}(p_{0}(\cdot; s_p), \lambda)$, and its dual function is given by $\hat D_{\text{rev}}(\lambda) \DefinedAs \min_{s_p\,\in\,\mathcal{S}} L_{\text{rev}}(p_{0}(\cdot; s_p), \lambda)$. Let an optimal dual variable be $\hat\lambda_{\text{rev}}^\star \in\argmax_{\lambda\,\geq\,0} \hat D_{\text{rev}}(\lambda)$, and the optimal value of the dual function be $\hat D_{\text{rev}}^\star \DefinedAs \hat D_{\text{rev}}(\hat\lambda_{\text{rev}}^\star)$.

\begin{theorem}[Strong duality]\label{thm: unlearning reverse KL strong duality_point-wise}
    Let Assumptions~\ref{as: admissibility reverse KL_point-wise} and~\ref{as: feasibility reverse KL_point-wise} hold. Then, strong duality holds for Problem~\eqref{eq: unlearning reverse KL diffusion models_point-wise_score}: $\hat P_{\text{\normalfont rev}}^\star = \hat D_{\text{\normalfont rev}}^\star$.
\end{theorem}

We prove Theorem~\ref{thm: unlearning reverse KL strong duality_point-wise} in Appendix~\ref{app: reverse kl strong duality_point-wise}. Despite the nonlinearity of Problem~\eqref{eq: unlearning reverse KL diffusion models_point-wise_score}, we exploit the convexity of a non-atomic vector measure---now it is defined on the path space---to establish strong duality. Since the score function class is sufficiently expressive, we then carry the result over to the score function space using the saddle-point characterization of strong duality. We note that the convexity of a non-atomic vector measure is unaffected by the nonlinearity introduced in~\eqref{eq: unlearning reverse KL diffusion models_point-wise_score} when passing to point distributions.

Exploiting strong duality, we introduce a primal-dual algorithm for solving Problem~\eqref{eq: unlearning reverse KL diffusion models_point-wise_score}, alternating between minimizing the Lagrangian via gradient descent, and maximizing the dual function via dual subgradient ascent: 
\begin{itemize}
    \item[(i)] Primal step is to solve the Lagrangian problem: $s^+ \in \argmin_{s}\hat{L}_{\text{rev}}(s,\lambda)$, where the Lagrangian $\hat{L}_{\text{rev}}(s,\lambda)$ evaluates the KL divergence to the target~\eqref{eq: negation neq}; 
    \item[(ii)] Dual step updates the dual variable via $\lambda^+ = \lambda + \eta \partial\hat{D}_{\text{rev}}(\lambda)$, where the subgradient $\partial\hat{D}_{\text{rev}}(\lambda)$ is evaluated using the current model $s$, and $\eta$ is the stepsize. 
\end{itemize}
See Appendix~\ref{app: algorithms} for the full description.

\subsection{Forward KL-Constrained Unlearning}\label{sec: data unlearning diffusion models}

Instead of the point distributions, we specify the probability measures $p$, $q$, $q_{\text{u}}^i$ in Problem~\eqref{eq: unlearning forward KL} as the path distributions $p_{0:T}(\cdot; s_p)$, $q_{0:T}(\cdot; s_q)$, and $q_{0:T}^i(\cdot; s_q^i)$. Problem~\eqref{eq: unlearning forward KL} can be cast as a constrained optimization problem,
    \begin{align}\label{eq: unlearning forward KL diffusion models_path-wise_score}
        \displaystyle
        \minimize_{s_p\,\in\,\mathcal{S}} 
        \;\; &  
        D_{\text{KL}}(q_{0:T}(\cdot; s_q)\,\Vert\, p_{0:T}(\cdot; s_p))
        \\
        \subject 
        \;\; &  D_{\text{KL}}(q^i_{0:T}(\cdot; s^i_q)\,\Vert\, p_{0:T}(\cdot; s_p)) \geq b_i\, \text{ for }\, i\in [m]. \nonumber
    \end{align}
Problem~\eqref{eq: unlearning forward KL diffusion models_path-wise_score} is a \emph{nonconvex} optimization problem for the same reason as Problem~\eqref{eq: unlearning reverse KL diffusion models_point-wise_score}. Similarly, we exploit the strong duality for Problem~\eqref{eq: unlearning forward KL} to establish strong duality. 

\begin{assumption}[Expressivity]\label{as: admissibility foward KL_path-wise}
    There is a class of path distributions in which each $p_{0:T}(\cdot)$ is induced by a backward process associated with a score function $s_p \in \mathcal{S}$.
\end{assumption}

\begin{assumption}[Feasibility]
\label{as: feasibility forward KL_path-wise}
    There exists an admissible score $s_p$: $D_{\text{\normalfont KL}}(q_{0:T}^i(\cdot; s_q^i)\,\Vert\, p_{0:T}(\cdot;s_p)) >  b_i$ for $i \in [m]$.
\end{assumption}

Assumption~\ref{as: admissibility foward KL_path-wise} is more restrictive than Assumption~\ref{as: admissibility reverse KL_point-wise}, as it uses joint distributions instead of marginal ones. Let $s_\text{fw}^\star$ be a solution to Problem~\eqref{eq: unlearning forward KL diffusion models_path-wise_score} and denote $\hat P_{\text{fw}}^\star = D_{\text{KL}}( q_{0:T}(\cdot; s_q)\,\Vert\,p_{0:T}(\cdot; s_{\text{fw}}^\star))$. The Lagrangian for Problem~\eqref{eq: unlearning forward KL diffusion models_path-wise_score} is given by $\hat L_\text{fw}(s_p, \lambda) = L_\text{fw}(p_{0:T}(\cdot; s_p), \lambda)$, and its dual function is given by $\hat D_\text{fw}(\lambda) \DefinedAs \min_{s_p\,\in\,\mathcal{S}} L_\text{fw}(p_{0:T}(\cdot; s_p), \lambda)$. Let an optimal dual variable be $\hat\lambda_\text{fw}^\star \in\argmax_{\lambda\,\geq\,0} \hat D_\text{fw}(\lambda)$, and the optimal value of the dual function be $\hat D_\text{fw}^\star \DefinedAs \hat D_\text{fw}(\hat\lambda_\text{fw}^\star)$. 

\begin{theorem}[Strong duality]\label{thm: unlearning forward KL strong duality_path-wise}
    Let Assumptions~\ref{as: admissibility foward KL_path-wise} and~\ref{as: feasibility forward KL_path-wise} hold. Then, strong duality holds for Problem~\eqref{eq: unlearning forward KL diffusion models_path-wise_score}: $\hat P_\text{\normalfont fw}^\star = \hat D_\text{\normalfont fw}^\star$.
\end{theorem}

The proof of Theorem~\ref{thm: unlearning forward KL strong duality_path-wise} in Appendix~\ref{app: unlearning forward KL strong duality_path-wise} is similar to that of Theorem~\ref{thm: unlearning reverse KL strong duality_point-wise}, except that the KL divergence is defined over a path distribution. The convexity of a non-atomic vector measure continues to hold. Leveraging strong duality, we introduce a primal-dual algorithm for solving Problem~\eqref{eq: unlearning forward KL diffusion models_path-wise_score}: 
\begin{itemize}
    \item[(i)] Primal step is to solve the Lagrangian problem: $s^+ \in \argmin_{s}\hat{L}_{\text{fw}}(s,\lambda)$, where the Lagrangian $\hat{L}_{\text{fw}}(s,\lambda)$ has the form of standard score-matching  or noise prediction losses; 
    \item[(ii)] Dual step updates the dual variable via $\lambda^+ = \lambda + \eta \partial\hat{D}_{\text{fw}}(\lambda)$, where the subgradient $\partial\hat{D}_{\text{fw}}(\lambda)$ is evaluated using the current model $s$, and $\eta$ is the stepsize. 
\end{itemize}
See Appendix~\ref{app: algorithms} for the full description.

\subsection{Likelihood-Constrained Unlearning}

As done in Section~\ref{subsec: concept unlearning diffusion models_reverse KL}, we introduce the point distributions $q_0^i(\cdot;s_q^i)$ to Problem~\eqref{eq: unlearning alignment} to rewrite it as
    \begin{align}\label{eq: unlearning alignment diffusion models}
        \displaystyle\minimize_{s_p\,\in\,\mathcal{S}} \;\;& 
        D_{\text{KL}}(p_{0:T}(\cdot; s_p)\,\Vert\, q_{0:T}(\cdot; s_q))
        \\
        \subject \;\;&  \mathbb{E}_{x_0 \,\sim\, p_0}\left[\, q_0^i (x_0; s^i_q)\,\right] \; \leq \; \epsilon_i \; \text{ for } \; i \in [m].\nonumber
    \end{align}
In contrast to Problem~\eqref{eq: unlearning alignment}, Problem~\eqref{eq: unlearning alignment diffusion models} is \emph{nonconvex} due to the nonlinear mapping from the score function to the KL divergence and the likelihood. We note that Problems~\eqref{eq: unlearning alignment diffusion models} and~\eqref{eq: unlearning reverse KL diffusion models_point-wise_score} are similar, differing only in the form of the constraints. Critically, the convexity of a non-atomic vector measure continues to hold despite this change in how the constraints depend on the probability measures.
\begin{assumption}[Feasibility]
\label{as: feasibility reverse KL likelihood_point-wise}
    There exists an admissible score $s_p$: $\mathbb{E}_{x_0\,\sim\,p_0(\cdot;s_p)}[ q_0^i(x_0; s_q^i) ] <  \epsilon_i$ for $i \in [m]$.
\end{assumption}

Let $s_{\text{revl}}^\star$ be a solution to Problem~\eqref{eq: unlearning alignment diffusion models} and denote $\hat P_{\text{revl}}^\star = D_{\text{KL}}(p_{0:T}(\cdot; s_{\text{revl}}^\star)\,\Vert\, q_{0:T}(\cdot; s_q))$. The Lagrangian for Problem~\eqref{eq: unlearning alignment diffusion models} is given by $\hat L_{\text{revl}}(s_p, \lambda) = L_{\text{revl}}(p_{0:T}(\cdot; s_p), \lambda)$, and its dual function is given by $\hat D_{\text{revl}}(\lambda) \DefinedAs \min_{s_p\,\in\,\mathcal{S}} L_{\text{revl}}(p_{0:T}(\cdot; s_p), \lambda)$. Let an optimal dual variable be $\hat \lambda_{\text{revl}}^\star \in\argmax_{\lambda\,\geq\,0} \hat D_{\text{revl}}(\lambda)$, and the optimal value of the dual function be $\hat D_{\text{revl}}^\star \DefinedAs \hat D_{\text{revl}}(\hat\lambda_{\text{revl}}^\star)$.

\begin{theorem}[Strong duality]\label{thm: unlearning reverse KL likelihood strong duality_point-wise}
    Let Assumptions~\ref{as: admissibility reverse KL_point-wise} and~\ref{as: feasibility reverse KL likelihood_point-wise} hold. Then, strong duality holds for Problem~\eqref{eq: unlearning alignment diffusion models}: $\hat P_{\text{\normalfont revl}}^\star = \hat D_{\text{\normalfont revl}}^\star$.
\end{theorem}

We prove Theorem~\ref{thm: unlearning reverse KL likelihood strong duality_point-wise} in Appendix~\ref{app: unlearning reverse KL likelihood strong duality_point-wise}. The proof is similar to that of Theorem~\ref{thm: unlearning reverse KL strong duality_point-wise}, except for the likelihood constraints. Since the likelihood constraints can be lifted to path-space constraints, we can still apply the convexity of a non-atomic vector measure. Because of strong duality, we similarly use a primal-dual algorithm for solving Problem~\eqref{eq: unlearning alignment diffusion models}: 

\begin{itemize}
    \item[(i)] Primal step is to solve the Lagrangian problem: $s^+ \in \argmin_{s}\hat{L}_{\text{revl}}(s,\lambda)$, where the Lagrangian $\hat{L}_{\text{revl}}(s,\lambda)$ involves the computation of the KL divergence and the point probability. To estimate the probability $q_0^i(x_0)$ for a sample $x_0$, we apply the information-theoretic diffusion results~\cite{kong2023informationtheoreticdiffusion} to compute $\log q_0^i(x_0)$ by viewing the pretrained diffusion model as a Gaussian noise channel (see Appendix~\ref{app: proof of likelihood estimation}); 
   \item[(ii)] Dual step updates the dual variable via $\lambda^+ = \lambda + \eta \partial\hat{D}_{\text{revl}}(\lambda)$, where the subgradient $\partial\hat{D}_{\text{revl}}(\lambda)$ is evaluated using the current model $s$, and $\eta$ is the stepsize. 
\end{itemize}
See Appendix~\ref{app: algorithms} for the full description.

\begin{remark}[Parametrization gap]
In Problems~\eqref{eq: unlearning reverse KL diffusion models_point-wise_score},~\eqref{eq: unlearning forward KL diffusion models_path-wise_score}, and~\eqref{eq: unlearning alignment diffusion models}, we optimize over distributions induced by diffusion models through a class of score functions. In contrast, the unlearning targets~\eqref{eq: unlearning solutions},~\eqref{eq: unlearning solutions_1}, and~\eqref{eq: unlearning solutions_2} describe the best unlearning distributions when one can optimize over all admissible probability measures. These targets are attained when the unlearning targets are well represented by the score function parametrization. The remaining discrepancy is a parametrization gap, not a duality gap: strong duality identifies the correct target in the ideal distribution optimization problems, while the diffusion implementation can only approximate these targets. 

In practice, this gap is expected to be small when the desired unlearning target is a smooth perturbation of the pretrained model, the retaining and unlearning distributions are well separated, and the score function class has sufficient expressiveness. It can be large when the constraints are aggressive, retained and unlearned concepts are highly entangled, the target requires sharp density suppression, or the practical parametrization is restrictive, for example because of low-rank adapters, limited optimization, or approximate sampling. Experimentally, a large gap would appear as persistent constraint violation, saturating or oscillating dual variables, a plateau in unlearning performance as thresholds become stricter, or excessive degradation of retention metrics such as KL divergence, KID, or sample quality.
\end{remark}

%% file: sec_experiments_camready.tex
\section{Computational Experiments}\label{sec: experiments}

We demonstrate our constrained unlearning approach through a series of computational experiments using Stable Diffusion v1.4~\cite{rombach2022highresolutionimagesynthesislatent} as the base model. Sections~\ref{subsec: Likelihood-Constrained Unlearning Experiments},~\ref{subsec: forward kl experiments}, and~\ref{subsec: reverse kl experiments} cover likelihood-constrained concept unlearning, Forward KL-based removal of memorized samples, and reverse KL-based multi-concept unlearning, respectively.

\subsection{Likelihood-Constrained Unlearning}\label{subsec: Likelihood-Constrained Unlearning Experiments}

We begin with the likelihood-constrained problem~\eqref{eq: unlearning alignment diffusion models} by considering a simple setting in which a diffusion model is pretrained on a three-Gaussian mixture, and one component is to be unlearned. As a baseline, we consider concept erasing~\cite{gandikota2023erasing}, which targets the unlearning distribution ${(q(\cdot))^{1+\eta}}/{(q_{\text{\normalfont u}}(\cdot))^{\eta}}$ for some $\eta>0$. This is similar to the unlearning target~\eqref{eq: negation neq} in the reverse KL case. We visualize the distributions of the pretrained and trained models for each approach in the three leftmost plots of Figure~\ref{fig: gaussians_likelihood_vs_reverse}. At the same unlearning-likelihood level, our likelihood-constrained method achieves a smaller KL divergence to the retained modes, as shown in the right plot of Figure~\ref{fig: gaussians_likelihood_vs_reverse}, where each point represents a constraint threshold or regularization weight for our method or the baseline. 

We further demonstrate this improvement in a pretrained text-to-image model. To unlearn attributes from a concept, we set $q = p_{\text{pre}}(\cdot\,|\,c)$, and $q^i_{\text{u}} = p_{\text{pre}}(\cdot\,|\,c_i)$, where $c$ is a retained concept (e.g., cowboy) that we aim to stay close to, and $c_i$ denotes other concepts or biases associated with the concept $c$ that we wish to remove (e.g., horse). We show the retention–unlearning tradeoff in the two rightmost plots in Figure~\ref{fig: cowboys_likelihood_vs_reverse}, where retention performance is measured by the Kernel Inception Distance (KID)~\cite{bińkowski2021demystifyingmmdgans} (Lower KID means better retention performance) and the KL divergence to the pretrained model, respectively. Compared to the baseline, our likelihood-constrained approach achieves strong unlearning (low likelihood) while more effectively limiting deviation from the retained model. 

This improved tradeoff can also be qualitatively observed. While both the concept erasure approach and our likelihood constrained approach successfully unlearn the undesired conept, in this case 'horse', (Figure~\ref{fig: cowboys_likelihood_vs_reverse} Left), our approach keeps the parts of the distribution that we wish to retain, closer to the pre-trained model (Figure~\ref{fig: cowboys_likelihood_vs_reverse} Middle left)

\begin{figure}[t]
    \centering
    \begin{subfigure}{0.23\textwidth}
        \centering
        \includegraphics[width=\linewidth]{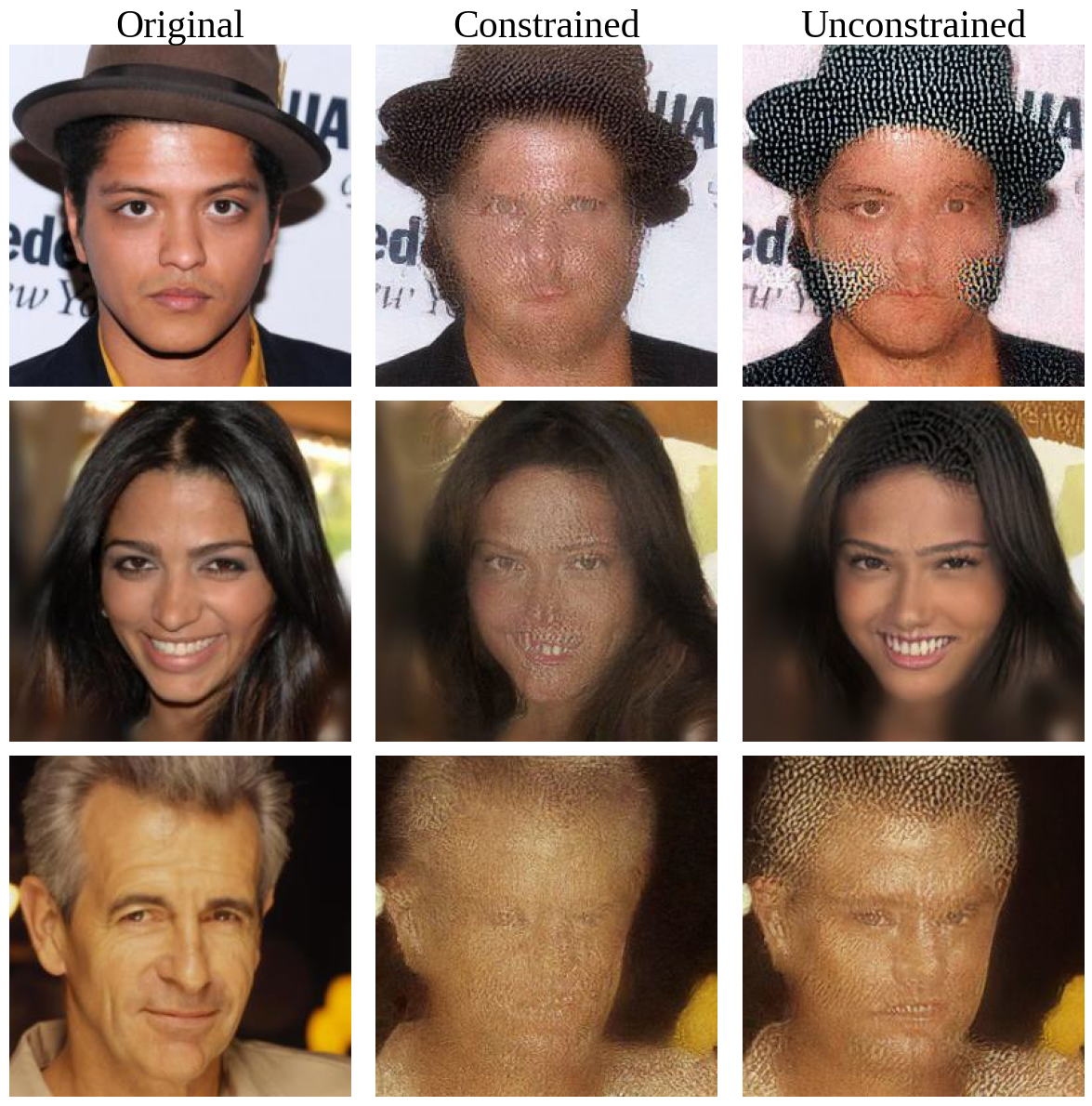}
        \label{fig:sub1}
    \end{subfigure}
    \hfill
    \begin{subfigure}{0.23\textwidth}
        \centering
        \includegraphics[width=\linewidth]{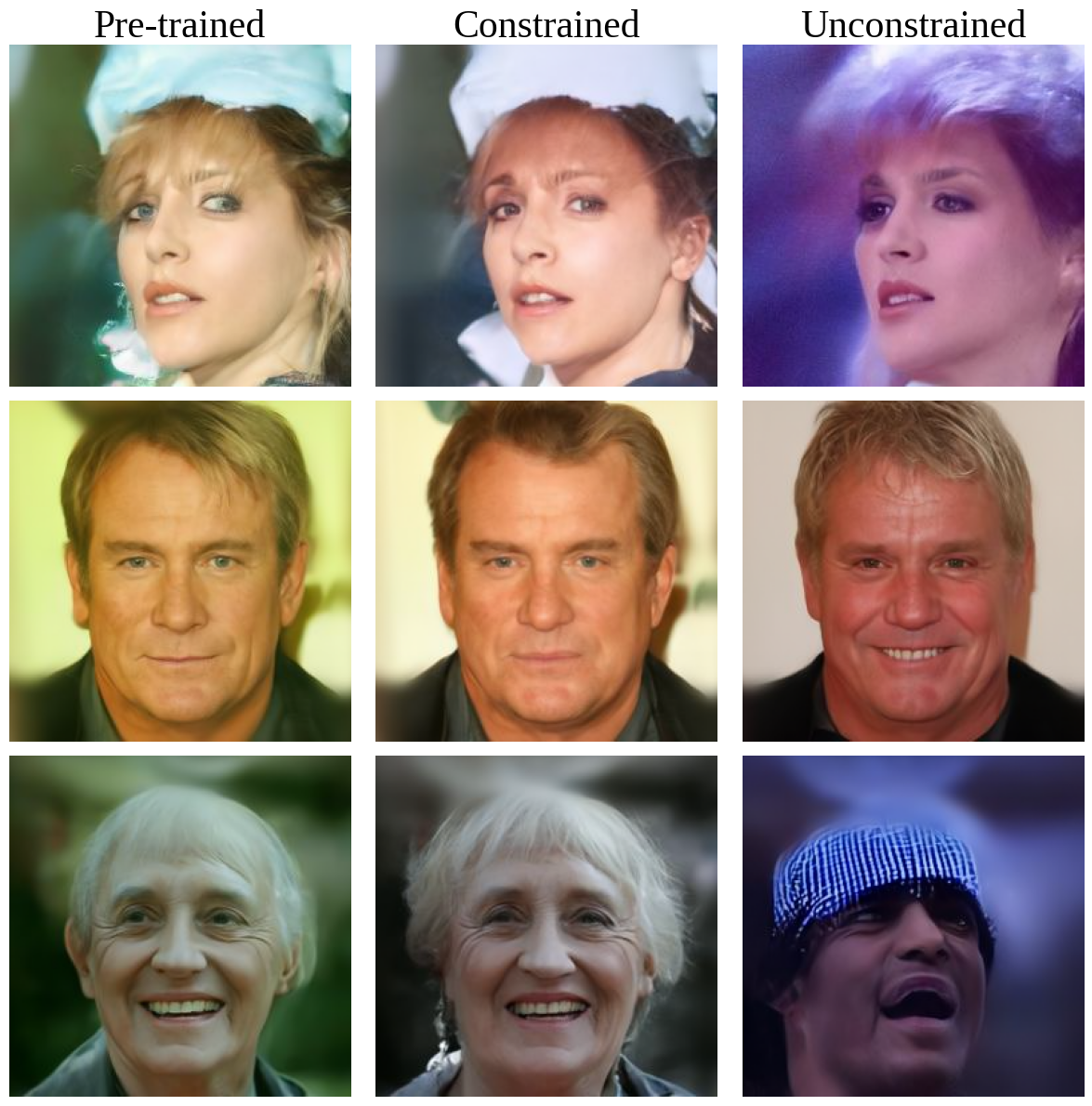}
        \label{fig:sub2}
    \end{subfigure}

    \caption{Performance of forward KL-constrained unlearning. (Left) Images generated by pretrained model (three random images to be removed), forward KL-constrained unlearning, and baseline. (Right) Images generated by pretrained model (three random images), forward KL-constrained unlearning, and  baseline.}
    \label{fig: forward_images}
\end{figure}

\begin{figure}[ht]
    \centering
    \includegraphics[width=0.57\linewidth]{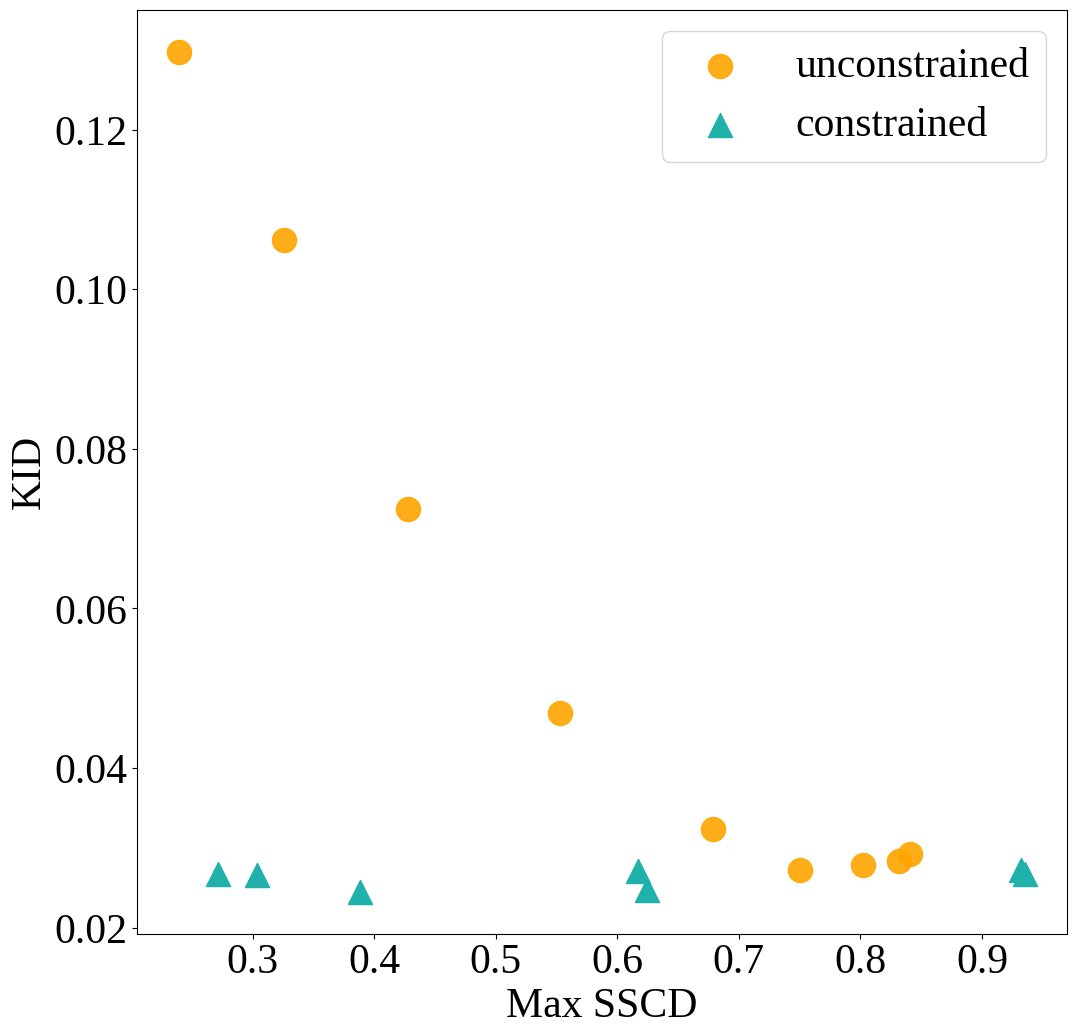}
    \caption{Retention–unlearning tradeoff for forward KL-constrained unlearning and unconstrained baseline. Our constrained model deviates less from the pretrained model at the same level of unlearning (max SSCD).}
    \label{fig: forward pareto}
\end{figure}

\begin{figure*}[t]
    \centering
    \begin{subfigure}{0.23\textwidth}
        \centering
        \includegraphics[width=\linewidth]{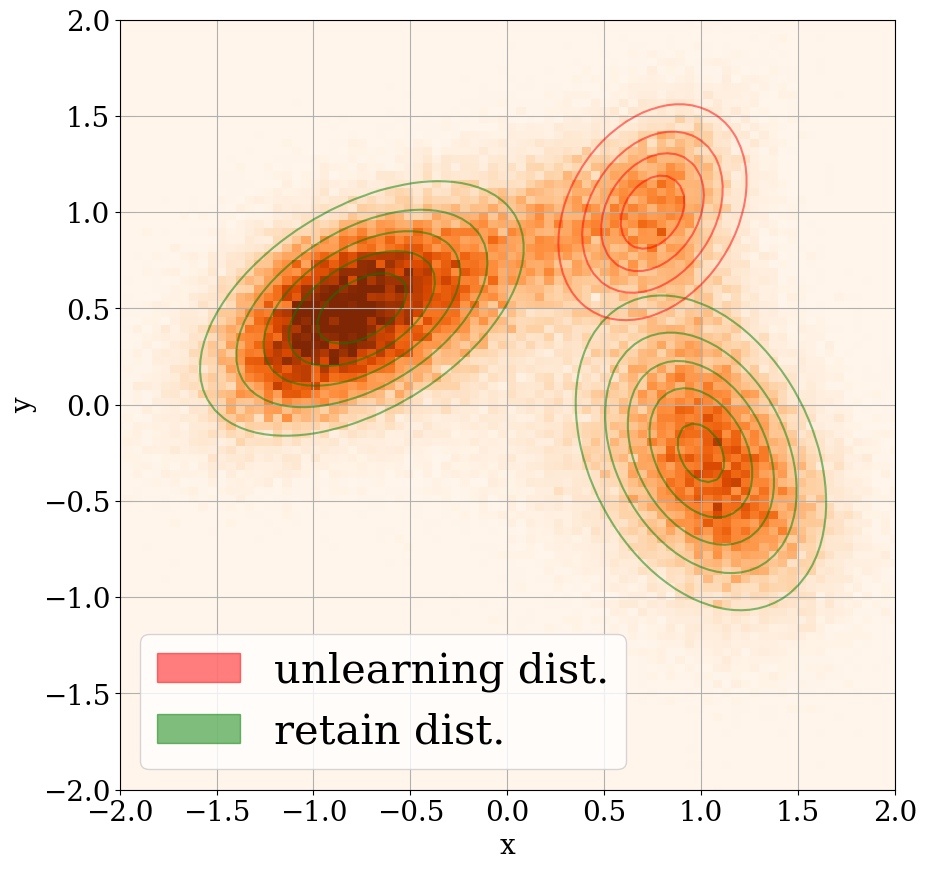}
    \end{subfigure}\hfill
    \begin{subfigure}{0.23\textwidth}
        \centering
        \includegraphics[width=\linewidth]{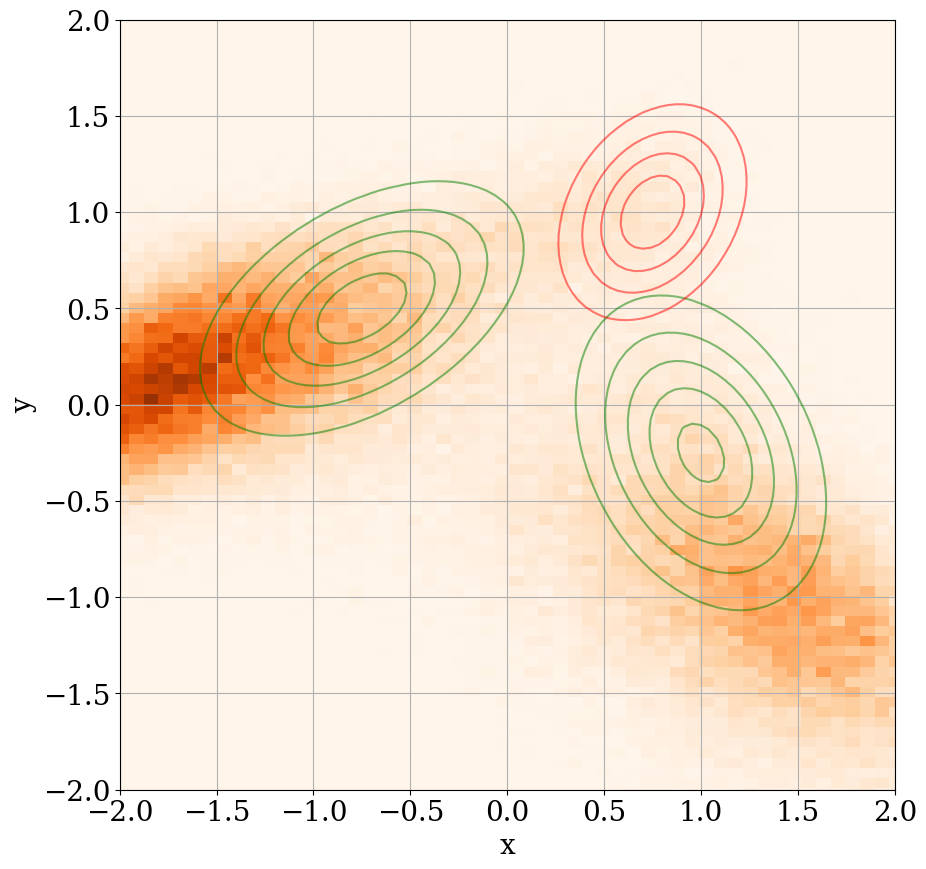}
    \end{subfigure}\hfill
    \begin{subfigure}{0.23\textwidth}
        \centering
        \includegraphics[width=\linewidth]{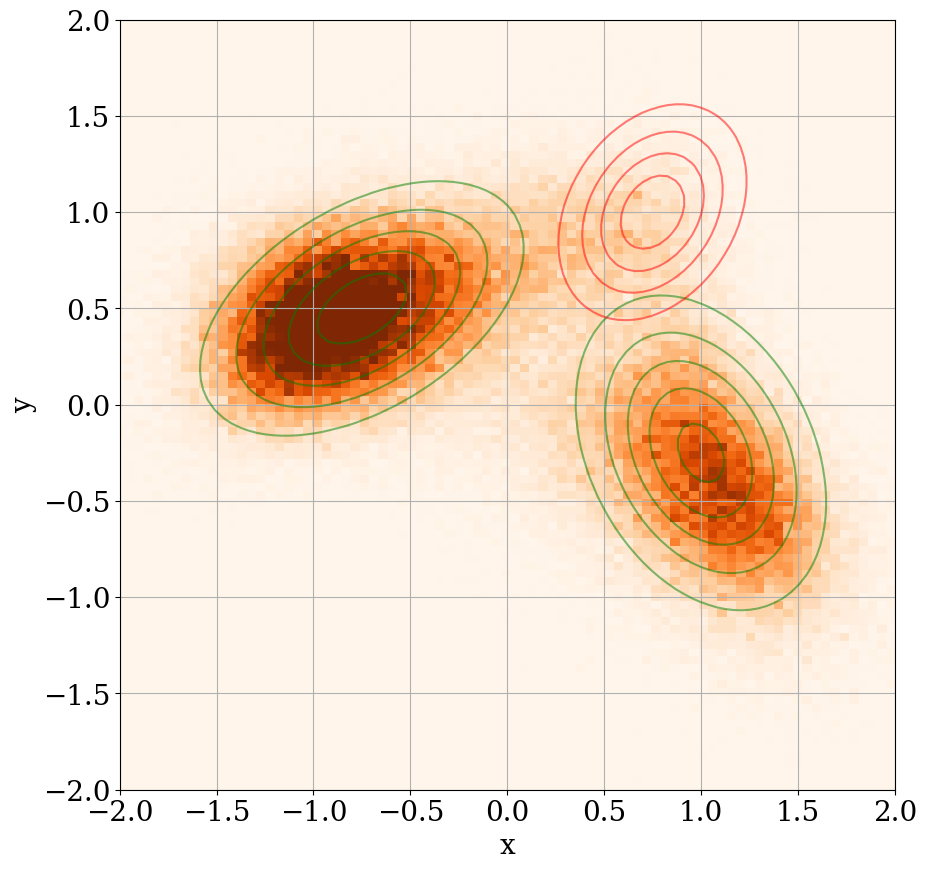}
    \end{subfigure}\hfill
    \begin{subfigure}{0.23\textwidth}
        \centering
        \includegraphics[width=\linewidth]{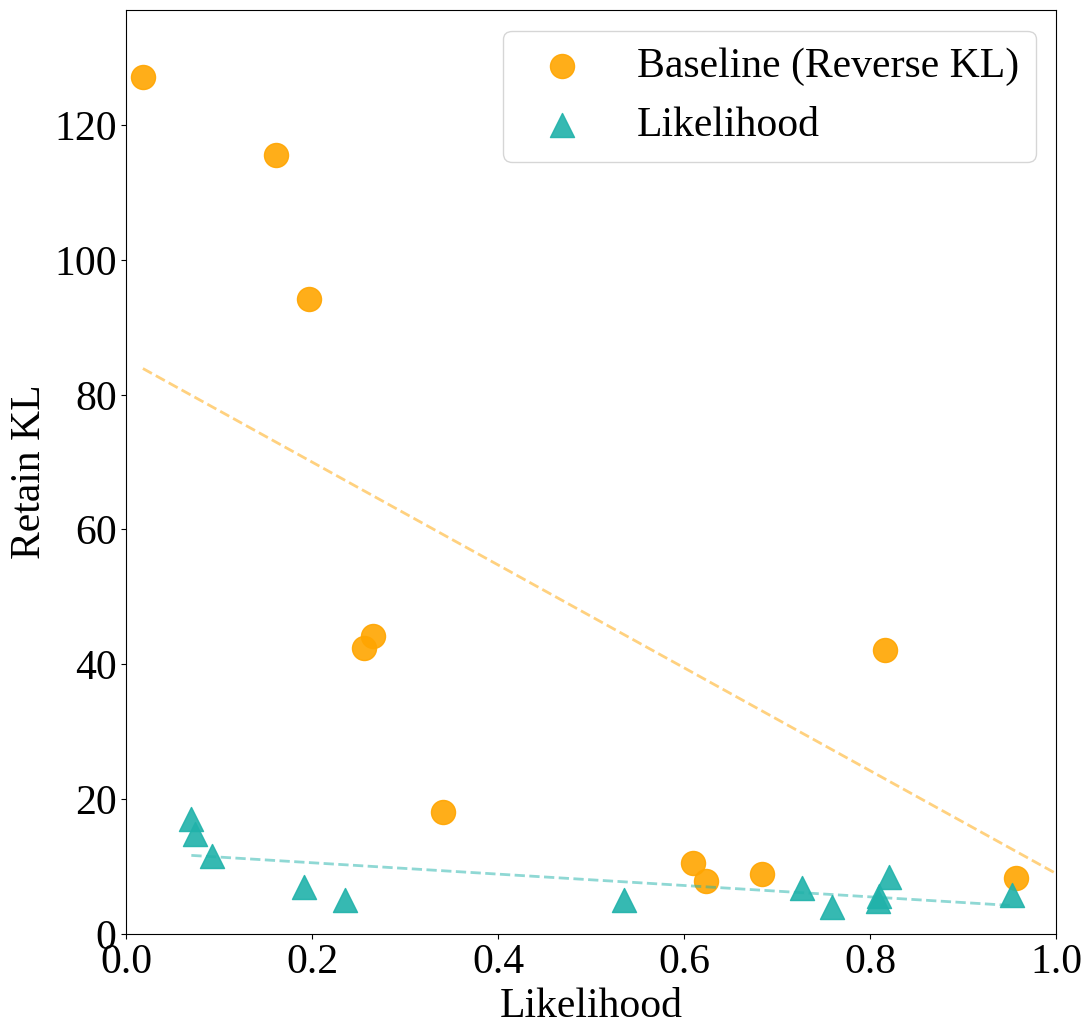}
    \end{subfigure}

    \caption{Performance of likelihood-constrained unlearning on a three-Gaussian mixture. From left to right: pretrained model, reverse-KL–constrained unlearning, likelihood-constrained unlearning, and the KL divergence to the retained model versus the likelihood of the unlearning mode. }
    \label{fig: gaussians_likelihood_vs_reverse}
\end{figure*}

\begin{figure*}[h]\label{fig: likelihood cowboys}
    \centering
    \begin{subfigure}{0.23\textwidth}
        \centering
        \includegraphics[width=\linewidth]{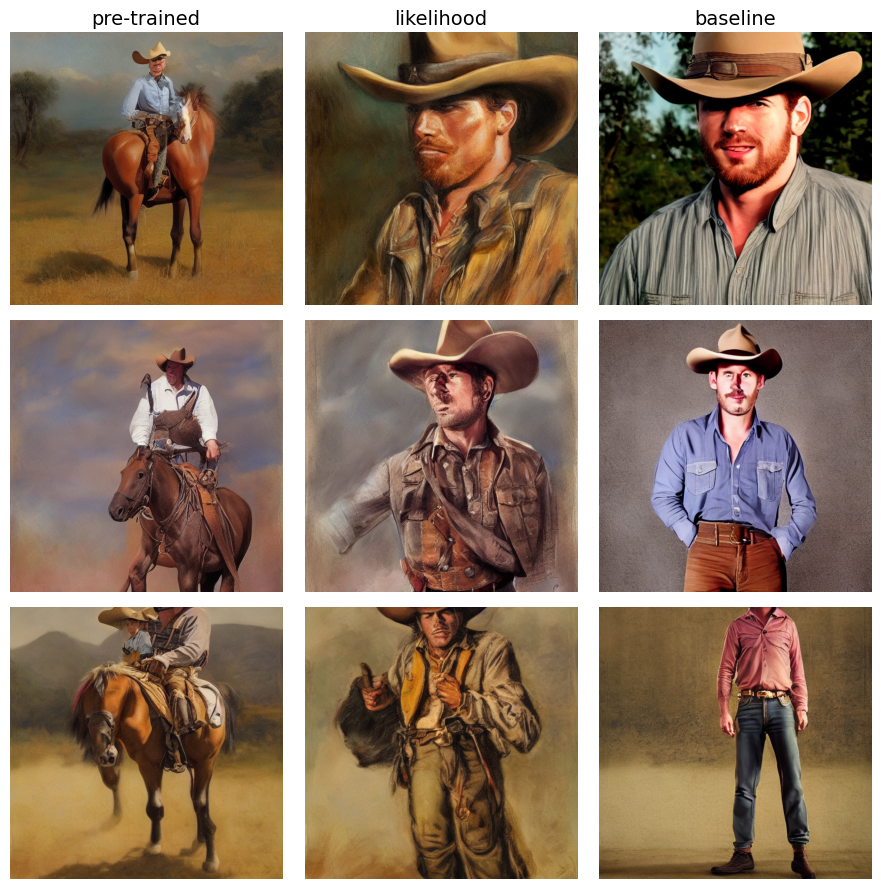}
    \end{subfigure}\hfill
    \begin{subfigure}{0.23\textwidth}
        \centering
        \includegraphics[width=\linewidth]{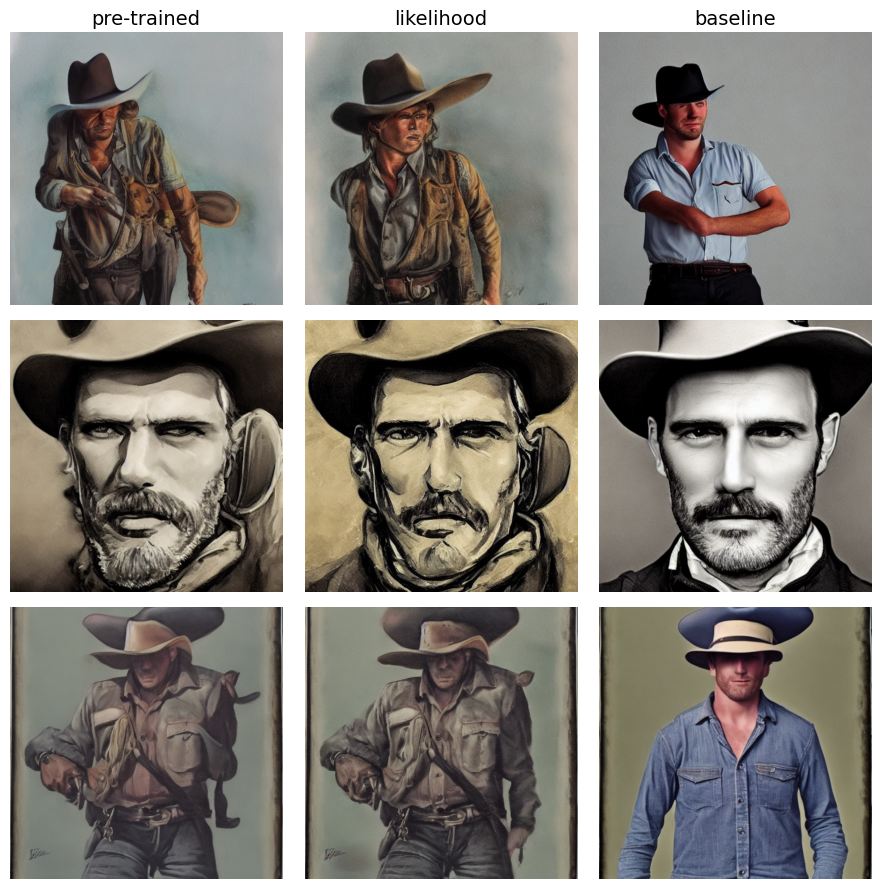}
    \end{subfigure}\hfill
    \begin{subfigure}{0.237\textwidth}
        \centering
        \includegraphics[width=\linewidth]{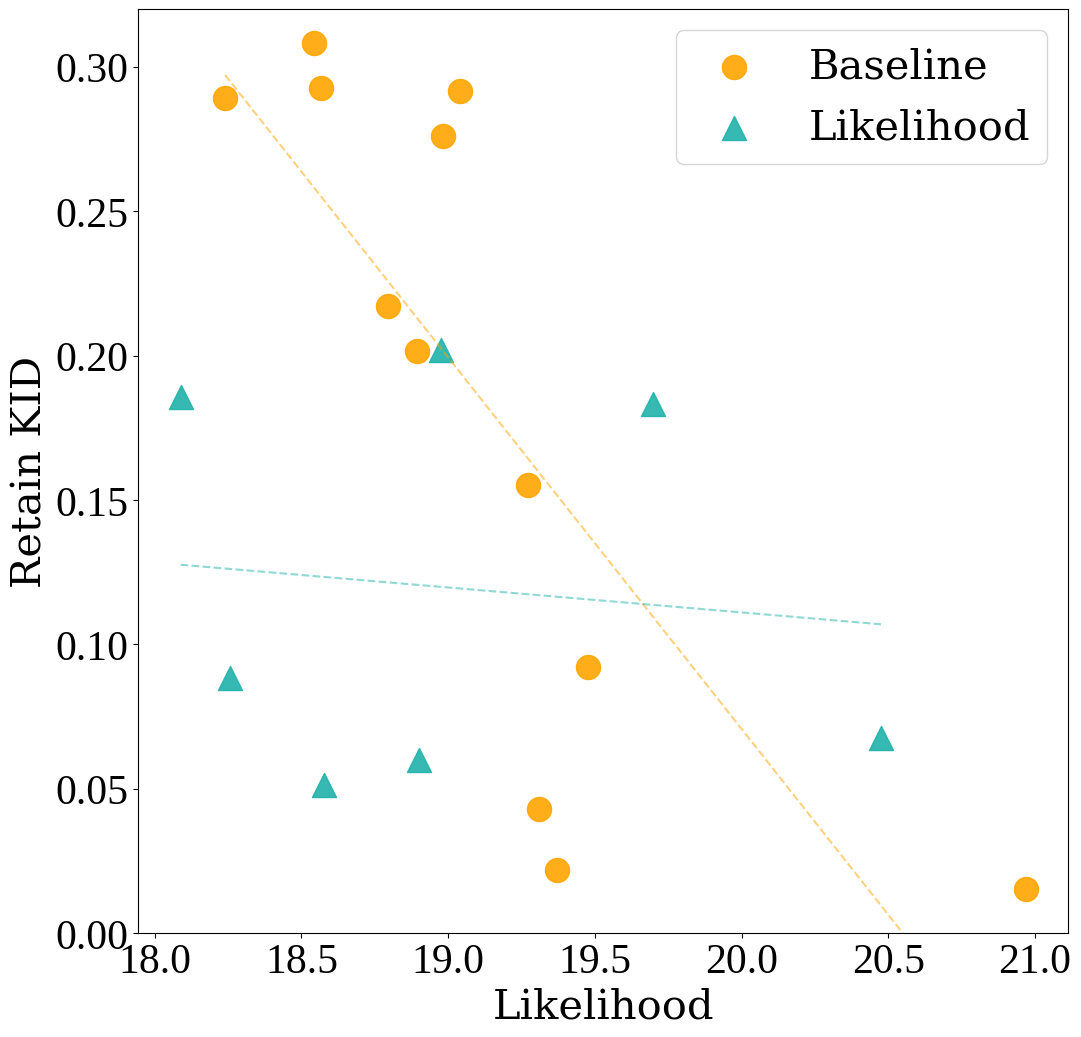}
    \end{subfigure}\hfill
    \begin{subfigure}{0.23\textwidth}
        \centering
        \includegraphics[width=\linewidth]{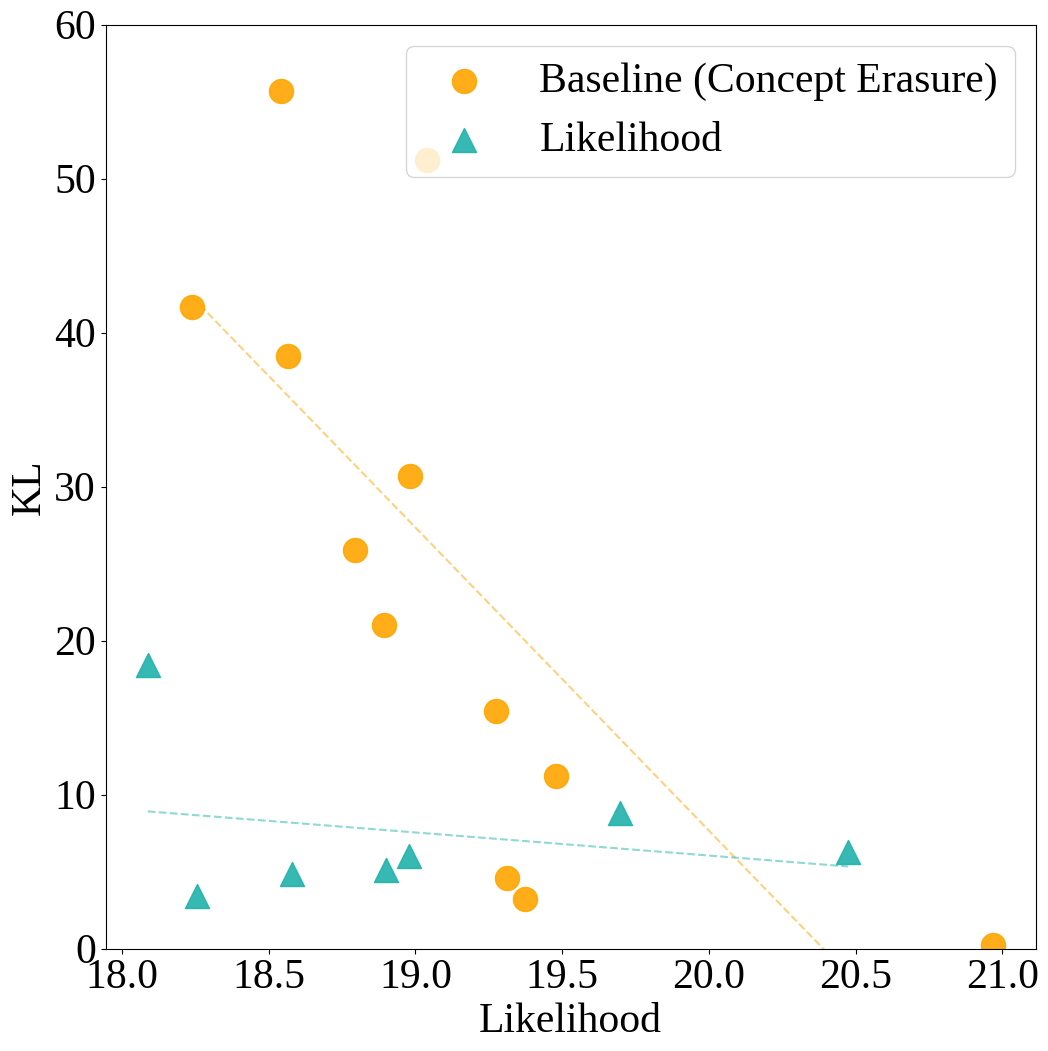}
    \end{subfigure}

        \caption{Performance of likelihood-constrained unlearning on a text-to-image model. From left to right: images generated by likelihood-constrained unlearning and baseline for unlearning the concept of `horse' (Left), for retaining the concept of `cowboy' (Middle left), the retention KID (Middle right) and the KL divergence (Right) to the retained model versus the likelihood of unlearning concept. }
    \label{fig: cowboys_likelihood_vs_reverse}
\end{figure*}

\subsection{Forward KL-Constrained Unlearning}\label{subsec: forward kl experiments}

To demonstrate the effectiveness of the forward KL constraints in~\eqref{eq: unlearning forward KL diffusion models_path-wise_score}, we work within the data unlearning paradigm~\cite{alberti2025data}, which seeks to prevent the model from generating specific samples that it may have memorized (e.g., for copyright reasons). We use the unconditional DDPM trained on the CelebA-HQ dataset as our pretrained model~\cite{ho2020denoisingdiffusionprobabilisticmodels}, and randomly select three images from the dataset to unlearn. As shown in Figure~\ref{fig: forward_images}, our constrained unlearning approach effectively removes the given samples in the left plot, while retaining the capability of generating other concepts, as shown in the right plot. This significantly improves upon the unconstrained baseline~\cite{alberti2025data}, which essentially solves the Lagrangian problem for Problem~\eqref{eq: unlearning forward KL diffusion models_path-wise_score} with equal weights assigned to all of the samples we wish to remove. We also show the retention–unlearning tradeoff in Figure~\ref{fig: forward pareto}, where the unlearning performance is measured by max SSCD~\cite{pizzi2022selfsuperviseddescriptorimagecopy} (Higher max SSCD indicates greater similarity to a forgotten sample and therefore weaker unlearning), the retention performance is measured by KID. We see that at the same level of unlearning, measured by max SSCD, our constrained unlearning approach deviates less from the pretrained model, as measured by KID. This is achieved by dynamically assigning large weights to the samples that are hard to unlearn and small weights to the easier samples.

\subsection{Reverse KL-Constrained Unlearning}\label{subsec: reverse kl experiments}

As in Section~\ref{subsec: Likelihood-Constrained Unlearning Experiments}, we consider concept unlearning by setting $q = p_{\text{pre}}(\cdot\,|\,c)$, and $q^i_{\text{u}} = p_{\text{pre}}(\cdot\,|\,c_i)$ in Problem~\eqref{eq: unlearning reverse KL diffusion models_point-wise_score}, where $c$ is a retained concept, and $c_i$ denotes other concepts or biases associated with the concept $c$ that we wish to remove. In Figure~\ref{fig: reverseKL_pareto}, we show the tradeoff between the worst-case unlearning performance (i.e., minimum KL divergences to unlearning concepts) and the retention performance (i.e., the KL divergence to the pretrained model), with each point representing a constraint threshold or regularization weight. We see that our constrained retention-unlearning tradeoff curve lies below the unconstrained baseline, achieving the same unlearning performance while deviating less from the pretrained model ({Figure~\ref{fig: reverseKL_pareto}, Left}). To confirm this, we also use the maximum CLIP text-to-image similarity score~\cite{hessel2022clipscorereferencefreeevaluationmetric} among the unlearning concepts, observing a similar result ({Figure~\ref{fig: reverseKL_pareto}, Right}).

\begin{figure}[ht]
    \centering
    \begin{subfigure}{0.23\textwidth}
        \centering
        \includegraphics[width=\linewidth]{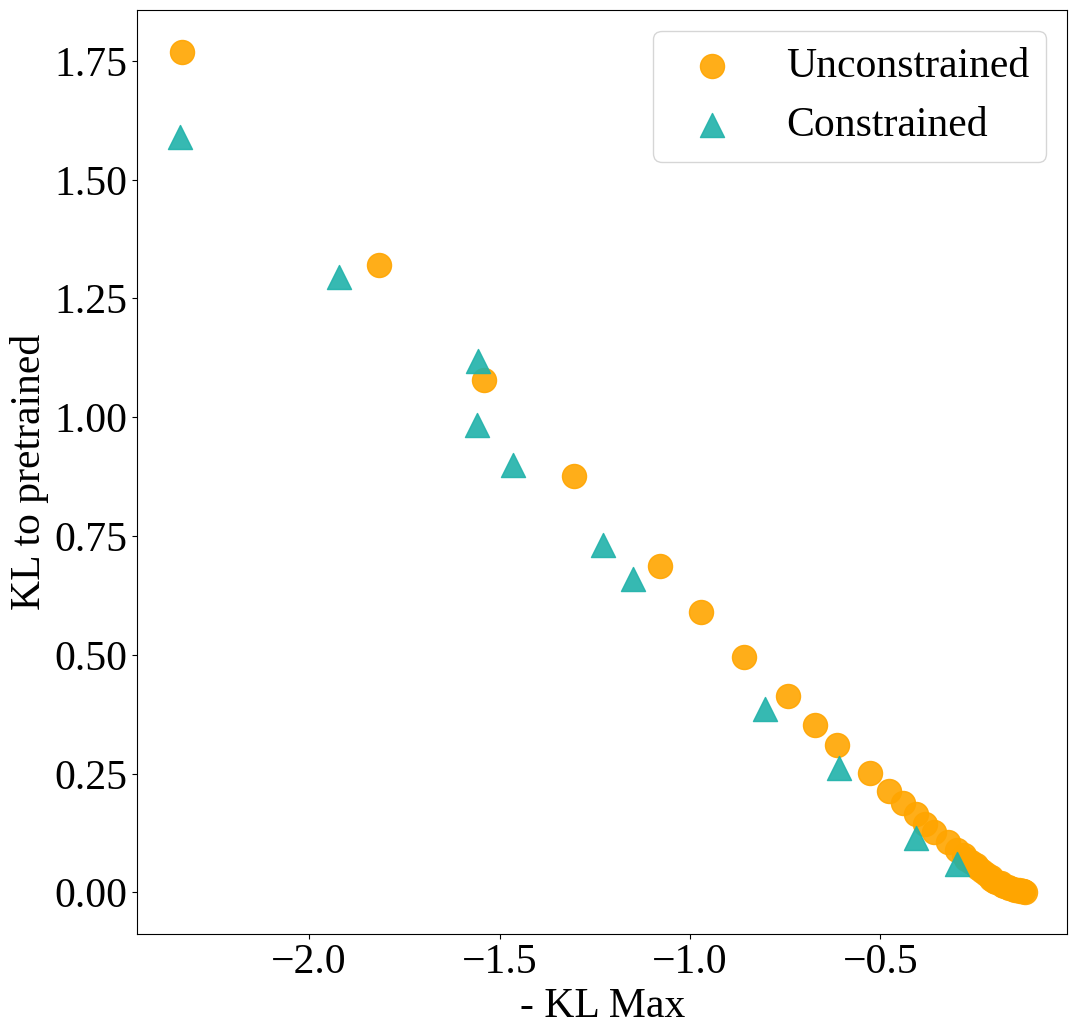}
    \end{subfigure}
    \hfill
    \begin{subfigure}{0.23\textwidth}
        \centering
        \includegraphics[width=\linewidth]{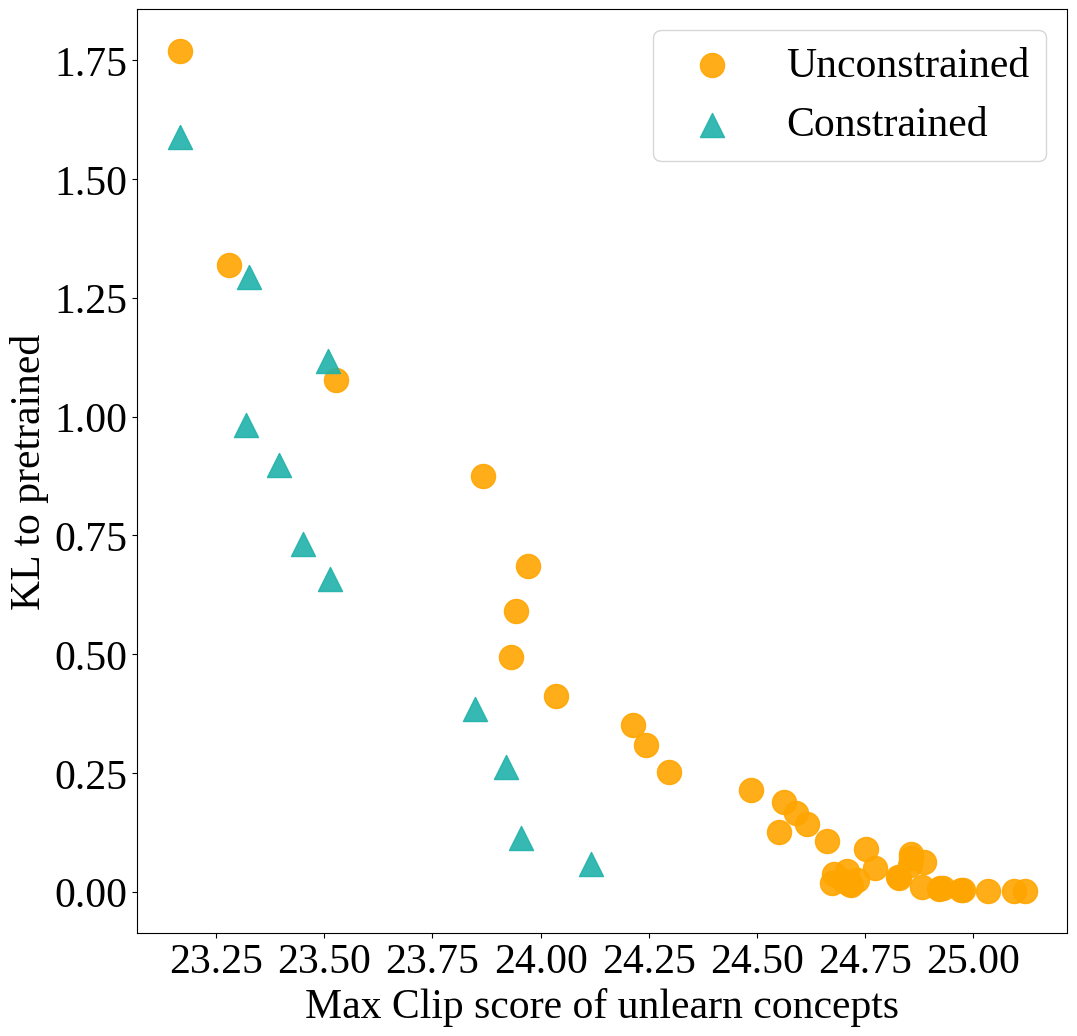}
    \end{subfigure}

    \caption{Retention–unlearning tradeoff for reverse KL-constrained unlearning and unconstrained baseline. The KL divergence to the pretrained model versus the KL constraint violation (Left) and the unlearning concept CLIP score (Right).}
    \label{fig: reverseKL_pareto}
\end{figure}

\section{Conclusion}
We have developed a constrained optimization framework that unifies unlearning as minimizing deviation from a pretrained model, subject to explicit separation constraints from the unlearning distributions. Despite the nonconvexity of these constraints, we prove that constrained unlearning has zero duality gap, which allows us to characterize optimal unlearning targets and develop primal–dual algorithms to approximate them. Empirically, we demonstrate our constrained unlearning approach on image generation tasks, showing a significantly-improved retention–unlearning tradeoff.

\textbf{Limitations}. Despite our theoretical characterizations and promising empirical results, additional experiments are needed to evaluate our method's effectiveness on unlearning tasks beyond image generation. Further theoretical study is required to establish convergence and sample-complexity guarantees for the proposed primal-dual algorithms.

%% file: appendices.tex
\section{Related Work}\label{app: related work}

\textbf{Unlearning in diffusion models.} Our constrained unlearning framework is related to prior work on unlearning in diffusion models. In this setting, given a pretrained diffusion model, unlearning aims to preserve the ability to generate diverse samples while avoiding the generation of samples that resemble specified samples or concepts to be forgotten. To handle these two conflicting objectives, existing approaches primarily adopt a classical regularization strategy, optimizing a weighted sum of retention and forgetting losses~\cite{golatkar2020eternal,fan2023salun,heng2023selective,park2024direct,alberti2025data,wu2025munba,park2025data,gao2025meta,li2025towards}. However, the weighted-sum objective has three key drawbacks: (i) the tradeoff between model utility and unlearning performance is tuned heuristically; (ii) the degree or likelihood of forgetting specific data or concepts is difficult to evaluate; and (iii) multiple unlearning objectives cannot be naturally incorporated within a single regularization term. In contrast, we formulate unlearning from the perspective of constrained distribution optimization. We treat the preservation of pretrained model utility as minimizing the distance between the trained and pretrained models, while expressing unlearning objectives as additional constraints that push the trained model away from specific data or concept distributions. This provides a more principled alternative to existing ad hoc approaches in~\citet{gandikota2023erasing,alberti2025data}. Our unlearning framework (i) offers a theoretical guarantee of an optimal
tradeoff between model utility retention and unlearning performance, and (ii) allows for
the direct imposition of multiple unlearning constraints. We also remark that our constrained unlearning approach is more tractable, both theoretically and practically, than the bilevel optimization approaches in~\citet{zhang2024defensive,chen2024score,wu2025erasing}, which often suffer from high computational costs and limited theoretical guarantees.

\textbf{Diffusion models under constraints.} Our work is also related to recent research on imposing constraints in diffusion models. Two types of constraints have been explored to control the samples generated by diffusion models: (i) hard constraints on individual samples~\cite{liang2024multi,liang2025simultaneous,christopher2024constrained,narasimhan2024constrained,zampini2025training,luan2025projected,christopher2025constrained} and (ii) soft constraints imposed in expectation~\cite{khalafi2024constrained,chamon2024constrained,khalafi2025composition}. Our constrained unlearning approach falls into the second category. In contrast to prior work~\cite{khalafi2024constrained,khalafi2025composition}, our constrained unlearning problem is nonconvex even in the distribution space, which constitutes the main challenge we address.  A closely related work is the constrained unlearning formulation of~\cite{feng2024controllable}, assuming a particular class of unlearning distributions. In contrast, our constrained unlearning framework is more general and does not impose such additional restrictions.

\newpage

\section{Proofs in Section~\ref{sec: unlearning in distribution space}}\label{app: proofs unlearning in distribution space}

\subsection{Proof of Theorem~\ref{thm: unlearning reverse KL strong duality}}\label{app: reverse kl strong duality}
\begin{proof}

    Define $f(p(x)) \DefinedAs p(x) \log \frac{p(x)}{q(x)}$. Since $p$, $q \ll \mu$, we can rewrite the reverse KL divergence $D_{\text{KL}}(p\,\Vert\, q)$ as
    \[
    D_{\text{KL}}(p\,\Vert\, q) 
    \; = \; 
    \int_{X} p(x) \log \frac{p(x)}{q(x)} \mu(dx)
    \; = \; 
    \mathbb{E}_\mu\left[\, f(p(x)) \,\right]
    \]
    where $\mu$: $\mathcal{B}(X) \to [0,\infty]$ is the standard Lebesgue measure. Similarly, we define $f_i(p(x)) \DefinedAs p(x) \log \frac{p(x)}{q_{\text{u}}^i(x)}$, and thus $D_{\text{KL}}(p\,\Vert\, q_{\text{u}}^i) = \mathbb{E}_\mu\left[\, f_i(p(x)) \,\right]$. Meanwhile, we explicitly express the probability constraint for  $p\in \Delta$ as
    \[
    \mathbb{E}_{\mu} \left[\,
    p(x)
    \,\right]
    \; \DefinedAs\;
    \int_X p(x) \mu(dx) 
    \; = \; 
    1.
    \]

    By relaxing Assumption~\ref{as: admissibility reverse KL}, we define a set of measures $\mathcal{P}$ that contains all probability measures that satisfy~\eqref{eq: admissible measures reverse KL}. To prove strong duality for Problem~\eqref{eq: unlearning reverse KL}, we prove it for an equivalent formulation of Problem~\eqref{eq: unlearning reverse KL},
    \begin{equation}\label{eq: unlearning reverse KL explicit}
    \begin{array}{rl}
        \displaystyle\minimize_{p\,\in\,\mathcal{P}} & 
        \mathbb{E}_\mu \left[\, f(p(x)) \,\right]
        \\[0.2cm]
        \subject &  \mathbb{E}_\mu \left[\, f_i(p(x)) \,\right]\; \geq \; b_i \;\; \text{ for } i \in [m]
        \\[0.2cm]
        & \displaystyle
        \mathbb{E}_{\mu} \left[\,
    p(x)
    \,\right] \; = \;1.
    \end{array}
\end{equation}
We note that the key difference between Problem~\eqref{eq: unlearning reverse KL} and Problem~\eqref{eq: unlearning reverse KL explicit} is the explicit constraint $\mathbb{E}_{\mu} \left[\,
    p(x)
    \,\right]  = 1$. Thus, they share the optimal solution $p_{\text{rev}}^\star$ and the optimal value $P_{\text{rev}}^\star$. We define the Lagrangian for Problem~\eqref{eq: unlearning reverse KL explicit} as $\hat{L}_{\text{rev}}(p, \lambda,\rho) \DefinedAs L_{\text{rev}}(p, \lambda) + \rho \left( \mathbb{E}_{\mu} \left[\,
    p(x)
    \,\right] - 1 \right)$. The associated dual function is $
\hat{D}_{\text{rev}}(\lambda,\rho)
= 
\minimize_{p\,\in\,\mathcal{P}} \hat{L}_{\text{rev}}(p, \lambda,\rho)$. Denote $(\hat\lambda_{\text{rev}}^\star,\hat\rho_{\text{rev}}^\star) \in \argmax_{\lambda\, \geq\,0,\, \rho} \hat{D}_{\text{rev}}(\lambda,\rho)$. We can verify that Problem~\eqref{eq: unlearning reverse KL explicit} is strongly dual, i.e.,
\[
    \hat{D}_{\text{rev}}^\star 
    \; \DefinedAs \; 
    \hat{D}_{\text{rev}} (\hat\lambda_{\text{rev}}^\star, \hat\rho_{\text{rev}}^\star)
    \; = \; 
    P_{\text{rev}}^\star
    \; \DefinedAs \; 
    \mathbb{E}_\mu \left[\, f(p_{\text{rev}}^\star(x)) \,\right].
\]
We provide a strong duality proof in Section~\ref{subsec: strong duality}.

Now, we show that Problem~\eqref{eq: unlearning reverse KL} is strongly dual, i.e., $D_{\text{rev}}^\star = \mathbb{E}_\mu \left[\, f(p_{\text{rev}}^\star(x)) \,\right]$. We note that weak duality always holds, i.e., $D_{\text{rev}}^\star \leq \mathbb{E}_\mu \left[\, f(p_{\text{rev}}^\star(x)) \,\right]$. It is sufficient to show $D_{\text{rev}}^\star \geq \hat{D}_{\text{rev}}^\star$. By the strong duality for Problem~\eqref{eq: unlearning reverse KL explicit}, we have
\[
\begin{array}{rcl}
     \hat{D}_{\text{rev}}^\star
     & = & \mathbb{E}_\mu \left[\, f(p_{\text{rev}}^\star(x)) \,\right]
     \\[0.2cm]
     & = & \displaystyle
     \mathbb{E}_\mu \left[\, f(p_{\text{rev}}^\star(x)) \,\right] \,+\, \sum_{i\,=\,1}^m \hat\lambda_{ \text{rev}, i}^\star \left(b_i - \mathbb{E}_\mu \left[\, f_i(p_{\text{rev}}^\star(x)) \,\right]\right) 
     \\[0.2cm]
     & = & \displaystyle \minimize_p\; L_{\text{rev}}(p; \hat\lambda_{\text{rev}}^\star)
     \\[0.2cm]
     & = & \displaystyle D_{\text{rev}}( \hat\lambda_{\text{rev}}^\star)
     \\[0.2cm]
     & \leq & \displaystyle \maximize_{\lambda\,\geq\,0} \; D_{\text{rev}}( \lambda)
     \\[0.2cm]
     & = & \displaystyle  D_{\text{rev}}^\star
\end{array}
\]
where the 2nd equality is due to that $\mathbb{E}_{\mu} \left[\,
    p_{\text{rev}}^\star(x)
    \,\right]  = 1$, the 3rd equality is due to that $(p_{\text{rev}}^\star,\hat{\lambda}_{\text{rev}}^\star)$ is a saddle point of $L_{\text{rev}}(p,\lambda)$ for $p\in\Delta$ and $\lambda\geq 0$, the 4th and 5th equalities follow the property of the dual function $D_{\text{rev}}(\lambda)$. Therefore, 
    \[
    D_{\text{rev}}^\star \; = \; \hat{D}_{\text{rev}}^\star \; = \; \mathbb{E}_\mu \left[\, f(p_{\text{rev}}^\star(x)) \,\right]
    \]
    which proves the strong duality for Problem~\eqref{eq: unlearning reverse KL}.
\end{proof}

\subsubsection{Strong duality for Problem~\eqref{eq: unlearning reverse KL explicit}}\label{subsec: strong duality}

The proof is an application of Lyapunov's convexity theorem in measure theory. We note that weak duality always holds, i.e., $\hat{D}_{\text{rev}}^\star \leq P_{\text{rev}}^\star$. Thus, it is sufficient to establish $\hat{D}_{\text{rev}}^\star \geq P_{\text{rev}}^\star$. We define an epigraph for Problem~\eqref{eq: unlearning reverse KL explicit} as
\[
    \Gamma
    \; \DefinedAs \; 
    \big\{ (\gamma, t) \,\vert\, \exists\, p \in \mathcal{P},  \mathbb{E}_\mu \left[\, f(p(x)) \,\right] \leq \gamma_0, b_i - \mathbb{E}_\mu \left[\, f_i(p(x)) \,\right] \leq \gamma_i \text{ for } i=1,\ldots,m, \text{ and }  \mathbb{E}_{\mu} \left[\,
    p(x)
    \,\right] - 1= \tau \big\}
\]
where $\gamma \DefinedAs [\,\gamma_0, \gamma_1,\ldots, \gamma_m\,]^\top$.

\begin{lemma}[Convexity]\label{lem: epigraph convexity}
    Let Assumption~\ref{as: feasibility reverse KL} hold. Then, the epigraph $\Gamma$ is non-empty and convex.
\end{lemma}

\begin{proof}
It is straightforward that $\Gamma$ is non-empty since Problem~\eqref{eq: unlearning reverse KL} is feasible. We next show that $\Gamma$ is a convex set. Assume that $(\gamma, \tau)$, $(\gamma', \tau') \in \Gamma$, i.e., there exist $p$, $p' \in \mathcal{P}$ such that
\begin{subequations}\label{eq: convexity assumption}
\begin{equation}
    \mathbb{E}_\mu \left[\, f(p(x)) \,\right] \leq \gamma_0, b_i - \mathbb{E}_\mu \left[\, f_i(p(x)) \,\right] \leq \gamma_i \text{ for } i=1,\ldots,m, \text{ and }  \mathbb{E}_{\mu} \left[\,
    p(x)
    \,\right] - 1= \tau
\end{equation}
\begin{equation}
    \mathbb{E}_\mu \left[\, f(p'(x)) \,\right] \leq \gamma_0', b_i - \mathbb{E}_\mu \left[\, f_i(p'(x)) \,\right] \leq \gamma_i' \text{ for } i=1,\ldots,m, \text{ and }  \mathbb{E}_{\mu} \left[\,
    p'(x)
    \,\right] - 1= \tau'.
\end{equation}
\end{subequations}
It is sufficient to show that $\alpha (\gamma,\tau) + (1-\alpha) (\gamma', \tau') \in \Gamma$ for $\alpha\in [0,1]$, i.e., there exists $p_\alpha \in \mathcal{P}$ such that
\begin{subequations}\label{eq: convexity existence}
\begin{equation}\label{eq: convexity existence a}
\mathbb{E}_\mu \left[\, f(p_\alpha(x)) \,\right] \; \leq \; \alpha \gamma_0 + (1-\alpha) \gamma_0'
\end{equation}
\begin{equation}\label{eq: convexity existence b}
    b_i - \mathbb{E}_\mu \left[\, f_i(p_\alpha(x)) \,\right] 
    \; \leq \; \alpha \gamma_i +(1-\alpha) \gamma_i' \; \text{ for } i=1,\ldots,m
\end{equation}
\begin{equation}\label{eq: convexity existence c}
\mathbb{E}_{\mu} \left[\,
    p_\alpha(x)
    \,\right] - 1
    \; = \; \alpha \tau +(1-\alpha) \tau'.
\end{equation}
\end{subequations}
To do so, we construct a vector measure $\mathfrak{p}$: $\mathcal{B}(X) \to \mathbb{R}^{2(m+2)}$,
\[
    \mathfrak{p}(E) 
    \; = \;
    \left[
    \begin{tabular}{c}
         $\int_E f(p(x)) \mu(dx)$
         \\[0.2cm]
         $\int_E f_1(p(x)) \mu(dx)$
         \\[0.2cm]
         $\vdots$
         \\[0.2cm]
         $\int_E  f_m(p(x)) \mu(dx) $
         \\[0.2cm]
         $\int_E p(x) \mu(dx)$
         \\[0.2cm]
         $\int_E f(p'(x)) \mu(dx)$
         \\[0.2cm]
         $\int_E f_1(p'(x)) \mu(dx)$
         \\[0.2cm]
         $\vdots$
         \\[0.2cm]
         $\int_E  f_m(p'(x)) \mu(dx)$
         \\[0.2cm]
         $\int_E p'(x) \mu(dx)$
    \end{tabular}
   \right]
\]
where each entry of $\mathfrak{p}$ is a Lebesgue integral. Clearly, $\mathfrak{p}(\emptyset) = 0$ and it is non-atomic. When $p\ll q$, $q_{\text{u}}^i$ (Assumption~\ref{as: admissibility reverse KL}), $\mathfrak{p}$ is a proper vector measure, i.e., $\mathfrak{p}(X) < \infty$. We note that $\mathfrak{p}(X) = \infty$ if and only if $p$ is not absolutely continuous with respect to $q$ or $q_{\text{u}}^i$, which is an infeasible case of Problem~\eqref{eq: unlearning reverse KL explicit}. According to a weak version of Lyapunov's convexity theorem~\cite{olech1968range}, there exists a set $E_\alpha \in \mathcal{B}(X)$ such that 
\[
    \mathfrak{p}(E_\alpha) \; = \;
    \alpha \mathfrak{p}(X) + (1-\alpha )\mathfrak{p}(\emptyset)
    \; = \;
    \alpha \mathfrak{p}(X)
\]
for $\alpha \in [0,1]$. By additivity, $\mathfrak{p}(X/E_\alpha) = (1-\alpha)\mathfrak{p}(X)$. Hence, we construct $p_\alpha \in \mathcal{P}$ as follows,
\[
p_\alpha(x) \; = \;
\begin{cases}
    p(x) \;\;\;\; \text{ when } x \in E_\alpha
    \\[0.2cm]
    p'(x) \;\;\; \text{ when } x \in X/E_\alpha.
\end{cases}
\]
It is ready to verify that $p_\alpha$ satisfies~\eqref{eq: convexity existence} regarding three conditions as follows. 
\begin{itemize}
    \item[(i)] Condition~\eqref{eq: convexity existence a}. By the construction of $p_\alpha$,
    \[
    \begin{array}{rcl}
         \mathbb{E}_\mu \left[\, f(p_\alpha(x)) \,\right]
    & = & \displaystyle
    \int_{E_\alpha} f(p(x)) \mu(dx)
    + 
    \int_{X/E_\alpha} f(p'(x)) \mu(dx)
    \\[0.2cm]
    & = & \displaystyle
    \mathfrak{p}_1(E_\alpha) + \mathfrak{p}_{m+3}(X/E_\alpha)
    \\[0.2cm]
    & = & \displaystyle
    \alpha\mathfrak{p}_1(X) + (1-\alpha)\mathfrak{p}_{m+3}(X)
    \\[0.2cm]
    & = & \displaystyle
    \alpha\mathbb{E}_\mu \left[\, f(p(x)) \,\right] + (1-\alpha)\mathbb{E}_\mu \left[\, f(p'(x)) \,\right].
    \end{array}
    \]
    By the assumption~\eqref{eq: convexity assumption},
    \[
    \mathbb{E}_\mu \left[\, f(p_\alpha(x)) \,\right]
    \; \leq \; \alpha \gamma_0 + (1-\alpha) \gamma_0'.
    \]
    \item[(ii)] Condition~\eqref{eq: convexity existence b}. Similarly, we have
    \[
    \mathbb{E}_\mu \left[\, f_i(p_\alpha(x)) \,\right] 
    \; = \;
     \alpha\mathbb{E}_\mu \left[\, f_i(p(x)) \,\right] + (1-\alpha)\mathbb{E}_\mu \left[\, f_i(p'(x)) \,\right]
    \]
     By the assumption~\eqref{eq: convexity assumption},
     \[
     b_i - \mathbb{E}_\mu \left[\, f_i(p_\alpha(x)) \,\right] 
    \; \leq \; \alpha \gamma_i +(1-\alpha) \gamma_i'.
     \]
     \item[(iii)] Condition~\eqref{eq: convexity existence c}. Similarly, we have
     \[
     \mathbb{E}_{\mu} \left[\,
    p_\alpha(x)
    \,\right] - 1
    \; = \; \alpha \tau +(1-\alpha) \tau'.
     \]
\end{itemize}
Therefore, there exists $p_\alpha \in \mathcal{P}$ such that~\eqref{eq: convexity existence} holds for $\alpha\in [0, 1]$, $(\gamma, \tau)$, $(\gamma', \tau') \in \Gamma$.
\end{proof}

With Lemma~\ref{lem: epigraph convexity}, we next show strong duality for Problem~\eqref{eq: unlearning reverse KL explicit}. First of all, it is easy to check that $(P_{\text{rev}}^\star, 0, 0)$ is not an interior point of $\Gamma$; otherwise, there always exists $\epsilon>0$ such that $(P_{\text{rev}}^\star-\epsilon, 0, 0) \in \Gamma$, which contradicts the optimality of $P_{\text{rev}}^\star$. By the supporting hyperplane theorem, there exists a non-zero vector $(y, z) \in \mathbb{R}^{m+2}$ such that
\begin{equation}\label{eq: supporting hyperplane}
\sum_{i\,=\,0}^{m}y_i \gamma_i + z \tau \; \geq \; y_0 P_{\text{rev}}^\star \; \text{ for } (\gamma, \tau) \in \Gamma.
\end{equation}
We note that $\Gamma$ is unbounded above, i.e., for any $(\gamma, \tau) \in \Gamma$, there exists $(\gamma', \tau) \in \Gamma$ such that $\gamma'\geq \gamma$ elementwise. To ensure~\eqref{eq: supporting hyperplane}, we must require $y_i\geq0$  for all $i$; otherwise, for some $y_i<0$, we can set the left-hand side of the inequality~\eqref{eq: supporting hyperplane} to be $-\infty$ by increasing $\gamma_i$. Furthermore, we can show that $y_0 > 0$; otherwise, $y_0=0$ simplifies~\eqref{eq: supporting hyperplane} to be $\sum_{i\,=\,1}^m y_i \gamma_i + z \tau \geq 0$ for $(\gamma, \tau) \in\Gamma$. This contradicts the strict feasibility assumption, since setting $(\gamma, \tau)$ to be the strict feasible point $\bar{p}$ concludes that $\gamma_i<0$ for all $i$, and $t=0$. Hence, $y_0 > 0$, and thus,
\[
\gamma_0 + \sum_{i\,=\,1}^m \bar{y}_i \gamma_i + \bar{z} \tau\; \geq \; P_{\text{rev}}^\star \; \text{ for } (\gamma, \tau) \in \Gamma
\]
where $\bar{y}_i \DefinedAs \frac{y_i}{y_0}$ and $\bar{z} \DefinedAs \frac{z}{y_0}$. It further implies that
\[
\mathbb{E}_\mu \left[\, f(p(x)) \,\right] + \sum_{i\,=\,1}^m \bar{y}_i \left( b_i - \mathbb{E}_\mu \left[\, f_i(p(x)) \,\right]\right) + \bar{z} \left(  \mathbb{E}_\mu \left[\,p(x) \,\right] -1 \right)
\; \geq \; P_{\text{rev}}^\star
\]
or, equivalently, $\hat{L}_{\text{rev}}(p; \bar{y}, \bar{z}) \geq P_{\text{rev}}^\star$ for $p\in\mathcal{P}$ and $\bar{y}\geq 0$. Thus, minimization of $\hat{L}_{\text{rev}}(p; \bar{y}, \bar{z})$ over $p\in\mathcal{P}$ leads to $\hat{D}_{\text{rev}}(\bar{y},\bar{z}) \geq P_{\text{rev}}^\star$, and thus $\hat{D}_{\text{rev}}^\star \geq P_{\text{rev}}^\star$.

\subsection{Proof of Corollary~\ref{cor: optimal distributions reverse KL}}\label{app: optimal distributions reverse KL}

\begin{proof}
    When $1-\one^\top\lambda  \neq 0$, we rearrange the Lagrangian $L_{\text{rev}}(p, \lambda)$ as follows,
\[
    \begin{array}{rcl}
         L_{\text{rev}}(p, \lambda)
        & = & 
        \displaystyle
        D_{\text{KL}}(p\,\Vert\, q)
        +
        \sum_{i\,=\,1}^m \lambda_i \left(
        b_i 
        - D_{\text{KL}}(p\,\Vert\, q_{\text{u}}^i)
        \right)
        \\[0.2cm]
        & = & 
        \displaystyle
        \mathbb{E}_{x\,\sim\,p}[ \,\log p(x) - \log q(x)\,]
        + 
        \lambda^\top b
        - 
        \one^\top\lambda  \mathbb{E}_{x\,\sim\,p}[ \,\log p(x)\,]
        +
        \sum_{i\,=\,1}^m \lambda_i \mathbb{E}_{x\,\sim\,p}[ \,\log q_{\text{u}}^i(x)\,]
        \\[0.2cm]
        & = & 
        \displaystyle
        \lambda^\top b
        +
        ( 1 - 
        \one^\top\lambda ) \mathbb{E}_{x\,\sim\,p}[ \,\log p(x)\,]
        -
        \left( 
        \mathbb{E}_{x\,\sim\,p}[ \,\log q(x)\,]
        -
        \sum_{i\,=\,1}^m \lambda_i \mathbb{E}_{x\,\sim\,p}[ \,\log q_{\text{u}}^i(x)\,]
        \right) 
        \\[0.2cm]
        & = & \displaystyle
        \lambda^\top b
        + 
        (1 - \one^\top\lambda) 
        \left(
            \mathbb{E}_{x\,\sim\,p}[ \,\log p(x)\,]
            - 
            \mathbb{E}_{x\,\sim\,p}\left[ \,\log \frac{(q(x))^{1/(1-\one^\top\lambda)}}{\prod_{i\,=\,1}^m(q_{\text{u}}^i(x))^{\hat\lambda_i}}\,\right]
        \right)
        \\[0.2cm]
        & = & \displaystyle
        \lambda^\top b
        + 
        (1-\one^\top\lambda)
        \left( D_{\text{KL}}(p(\cdot)\,\Vert\, q^{\dagger}_{\text{\normalfont u}}(\cdot;\lambda) ) 
        -
        \log Z^{\dagger}_{\text{\normalfont 
    u}}(\lambda)
    \right)
    \end{array}
\]
where the 4th equality abbreviates $\hat\lambda_i \DefinedAs \frac{\lambda_i}{1-\one^\top\lambda}$ for $i=1,\ldots,m$, and the last equality is due to that  $q_{\text{u}}^{\dagger}$ is well-defined for $q$ and $q_{\text{u}}^i$ that share the same support. When $1-\one^\top\lambda = 0$, it is straightforward to compute 
\[
\begin{array}{rcl}
     L_{\text{rev}}(p, \lambda) 
    & = & \displaystyle
    \lambda^\top b - \mathbb{E}_{x\,\sim\,p} \left[ \log \frac{q(x)}{\prod_{i\,=\,1}^m (q_{\text{u}}^i(x))^{\lambda_i}}\right]
    \\[0.2cm]
    & = & \displaystyle
    \one^\top b - \mathbb{E}_{x\,\sim\,p} \left[ \log q^{\ddagger}_{\text{u}}(x;\lambda) \right] - \log Z^{\ddagger}_{\text{u}}(\lambda)
\end{array}
\]
and simplify the unlearning target $q_{\text{\normalfont u}}^{\dagger}$ as
    \[
    q_{\text{\normalfont u}}^{\ddagger} (\cdot;\lambda)
    \; \DefinedAs \;
    \frac{1}{ Z^{\ddagger}_{\text{\normalfont 
    u}}(\lambda)}
    \frac{q(\cdot)}{\prod_{i\,=\,1}^m(q_{\text{\normalfont u}}^i(\cdot))^{\lambda_i}} 
\]
where $Z_{\text{\normalfont 
    u}}^{\ddagger}(\lambda) \DefinedAs \int {q(x)}/{\prod_{i\,=\,1}^m(q_{\text{\normalfont u}}^i(x))^{\lambda_i}} dx$.

To evaluate the dual function $D_{\text{rev}}(\lambda) \DefinedAs \minimize_{p\,\in\,\Delta} L_{\text{rev}}(p, \lambda)$, there are three cases. 
\begin{itemize}
    \item[(i)] When $0\leq \one^\top\lambda <1$, minimization of $L_{\text{rev}}(p, \lambda)$ over $p$ is at a unique minimizer,  
    \[
    p_{\text{rev}}^\star (\cdot; \lambda) 
    \; = \;
   q_{\text{u}}^{\dagger}(\cdot; \lambda)
    \]
    which defines the dual function $D_{\text{rev}}(\lambda) = \lambda^\top b - (1-\one^\top\lambda) \log Z^{\dagger}_{\text{u}}(\lambda)$. 
    
    \item[(ii)] When $ \one^\top\lambda =1$, minimization of $L_{\text{rev}}(p, \lambda)$ over $p$ is at $X(\lambda) \DefinedAs\argmax_{x}{q(x)}/{\prod_{i\,=\,1}^m (q_{\text{u}}^i(x))^{\lambda_i}}$. Thus,
    \[
    p_{\text{rev}}^\star(\cdot ;\lambda) \; = \;
    \frac{1}{q^{\ddagger}_{\text{u}}(X(\lambda);\lambda)}\delta_{X(\lambda)}(\cdot).
    \]
    Hence, 
    \[
    D_{\text{rev}}(\lambda) \; = \; \one^\top b -  \frac{1}{q^{\ddagger}_{\text{u}}(X(\lambda);\lambda)}\int_{X(\lambda)} \log q_{\text{u}}^{\ddagger}(x ;\lambda) dx - \log Z_{\text{u}}^{\ddagger}(\lambda).
    \]
    
    \item[(iii)] When $\one^\top\lambda >1$, minimization of $L_{\text{rev}}(p, \lambda)$ over $p$ is equivalent to 
    \[
    \maximize_{p\,\in\,\Delta} \; 
    D_{\text{KL}}(p\,\Vert\, q^{\dagger}_{\text{\normalfont u}} ).
    \]
    We note that when $p\in\Delta$ does not satisfy Assumption~\ref{as: admissibility reverse KL}, $D_{\text{KL}}(p\,\Vert\, q^{\dagger}_{\text{\normalfont u}} ) = \infty$. Hence, $D_{\text{rev}}^\star = -\infty$, which contradicts to Assumption~\ref{as: feasibility reverse KL}.
\end{itemize}

Hence, we can summarize them as follows. When $\one^\top\lambda \leq 1$ and $\lambda\geq 0$, the dual function $D_{\text{rev}}(\lambda)$ has an explicit form,
\[
    D_{\text{rev}}(\lambda)
    \; =\;
    \begin{cases}
        \displaystyle\lambda^\top b - (1-\one^\top \lambda)
        \log Z^{\dagger}_{\text{\normalfont 
    u}}(\lambda) 
    \;\;\;\;  \;\;\;\;  
    \;\;\;\;  \;\;\;\;
    \;\;\;\;  \;\;\;\;
    \;\;\;\;  \;\;\;\;
    \;\;\;\;  \;\;\;\;
    \;\;\;\;  \;\;\;
    \text{\normalfont when } 0 \leq \one^\top\lambda < 1 \text{ and }  \lambda\geq 0
    \\
    \displaystyle\one^\top b -  \frac{1}{q^{\ddagger}_{\text{u}}(X(\lambda);\lambda)}\int_{X(\lambda)} \log q_{\text{u}}^{\ddagger}(x ;\lambda) dx - \log Z_{\text{u}}^{\ddagger}(\lambda)
    \;\;\;\;  \;\;\;\;
    \;
    \text{\normalfont when } \one^\top\lambda = 1 \text{ and } \lambda\geq 0
    \end{cases}
    \]
    where $X(\lambda) \DefinedAs\argmax_{x}{q(x)}/{\prod_{i\,=\,1}^m (q_{\text{u}}^i(x))^{\lambda_i}}$. 
    The Lagrangian function $L_{\text{rev}}(p;\lambda)$ has an explicit form,
    \[
    L_{\text{rev}}(p, \lambda) \; = \;
    \begin{cases}
        \lambda^\top b
        + 
        (1-\one^\top\lambda)
        \left( D_{\text{\normalfont KL}}(p\,\Vert\, q^{\dagger}_{\text{\normalfont u}} ) 
        -
        \log Z^{\dagger}_{\text{\normalfont 
    u}}(\lambda)
    \right)
    \;\;\;\;  \;\;\;\; 
    \;\;
    \text{\normalfont when } \one^\top\lambda \neq 1 \text{ and } \lambda\geq 0
    \\[0.2cm]
    \one^\top b - \mathbb{E}_{x\,\sim\,p} \left[ \log q^{\ddagger}_{\text{\normalfont u}}(x;\lambda) \right] - \log Z^{\ddagger}_{\text{\normalfont u}}(\lambda)
    \;\;\;\;  \;\;\;\; 
    \;\;\;\;  \;\;\;\;
    \;\;\;
    \text{\normalfont when } \one^\top\lambda = 1 \text{ and } \lambda\geq 0.
    \end{cases}
    \]
    A partial minimizer of the Lagrangian $L_{\text{rev}}(p; \lambda)$ over $p$ is achieved at 
    \[
    p_{\text{rev}}^\star(\cdot;\lambda) 
    \; = \;
    \begin{cases}
    \displaystyle 
    q_{\text{\normalfont u}}^{\dagger}(\cdot; \lambda)
    \;\;\;\;  \;\;\;\;
    \;\;\;\;  \;\;\;\;
    \;\;\;\;  \;\;\;\;
    \;\;\;\;  \;\;\;\;
    \text{\normalfont when } 0 \leq \one^\top\lambda < 1 \text{ and }  \lambda\geq 0
    \\[0.2cm]
    \displaystyle 
    \frac{1}{q^{\ddagger}_{\text{\normalfont u}}(X(\lambda);\lambda)}\delta_{X(\lambda)}(\cdot)
    \;\;\;\;  \;\;\;\;
    \;\;\;\;  
    \text{\normalfont when } \one^\top\lambda = 1 \text{ and }  \lambda\geq 0.
    \end{cases}
    \]
    
    Finally, we apply the strong duality in Theorem~\ref{thm: unlearning reverse KL strong duality} to evaluate $p_{\text{rev}}^\star(\cdot) = p_{\text{rev}}^\star(\cdot;\lambda_{\text{rev}}^\star)$. 
    When $\one^\top \lambda_{\text{rev}}^\star = 1$, the solution $p_{\text{rev}}^\star$ is a delta function, which is not practically attainable.
    This case can always be avoided by scaling $\lambda_{\text{rev}}^\star$, allowing us to focus on $0 \le \mathbf{1}^\top \lambda_{\text{rev}}^\star < 1$. 
\end{proof}

\subsection{Proof of Theorem~\ref{thm: unlearning forward KL strong duality}}\label{app: unlearning forward KL strong duality}

\begin{proof}
    Since the proof is similar to the proof of Theorem~\ref{thm: unlearning reverse KL strong duality}, we omit the detail and only state the key changes. Define $f(p(x)) \DefinedAs q(x) \log \frac{q(x)}{p(x)}$. Since $p$, $q \ll \mu$, we can rewrite the forward KL divergence $D_{\text{KL}}(q\,\Vert\, p)$ as
    \[
    D_{\text{KL}}(q\,\Vert\, p) 
    \; = \; 
    \int_{X} q(x) \log \frac{q(x)}{p(x)} \mu(dx)
    \; = \; 
    \mathbb{E}_\mu\left[\, f(p(x)) \,\right]
    \]
    where $\mu$: $\mathcal{B}(X) \to [0,\infty]$ is the standard Lebesgue measure. Similarly, we define $f_i(p(x)) \DefinedAs q_{\text{u}}^i(x) \log \frac{q_{\text{u}}^i(x)}{p(x)}$, and thus $D_{\text{KL}}(q_{\text{u}}^i\,\Vert\, p) = \mathbb{E}_\mu\left[\, f_i(p(x)) \,\right]$. Meanwhile, we explicitly express the probability constraint for  $p\in \Delta_X$ as
    \[
    \mathbb{E}_{\mu} \left[\,
    p(x)
    \,\right]
    \; \DefinedAs\;
    \int_X p(x) \mu(dx) 
    \; = \; 
    1.
    \]

    With a slight abuse of notation, we also denote the set of all measures that satisfy~\eqref{eq: admissible measures forward KL} by $\mathcal{P}$. 
    To prove strong duality for Problem~\eqref{eq: unlearning forward KL}, we prove it for an equivalent formulation of Problem~\eqref{eq: unlearning forward KL},
    \begin{equation}\label{eq: unlearning forward KL explicit}
    \begin{array}{rl}
        \displaystyle\minimize_{p\,\in\,\mathcal{P}} & 
        \mathbb{E}_\mu \left[\, f(p(x)) \,\right]
        \\[0.2cm]
        \subject &  \mathbb{E}_\mu \left[\, f_i(p(x)) \,\right]\; \geq \; b_i \;\; \text{ for } i \in [m]
        \\[0.2cm]
        & \displaystyle
        \mathbb{E}_{\mu} \left[\,
    p(x)
    \,\right] \; = \;1.
    \end{array}
\end{equation}

The rest follows the same proof steps in the proof of Theorem~\ref{thm: unlearning reverse KL strong duality}. First, we employ Lyapunov's convexity theorem to show that the epigraph for Problem~\eqref{eq: unlearning forward KL explicit} is non-empty and convex, implying that~\eqref{eq: unlearning forward KL explicit} is strongly dual. Second, we translate the strong duality for Problem~\eqref{eq: unlearning forward KL explicit} to the original problem~\eqref{eq: unlearning forward KL}.
\end{proof}

\subsection{Proof of Corollary~\ref{cor: optimal distributions forward KL}}\label{app: optimal distributions forward KL}

\begin{proof}
    For any $\lambda\geq0$, we rearrange the Lagrangian $L_{\text{fw}}(p,\lambda)$ as follows,
    \[
    \begin{array}{rcl}
         L_{\text{fw}}(p,\lambda) 
         & = & 
         \displaystyle
        D_{\text{KL}}(q\,\Vert\, p)
        +
        \sum_{i\,=\,1}^m \lambda_i \left(
        b_i 
        - D_{\text{KL}}(q_{\text{u}}^i\,\Vert\, p)
        \right)
        \\[0.2cm]
        & = & 
        \displaystyle\mathbb{E}_{x\,\sim\,q} \left[ \log q(x)\right] + \sum_{i\,=\,1}^m \lambda_i (b_i - \mathbb{E}_{x\,\sim\,q_{\text{u}}^i} \left[ \log q_{\text{u}}^i(x)\right])
        - \int \left(q(x) - \sum_{i\,=\,1}^m \lambda_i q_{\text{u}}^i(x)\right)  \log p(x)
        \\[0.2cm]
        & = & 
        \displaystyle\mathbb{E}_{x\,\sim\,q} \left[ \log q(x)\right] + \sum_{i\,=\,1}^m \lambda_i (b_i - \mathbb{E}_{x\,\sim\,q_{\text{u}}^i} \left[ \log q_{\text{u}}^i(x)\right])
        - Z_{\text{u}}^\triangleleft(\lambda)\int q_{\text{u}}^\triangleleft(x;\lambda)\log p(x)
        \\[0.2cm]
        & = & 
        \displaystyle\mathbb{E}_{x\,\sim\,q} \left[ \log q(x)\right] + \sum_{i\,=\,1}^m \lambda_i (b_i - \mathbb{E}_{x\,\sim\,q_{\text{u}}^i} \left[ \log q_{\text{u}}^i(x)\right]) - Z_{\text{u}}^\triangleleft(\lambda) \mathbb{E}_{x\,\sim\,q_{\text{u}}^\triangleleft(\cdot;\lambda)} \left[ \log q_{\text{u}}^\triangleleft(x;\lambda)\right]
         \\[0.2cm]
        &  & 
        \displaystyle+ Z_{\text{u}}^\triangleleft(\lambda)D_{\text{KL}}\left( q_{\text{u}}^\triangleleft(\cdot;\lambda)\,\Vert\, p(\cdot)\right).
    \end{array}
    \]
    It is important to verify that $q(x) - \sum_{i\,=\,1}^m \lambda_{\text{fw}, i}^\star q_{\text{u}}^i(x) \geq 0 $ for any $x\in X$. If not, then $q(x) - \sum_{i\,=\,1}^m \lambda_{\text{fw}, i}^\star q_{\text{u}}^i(x) < 0 $ for some $x\in X$. In this case, minimization of $L_{\text{fw}}(p, \lambda_{\text{fw}}^\star)$ is achieved by $p_{\text{fw}}^\star$ that satisfies $p_{\text{fw}}^\star(x) = 0$, and $L_{\text{fw}}(p_{\text{fw}}^\star,\lambda_{\text{fw}}^\star) = -\infty$. However, from the strong duality we have $L_{\text{fw}}(p_{\text{fw}}^\star,\lambda_{\text{fw}}^\star) = D_{\text{KL}}(q\,\Vert\,p_{\text{fw}}^\star)\geq 0$, leading to a contradiction. Therefore, we apply the strong duality from Theorem~\ref{thm: unlearning forward KL strong duality} to evaluate, 
    \[
    p_{\text{fw}}^\star(\cdot) 
    \; = \; 
    p_{\text{fw}}^\star(\cdot;\lambda_{\text{fw}}^\star) 
    \; = \; 
    \argmin_{p\,\in\,\Delta} L_{\text{fw}}(p,\lambda_{\text{fw}}^\star)
    \; = \; 
    q_{\text{u}}^{\triangleleft}(\cdot;\lambda_{\text{fw}}^\star)
    \]
    which completes the proof.
\end{proof}

\subsection{Proof of Theorem~\ref{thm: unlearning likelihood}}\label{app: unlearning likelihood}

\begin{proof}
We note that Problem~\eqref{eq: unlearning alignment} is a convex optimization problem, since the reverse KL divergence is a strongly convex function and the likelihood constraints are linear. According to the standard duality analysis~\cite{boyd2004convex}, there is no duality gap when the problem is strictly feasible.

\end{proof}

\section{Proofs in Section~\ref{sec: unlearning in score function space}}\label{app: proofs unlearning in score function space}

\subsection{Proof of Theorem~\ref{thm: unlearning reverse KL strong duality_point-wise}}\label{app: reverse kl strong duality_point-wise}

\begin{proof}
We consider Problem~\eqref{eq: unlearning reverse KL diffusion models_point-wise_score} in the distribution space, 
\begin{equation}\label{eq: unlearning reverse KL diffusion models_point-wise_distribution}
    \begin{array}{rl}
        \displaystyle\minimize_{p_{0}(\cdot)} & 
        D_{\text{KL}}(p_{0}(\cdot)\,\Vert\, q_{0}(\cdot; s_q))
        \\[0.2cm]
        \subject &  D_{\text{KL}}(p_{0}(\cdot)\,\Vert\, q^i_{0}(\cdot; s_q^i))\; \geq \; b_i \; \text{ for }\; i\in [m].
    \end{array}
\end{equation}
    Define $f(p_{0:T}(x_{0:T})) \DefinedAs p_{0:T}(x_{0:T})\log \frac{\int p_{0:T}(x_{0:T})d{x_{1:T}}}{\int q_{0:T}(x_{0:T})d{x_{1:T}}}$. Since $p_{0:T}$, $q_{0:T} \ll \mu$, we can rewrite the reverse KL divergence $D_{\text{KL}}(p_{0}\,\Vert\, q_{0})$ as
    \[
    D_{\text{KL}}(p_{0}\,\Vert\, q_{0}) 
    \; = \; 
    \int_{X^{T+1}} p_{0:T}(x_{0:T}) \log \frac{\int p_{0:T}(x_{0:T})d{x_{1:T}}}{\int q_{0:T}(x_{0:T})d{x_{1:T}}} \mu(dx_{0:T})
    \; = \; 
    \mathbb{E}_\mu\left[\, f(p_{0:T}(x_{0:T})) \,\right]
    \]
    where $\mu$: $\mathcal{B}(X^{T+1}) \to [0,\infty]$ is the standard Lebesgue measure. Similarly, we define $f_i(p_{0:T}(x_{0:T})) \DefinedAs p_{0:T}(x_{0:T}) \log \frac{\int p_{0:T}(x_{0:T})dx_{1:T}}{\int q_{0:T}^i(x_{0:T})dx_{1:T}}$, and thus $D_{\text{KL}}(p_{0}\,\Vert\, q_{0}^i) = \mathbb{E}_\mu\left[\, f_i(p_{0:T}(x_{0:T})) \,\right]$. Meanwhile, we explicitly express the probability constraint for $p_{0:T}\in \Delta(X^{T+1})$ as
    \[
    \mathbb{E}_{\mu} \left[\,
    p_{0:T}(x_{0:T})
    \,\right]
    \; \DefinedAs\;
    \int_{X^{T+1}} p_{0:T}(x_{0:T}) \mu(dx_{0:T}) 
    \; = \; 
    1.
    \]
    
    To prove strong duality for Problem~\eqref{eq: unlearning reverse KL diffusion models_point-wise_distribution}, we prove it for an equivalent formulation of Problem~\eqref{eq: unlearning reverse KL diffusion models_point-wise_distribution},
    \begin{equation}\label{eq: unlearning reverse KL explicit diffusion models_point-wise_equiv}
    \begin{array}{rl}
        \displaystyle\minimize_{p_{0:T}\,\in\,\mathcal{P}} & 
        \mathbb{E}_\mu \left[\, f(p_{0:T}(x_{0:T})) \,\right]
        \\[0.2cm]
        \subject &  \mathbb{E}_\mu \left[\, f_i(p_{0:T}(x_{0:T})) \,\right]\; \geq \; b_i \;\; \text{ for } i \in [m]
        \\[0.2cm]
        & \displaystyle
        \mathbb{E}_{\mu} \left[\,
    p_{0:T}(x_{0:T})
    \,\right] \; = \;1.
    \end{array}
\end{equation}
The rest follows the same proof steps in the proof of Theorem~\ref{thm: unlearning reverse KL strong duality}. First, we employ Lyapunov's convexity theorem to show that the epigraph for Problem~\eqref{eq: unlearning reverse KL explicit diffusion models_point-wise_equiv} is non-empty and convex, implying that~\eqref{eq: unlearning reverse KL explicit diffusion models_point-wise_equiv} is strongly dual. Second, we translate the strong duality for Problem~\eqref{eq: unlearning reverse KL explicit diffusion models_point-wise_equiv} to the original problem~\eqref{eq: unlearning reverse KL diffusion models_point-wise_distribution}.

To prove strong duality for Problem~\eqref{eq: unlearning reverse KL diffusion models_point-wise_score}, we introduce some notation as follows. Let $\bar p_{0, \text{rev}}^\star$ be a solution to Problem~\eqref{eq: unlearning reverse KL diffusion models_point-wise_distribution} and denote $\bar P_{\text{rev}}^\star \DefinedAs D_{\text{KL}}(\bar p_{0, \text{rev}}^\star(\cdot)\,\Vert\, q_{0}(\cdot; s_q))$. The Lagrangian for Problem~\eqref{eq: unlearning reverse KL diffusion models_point-wise_distribution} is given by $\bar L_{\text{rev}}(p_{0}(\cdot), \lambda) \DefinedAs L_{\text{rev}}(p_{0}(\cdot), \lambda)$, and its dual function is given by $\bar D_{\text{rev}}(\lambda) \DefinedAs \min_{p_{0}} \bar L_{\text{rev}}(p_{0}(\cdot), \lambda)$. Let an optimal dual variable be $\bar \lambda_{\text{rev}}^\star \in\argmax_{\lambda\,\geq\,0} \bar D_{\text{rev}}(\lambda)$, and the optimal value of the dual function be $\bar D_{\text{rev}}^\star \DefinedAs \bar D_{\text{rev}}(\bar \lambda_{\text{rev}}^\star)$. From the strong duality of Problem~\eqref{eq: unlearning reverse KL diffusion models_point-wise_distribution}, there exists a pair $(\bar p_{0, \text{rev}}^\star, \bar \lambda_{\text{rev}}^\star)$ such that 
\[
\bar P_{\text{rev}}^\star \;=\; \bar D_{\text{rev}}^\star \; \text{ or }\; D_{\text{KL}}(\bar p_{0, \text{rev}}^\star(\cdot)\,\Vert\, q_{0}(\cdot; s_q)) \;=\; \bar D_{\text{rev}}(\bar \lambda_{\text{rev}}^\star).
\]
With Assumption~\ref{as: feasibility reverse KL_point-wise}, this implies that $(\bar p_{0, \text{rev}}^\star, \bar \lambda_{\text{rev}}^\star)$ is a saddle point of the Lagrangian $\bar L_{\text{rev}}(p_{0}(\cdot), \lambda)$,
\[
\bar L_{\text{rev}}(\bar p_{0, \text{rev}}^\star(\cdot), \lambda) \; \leq \; \bar L_{\text{rev}}(\bar p_{0, \text{rev}}^\star(\cdot), \bar\lambda_{\text{rev}}^\star) \; \leq \; \bar L_{\text{rev}}(p_{0}(\cdot), \bar\lambda_{\text{rev}}^\star)\; \text{ for all } p_{0}(\cdot) \text{ and } \lambda \geq 0.
\]
By Assumption~\ref{as: admissibility reverse KL_point-wise}, the score function class $\mathcal{S}$ is expressive enough, any point distribution $p_{0}(\cdot)$ can be represented as $p_{0}(\cdot;s_p)$ with some $s_p\in\mathcal{S}$; and vice versa. Thus, we can express $\bar p_{0, \text{rev}}^\star(\cdot)$ as $ p_{0}(\cdot; s_{\text{rev}}^\star)$ with some $s_{\text{rev}}^\star\in\mathcal{S}$. We also note that the dual function $\bar D_{\text{rev}}(\lambda)$ in the path and score function spaces are the same. Hence, the optimal dual function for Problem~\eqref{eq: unlearning reverse KL diffusion models_point-wise_score} remains to be $\hat D_{\text{rev}}(\lambda) = \bar D_{\text{rev}}(\lambda)$. Thus, $(s_{\text{rev}}^\star, \bar\lambda_{\text{rev}}^\star)$ is a saddle point of the Lagrangian $\hat L_{\text{rev}}(s_p;\lambda) \DefinedAs \bar L_{\text{rev}}(p_{0}(\cdot; s_p);\lambda)$,
\[
\hat L_{\text{rev}}(s_{\text{rev}}^\star;\lambda) \; \leq \; \hat L_{\text{rev}}(s_{\text{rev}}^\star;\bar\lambda_{\text{rev}}^\star) \; \leq\; \hat L_{\text{rev}}(s_p;\bar\lambda_{\text{rev}}^\star) \; \text{ for all } s_p\in\mathcal{S} \text{ and } \lambda\geq0.
\]
Therefore, the strong duality holds for Problem~\eqref{eq: unlearning reverse KL diffusion models_point-wise_score} in the score function space.
\end{proof}

\subsection{Proof of Theorem~\ref{thm: unlearning forward KL strong duality_path-wise}}\label{app: unlearning forward KL strong duality_path-wise}

\begin{proof}
We consider Problem~\eqref{eq: unlearning forward KL diffusion models_path-wise_score} in the path distribution space, 
\begin{equation}\label{eq: unlearning forward KL diffusion models_path-wise_distribution}
    \begin{array}{rl}
        \displaystyle\minimize_{p_{0:T}(\cdot)} & 
        D_{\text{KL}}( q_{0:T}(\cdot; s_q) \,\Vert\, p_{0:T}(\cdot))
        \\[0.2cm]
        \subject &  D_{\text{KL}}( q^i_{0:T}(\cdot; s_q^i) \,\Vert\, p_{0:T}(\cdot))\; \geq \; b_i \; \text{ for }\; i\in [m].
    \end{array}
\end{equation}
    Define $f(p_{0:T}(x_{0:T})) \DefinedAs q_{0:T}(x_{0:T}) \log \frac{q_{0:T}(x_{0:T})}{p_{0:T}(x_{0:T})}$. Since $p_{0:T}$, $q_{0:T} \ll \mu$, we can rewrite the forward KL divergence $D_{\text{KL}}(q_{0:T}\,\Vert\, p_{0:T})$ as
    \[
    D_{\text{KL}}(q_{0:T}\,\Vert\, p_{0:T}) 
    \; = \; 
    \int_{X^{T+1}} q_{0:T}(x_{0:T}) \log \frac{q_{0:T}(x_{0:T})}{p_{0:T}(x_{0:T})} \mu(dx_{0:T})
    \; = \; 
    \mathbb{E}_\mu\left[\, f(p_{0:T}(x_{0:T})) \,\right]
    \]
    where $\mu$: $\mathcal{B}(X^{T+1}) \to [0,\infty]$ is the standard Lebesgue measure. Similarly, we define $f_i(p_{0:T}(x_{0:T})) \DefinedAs q_{0:T}^i(x_{0:T}) \log \frac{q_{0:T}^i(x_{0:T})}{p_{0:T}(x_{0:T})}$, and thus $D_{\text{KL}}(q_{0:T}^i \,\Vert\, p_{0:T}) = \mathbb{E}_\mu\left[\, f_i(p_{0:T}(x_{0:T})) \,\right]$. Meanwhile, we explicitly express the probability constraint for $p_{0:T}\in \Delta(X^{T+1})$ as
    \[
    \mathbb{E}_{\mu} \left[\,
    p_{0:T}(x_{0:T})
    \,\right]
    \; \DefinedAs\;
    \int_{X^{T+1}} p_{0:T}(x_{0:T}) \mu(dx_{0:T}) 
    \; = \; 
    1.
    \]
    
    To prove strong duality for Problem~\eqref{eq: unlearning forward KL diffusion models_path-wise_distribution}, we prove it for an equivalent formulation of Problem~\eqref{eq: unlearning forward KL diffusion models_path-wise_distribution},
    \begin{equation}\label{eq: unlearning forward KL diffusion models_path-wise_distribution_equiv}
    \begin{array}{rl}
        \displaystyle\minimize_{p_{0:T}\,\in\,\mathcal{P}} & 
        \mathbb{E}_\mu \left[\, f(p_{0:T}(x_{0:T})) \,\right]
        \\[0.2cm]
        \subject &  \mathbb{E}_\mu \left[\, f_i(p_{0:T}(x_{0:T})) \,\right]\; \geq \; b_i \;\; \text{ for } i \in [m]
        \\[0.2cm]
        & \displaystyle
        \mathbb{E}_{\mu} \left[\,
    p_{0:T}(x_{0:T})
    \,\right] \; = \;1.
    \end{array}
\end{equation}
The rest follows the same proof steps in the proof of Theorem~\ref{thm: unlearning forward KL strong duality}. First, we employ Lyapunov's convexity theorem to show that the epigraph for Problem~\eqref{eq: unlearning forward KL diffusion models_path-wise_distribution_equiv} is non-empty and convex, implying that~\eqref{eq: unlearning forward KL diffusion models_path-wise_distribution_equiv} is strongly dual. Second, we translate the strong duality for Problem~\eqref{eq: unlearning forward KL diffusion models_path-wise_distribution_equiv} to the original problem~\eqref{eq: unlearning forward KL diffusion models_path-wise_distribution}.

To prove strong duality for Problem~\eqref{eq: unlearning forward KL diffusion models_path-wise_distribution_equiv}, we introduce some notation as follows. Let $\bar p_{0:T, \text{fw}}^\star$ be a solution to Problem~\eqref{eq: unlearning forward KL diffusion models_path-wise_distribution_equiv} and denote $\bar P_\text{fw}^\star \DefinedAs D_{\text{KL}}( q_{0:T}(\cdot; s_q) \,\Vert\,  \bar p_{0:T, \text{fw}}^\star(\cdot))$. The Lagrangian for Problem~\eqref{eq: unlearning forward KL diffusion models_path-wise_distribution} is given by $\bar L_\text{fw}(p_{0:T}(\cdot), \lambda) \DefinedAs L_\text{fw}(p_{0:T}(\cdot), \lambda)$, and its dual function is given by $\bar D_\text{fw}(\lambda) \DefinedAs \min_{p_{0:T}} \bar L_\text{fw}(p_{0:T}(\cdot), \lambda)$. Let an optimal dual variable be $\bar \lambda_\text{fw}^\star \in\argmax_{\lambda\,\geq\,0} \bar D_\text{fw}(\lambda)$, and the optimal value of the dual function be $\bar D_\text{fw}^\star \DefinedAs \bar D_\text{fw}(\bar \lambda^\star)$. From the strong duality of Problem~\eqref{eq: unlearning forward KL diffusion models_path-wise_distribution_equiv}, there exists a pair $(\bar p_{0:T, \text{fw}}^\star, \bar \lambda_\text{fw}^\star)$ such that 
\[
\bar P_\text{fw}^\star \;=\; \bar D_\text{fw}^\star \; \text{ or }\; D_{\text{KL}}(q_{0:T}(\cdot; s_q)\,\Vert\, \bar p_{0:T, \text{fw}}^\star(\cdot)) \;=\; \bar D_\text{fw}(\bar \lambda_\text{fw}^\star).
\]
With Assumption~\ref{as: feasibility forward KL_path-wise}, this implies that $(\bar p_{0:T, \text{fw}}^\star, \bar \lambda_{\text{fw}}^\star)$ is a saddle point of the Lagrangian $\bar L_{\text{fw}}(p_{0:T}(\cdot), \lambda)$,
\[
\bar L_{\text{fw}}(\bar p_{0:T,{\text{fw}}}^\star(\cdot), \lambda) \; \leq \; \bar L_{\text{fw}}(\bar p_{0:T, \text{fw}}^\star(\cdot), \bar\lambda_{\text{fw}}^\star) \; \leq \; \bar L_{\text{fw}}(p_{0:T}(\cdot), \bar\lambda_{\text{fw}}^\star)\; \text{ for all } p_{0:T}(\cdot) \text{ and } \lambda \geq 0.
\]
By Assumption~\ref{as: admissibility foward KL_path-wise}, the score function class $\mathcal{S}$ is expressive enough, any path distribution $p_{0:T}(\cdot)$ can be represented as $p_{0:T}(\cdot;s_p)$ with some $s_p\in\mathcal{S}$; and vice versa. Thus, we can express $\bar p_{0:T, \text{fw}}^\star(\cdot)$ as $p_{0:T}(\cdot; s_\text{fw}^\star)$ with some $s_\text{fw}^\star\in\mathcal{S}$. We also note that the dual function $\bar D_{\text{fw}}(\lambda)$ in the path and score function spaces are the same. Hence, the optimal dual function for Problem~\eqref{eq: unlearning forward KL diffusion models_path-wise_score} remains to be $\hat D_{\text{fw}}(\lambda) = \bar D_{\text{fw}}(\lambda)$. Thus, $(s_{\text{fw}}^\star, \bar\lambda_\text{fw}^\star)$ is a saddle point of the Lagrangian $\hat L_\text{fw}(s_p;\lambda) \DefinedAs \bar L_\text{fw}(p_{0:T}(x_{0:T}; s_p);\lambda)$,
\[
\hat L_\text{fw}(s_\text{fw}^\star;\lambda) \; \leq \; \hat L_\text{fw}(s_\text{fw}^\star;\bar\lambda_\text{fw}^\star) \; \leq\; \hat L_\text{fw}(s_p;\bar\lambda_\text{fw}^\star) \; \text{ for all } s_p\in\mathcal{S} \text{ and } \lambda\geq0.
\]
Therefore, the strong duality holds for Problem~\eqref{eq: unlearning forward KL diffusion models_path-wise_score} in the score function space.
\end{proof}

\subsection{Proof of Theorem~\ref{thm: unlearning reverse KL likelihood strong duality_point-wise}}\label{app: unlearning reverse KL likelihood strong duality_point-wise}

\begin{proof}
We consider Problem~\eqref{eq: unlearning alignment diffusion models} in the distribution space, 
\begin{equation}\label{eq: unlearning reverse KL likelihood diffusion models_point-wise_distribution}
    \begin{array}{rl}
        \displaystyle\minimize_{p_{0:T}(\cdot)} & 
        D_{\text{KL}}(p_{0:T}(\cdot)\,\Vert\, q_{0:T}(\cdot; s_q))
        \\[0.2cm]
        \subject &  \mathbb{E}_{x_0 \,\sim\, p_0}\left[\, q_0^i (x_0; s^i_q)\,\right]\; \leq \; \epsilon_i \; \text{ for }\; i\in [m].
    \end{array}
\end{equation}
    Define $f(p_{0:T}(x_{0:T})) \DefinedAs p_{0:T}(x_{0:T})\log \frac{ p_{0:T}(x_{0:T})}{ q_{0:T}(x_{0:T})}$. Since $p_{0:T}$, $q_{0:T} \ll \mu$, we can rewrite the reverse KL divergence $D_{\text{KL}}(p_{0}\,\Vert\, q_{0})$ as
    \[
    D_{\text{KL}}(p_{0:T}\,\Vert\, q_{0:T}) 
    \; = \; 
    \int_{X^{T+1}} p_{0:T}(x_{0:T}) \log \frac{ p_{0:T}(x_{0:T}) }{  q_{0:T}(x_{0:T}) } \mu(dx_{0:T})
    \; = \; 
    \mathbb{E}_\mu\left[\, f(p_{0:T}(x_{0:T})) \,\right]
    \]
    where $\mu$: $\mathcal{B}(X^{T+1}) \to [0,\infty]$ is the standard Lebesgue measure. Similarly, we define $f_i(p_{0:T}(x_{0:T})) \DefinedAs p_{0:T}(x_{0:T}) {\int q_{0:T}^i(x_{0:T})dx_{1:T}}$, and thus $\mathbb{E}_{x_0 \,\sim\, p_0}\left[\, q_0^i (x_0; s^i_q)\,\right] = \mathbb{E}_\mu\left[\, f_i(p_{0:T}(x_{0:T})) \,\right]$. Meanwhile, we explicitly express the probability constraint for $p_{0:T}\in \Delta(X^{T+1})$ as
    \[
    \mathbb{E}_{\mu} \left[\,
    p_{0:T}(x_{0:T})
    \,\right]
    \; \DefinedAs\;
    \int_{X^{T+1}} p_{0:T}(x_{0:T}) \mu(dx_{0:T}) 
    \; = \; 
    1.
    \]
    
    To prove strong duality for Problem~\eqref{eq: unlearning reverse KL likelihood diffusion models_point-wise_distribution}, we prove it for an equivalent formulation of Problem~\eqref{eq: unlearning reverse KL likelihood diffusion models_point-wise_distribution},
    \begin{equation}\label{eq: unlearning reverse KL likelihood explicit diffusion models_point-wise_equiv}
    \begin{array}{rl}
        \displaystyle\minimize_{p_{0:T}\,\in\,\mathcal{P}} & 
        \mathbb{E}_\mu \left[\, f(p_{0:T}(x_{0:T})) \,\right]
        \\[0.2cm]
        \subject &  \mathbb{E}_\mu \left[\, f_i(p_{0:T}(x_{0:T})) \,\right]\; \leq \; \epsilon_i \;\; \text{ for } i \in [m]
        \\[0.2cm]
        & \displaystyle
        \mathbb{E}_{\mu} \left[\,
    p_{0:T}(x_{0:T})
    \,\right] \; = \;1.
    \end{array}
\end{equation}
The rest follows the same proof steps in the proof of Theorem~\ref{thm: unlearning reverse KL strong duality}. First, we employ Lyapunov's convexity theorem to show that the epigraph for Problem~\eqref{eq: unlearning reverse KL likelihood explicit diffusion models_point-wise_equiv} is non-empty and convex, implying that~\eqref{eq: unlearning reverse KL likelihood explicit diffusion models_point-wise_equiv} is strongly dual. Second, we translate the strong duality for Problem~\eqref{eq: unlearning reverse KL likelihood explicit diffusion models_point-wise_equiv} to the original problem~\eqref{eq: unlearning reverse KL likelihood diffusion models_point-wise_distribution}.

To prove strong duality for Problem~\eqref{eq: unlearning alignment diffusion models}, we introduce some notation as follows. Let $\bar p_{0:T, \text{revl}}^\star$ be a solution to Problem~\eqref{eq: unlearning reverse KL likelihood diffusion models_point-wise_distribution} and denote $\bar P_{\text{revl}}^\star \DefinedAs D_{\text{KL}}(\bar p_{0:T, \text{revl}}^\star(\cdot)\,\Vert\, q_{0:T}(\cdot; s_q))$. The Lagrangian for Problem~\eqref{eq: unlearning reverse KL likelihood diffusion models_point-wise_distribution} is given by $\bar L_\text{revl}(p_{0:T}(\cdot), \lambda) \DefinedAs L_\text{revl}(p_{0:T}(\cdot), \lambda)$, and its dual function is given by $\bar D_\text{revl}(\lambda) \DefinedAs \min_{p_{0:T}} \bar L_\text{revl}(p_{0:T}(\cdot), \lambda)$. Let an optimal dual variable be $\bar \lambda_\text{revl}^\star \in\argmax_{\lambda\,\geq\,0} \bar D_\text{revl}(\lambda)$, and the optimal value of the dual function be $\bar D_\text{revl}^\star \DefinedAs \bar D_\text{revl}(\bar \lambda_\text{revl}^\star)$. From the strong duality of Problem~\eqref{eq: unlearning reverse KL likelihood diffusion models_point-wise_distribution}, there exists a pair $(\bar p_{0:T,\text{revl}}^\star, \bar \lambda_\text{revl}^\star)$ such that 
\[
\bar P_\text{revl}^\star 
\;=\; 
\bar D_\text{revl}^\star 
\; \text{ or }\; 
D_{\text{KL}}(\bar p_{0:T, \text{revl}}^\star(\cdot)\,\Vert\, q_{0:T}(\cdot; s_q)) 
\;=\; 
\bar D_\text{revl}(\bar \lambda_\text{revl}^\star).
\]
With Assumption~\ref{as: feasibility reverse KL likelihood_point-wise}, this implies that $(\bar p_{0:T,\text{revl}}^\star, \bar \lambda_\text{revl}^\star)$ is a saddle point of the Lagrangian $\bar L_\text{revl}(p_{0:T}(\cdot), \lambda)$,
\[
\bar L_{\text{revl}}(\bar p_{0:T, \text{revl}}^\star(\cdot), \lambda) \; \leq \; \bar L_{\text{revl}}(\bar p_{0:T, \text{revl}}^\star(\cdot), \bar\lambda_{\text{revl}}^\star) \; \leq \; \bar L_{\text{revl}}(p_{0:T}(\cdot), \bar\lambda_{\text{revl}}^\star)\; \text{ for all } p_{0:T}(\cdot) \text{ and } \lambda \geq 0.
\]
By Assumption~\ref{as: admissibility reverse KL_point-wise}, the score function class $\mathcal{S}$ is expressive enough, any path distribution $p_{0:T}(\cdot)$ can be represented as $p_{0:T}(\cdot;s_p)$ with some $s_p\in\mathcal{S}$; and vice versa. Thus, we can express $\bar p_{0:T, \text{revl}}^\star(\cdot)$ as $ p_{0}(\cdot; s_{\text{revl}}^\star)$ with some $s_{\text{revl}}^\star\in\mathcal{S}$. We also note that the dual function $\bar D_{\text{revl}}(\lambda)$ in the path and score function spaces are the same. Hence, the optimal dual function for Problem~\eqref{eq: unlearning alignment diffusion models} remains to be $\hat D_{\text{revl}}(\lambda) = \bar D_{\text{revl}}(\lambda)$. Thus, $(s_{\text{revl}}^\star, \bar\lambda_{\text{revl}}^\star)$ is a saddle point of the Lagrangian $\hat L_{\text{revl}}(s_p;\lambda) \DefinedAs \bar L_{\text{revl}}(p_{0:T}(\cdot; s_p);\lambda)$,
\[
\hat L_{\text{revl}}(s_{\text{revl}}^\star;\lambda) 
\; \leq \; 
\hat L_{\text{revl}}(s_{\text{revl}}^\star;\bar\lambda_{\text{revl}}^\star) 
\; \leq\; 
\hat L_{\text{revl}}(s_p;\bar\lambda_{\text{revl}}^\star) \; \text{ for all } s_p\in\mathcal{S} \text{ and } \lambda\geq0.
\]
Therefore, the strong duality holds for Problem~\eqref{eq: unlearning alignment diffusion models} in the score function space.
\end{proof}

\subsection{Likelihood Estimation}\label{app: proof of likelihood estimation}
\begin{lemma}\label{lem: likelihood estimation} Consider $q_0(\cdot)$ and $q_0^i(\cdot)$ to be the distributions of samples for two diffusion models with score functions $s_q(x_t, t)$ and $s_{q^i}(x_t,t)$ respectively. Then the  logarithm of the ratio of the probabilities $\log q_{0}^i(x_0)/q_{0}(x_0)$ assigned to a given sample $x_0$ by the two models can be written as:
    \[
        \log q^i_{0}(x_0) - \log q_{0}(x_0)
        \; \approx \;
        \frac{1}{2}\int_{0}^{\infty} d t
 \left(
\mathbb{E}_{q(x_t|x_0)} \left[ \norm{\nabla \log q(x_t) - s_q(x_t, t)}_2^2\right] - \mathbb{E}_{q(x_t|x_0)} \left[ \norm{\nabla \log q^i(x_t) - s_{q^i}(x_t,t)}_2^2\right] \right) \omega_t
    \]
where $\omega_t$ is a time-dependent constant related to the noise schedule of the diffusion process.
\end{lemma}

\begin{proof}

We begin with the setting and notation of~\cite{kong2023informationtheoreticdiffusion}. Let $p(z_\gamma \mid x)$ be a Gaussian noise channel with
\[
z_\gamma 
\;
= 
\;
\sqrt{\gamma}\,x + \epsilon
\]
and $\epsilon \sim \mathcal{N}(0, I)$, where $\gamma$ represents the
Signal-to-Noise Ratio (SNR) and $p(x)$ is the unknown data distribution.
Then the log probability of the unknown distribution for a given sample $x$ can be written as an integral (Equation 9 from~\cite{kong2023informationtheoreticdiffusion}):
\begin{equation}
    - \log p(x)
    \;
    = 
    \;
    \frac{d}{2}\log(2\pi e)
- \frac{1}{2}\int_{0}^{\infty} d\gamma
\left(
\frac{d}{1+\gamma} - \operatorname{mmse}(x,\gamma)
\right)
\end{equation}
where we define the point-wise mmse as:
\begin{equation}
    \operatorname{mmse}(x,\gamma) 
    \;
    \DefinedAs 
    \;
    \mathbb{E}_{p(z_\gamma|x)} \left[ \norm{x - \hat{x}^\star(z_\gamma,\gamma)}_2^2\right]
\end{equation}
Note that we can write the training loss of a diffusion model trained with samples from $p(x)$ as minimizing the expected point-wise mmse, thus $\hat{x}^\star(z_\gamma,\gamma)$ can be viewed as the optimal diffusion model denoiser. With this in mind, we modify the notation slightly to write:
\begin{equation}
    \operatorname{mmse}(x_0,t) 
    \;
    \DefinedAs 
    \;
    \mathbb{E}_{q(x_t|x_0)} \left[ \norm{x_0 - \hat{x}^\star(x_t,t)}_2^2\right].
\end{equation}

Furthermore, from~\cite{luo2022understandingdiffusionmodelsunified} we recall that the denoising diffusion loss and the score matching loss are equivalent up to a time-dependent constant which we denote $\omega_t$ allowing us to write:
\begin{equation}
    \operatorname{mmse}(x_0,t) 
    \; 
    = 
    \;
    \omega_t \cdot \mathbb{E}_{q(x_t|x_0)} \left[ \norm{\nabla \log p(x_t) - \hat{s}^\star(x_t,t)}_2^2\right]
\end{equation}
where $s^\star(\cdot, t)$ is the minimizer of the score matching loss $\mathbb{E}_{p(x_0)} \mathbb{E}_{q(x_t|x_0)} \left[ \norm{\nabla \log p(x_t) - \hat{s}(x_t,t)}_2^2\right]$.
Thus, if we consider two distributions $p_1(x_0)$, $p_2(x_0)$ we can write:
\[
\begin{array}{rcl}
    &&
    \!\!\!\!  \!\!\!\!
    \!\!
    \log p_2(x_0) - \log p_1(x_0) 
    \\[0.2cm]
    & = & 
    \displaystyle
    \frac{1}{2}\int_{0}^{\infty} d t
\left(
\operatorname{mmse}_1(x_t,t) - \operatorname{mmse}_2(x_t, t)
\right)
\\[0.2cm]
   & = & 
   \displaystyle
   \frac{1}{2}\int_{0}^{\infty} d t
 \left(
\mathbb{E}_{q(x_t|x_0)} \left[ \norm{\nabla \log p_1(x_t) - \hat{s}_1^\star(x_t,t)}_2^2\right] - \mathbb{E}_{q(x_t|x_0)} \left[ \norm{\nabla \log p_2(x_t) - \hat{s}_2^\star(x_t,t)}_2^2\right] \right) \omega_t.
\end{array}
\]

Again, from~\cite{luo2022understandingdiffusionmodelsunified} we know:
\begin{equation}
    \nabla \log p(x_t) = \frac{\sqrt{\bar \alpha_t} x_0 - x_t}{1 - \bar \alpha_t}
\end{equation}
where $\alpha_t$ represents the noise schedule, allowing us to compute the integral.
\end{proof}

\newpage
\section{Algorithms}\label{app: algorithms}
Here we detail the primal-dual algorithms discussed in Section~\ref{sec: unlearning in score function space} of the main paper.

\subsection{Reverse KL-Constrained Unlearning Algorithm}\label{app: algorithm reverse KL}

As discussed in the main paper the algorithm alternates between minimizing the Lagrangian (primal) function, and maximizing the dual function.

\textbf{Primal step: }In the primal step we minimize the Lagrangian:
\begin{equation}\label{eq: reverse kl alg lagrangian}
    \minimize_{s_p}\;\hat{L}_{\text{rev}}(s_p,\lambda) 
    \; \DefinedAs \; 
    D_{\text{KL}}(p_{0}(\cdot; s_p)\,\Vert\, q_{0}(\cdot; s_q))  -\sum_i \lambda_i \left( D_{\text{KL}}(p_{0}(\cdot; s_p)\,\Vert\, q^i_{0}(\cdot; s^i_q)) \; - b_i \right)
\end{equation}

We can estimate the point-wise KL divergences in~\eqref{eq: reverse kl alg lagrangian} by sampling trajectories from the model and using the following result (Lemma  2 from~\cite{khalafi2025composition}):

\begin{lemma}[Point-wise KL divergence]\label{lem: marginal kl computation}
    Assume two score functions $s_p(x,t) = \nabla \log \bar p_t(x)$, $s_q (x,t) = \nabla \log \bar q_t(x)$, where $\bar p_t$, $\bar q_t$ are two marginal densities induced by two forward diffusion processes, with the same noise schedule, starting from initial distributions $\bar p_0$ and $\bar q_0$, respectively. Then, the point-wise KL divergence between two distributions of the samples generated by running DDIM with $s_p$ and $s_q$ is given by
    \begin{equation}\label{eq: marginal kl computation}
        D_{\text{\normalfont KL}} (p_0(\cdot\,; s_p) \,\Vert\, q_0(\cdot\,; s_q))
        \; = \;
        \sum_{t \,=\, 0}^T \tilde \omega_t \, \mathbb{E}_{x_t\,\sim\, p_t(\cdot\,; s_p)} \left[ \,
        \norm{s_p(x_t, t) - s_q(x_t, t)}_2^2
        \,\right] 
        \,+\, \epsilon_T
    \end{equation}
    where $\tilde \omega_t$ is a time-dependent constant, and $\epsilon_T$ is a discretization error that depends on the total number of diffusion time steps $T$ and goes to zero as $T \rightarrow \infty$
\end{lemma}

We can then minimize the weighted sum of the KL divergences using standard gradient-based optimization techniques.

\textbf{Dual step:} In the dual step we update the dual multipliers via $\lambda^+ = \lambda + \eta \partial\hat{D}_{\text{rev}}(\lambda)$, using the fact that the constraint violations  give us a subgradient $\partial\hat{D}_{\text{rev}}(\lambda)$ of the dual function.
\begin{equation}\label{eq: dual update reverse kl 123}
    \lambda^+_i = \left(\lambda_i - \eta_d \left( D_{\text{KL}}(p_{0}(\cdot; s_p)\,\Vert\, q^i_{0}(\cdot; s^i_q)) \; - b_i \right) \right)_+.
\end{equation}
A simple observation regarding the dual update is that if a constraint is not satisfied, we increase the corresponding multiplier.

\subsection{Forward KL-Constrained Unlearning Algorithm}

Similar to the reverse KL algorithm, we alternate between minimizing the Lagrangian (primal) function, and maximizing the dual function.

\textbf{Primal step: }In the primal step we minimize the Lagrangian:
\begin{equation}
    \minimize_{s_p}\;
    \hat{L}_{\text{fw}}(s_p,\lambda) 
    \; \DefinedAs \;
    D_{\text{KL}}(q_{0:T}(\cdot; s_q)\,\Vert\, p_{0:T}(\cdot; s_p))
   -\sum_i \lambda_i \left( D_{\text{KL}}(q^i_{0:T}(\cdot; s^i_q)\,\Vert\, p_{0:T}(\cdot; s_p))\; - b_i \right).
\end{equation}

The forward KL divergence $D_{\text{KL}}(q_{0:T}(\cdot; s^i_q)\,\Vert\, p_{0:T}(\cdot; s_p))$ is the standard diffusion model training objective which we can estimate using samples from the distribution $q$:

\begin{equation}\label{eq: score matching}
    D_{\text{KL}}(q_{0:T}(\cdot; s_q)\,\Vert\, p_{0:T}(\cdot; s_p)) 
    \; = \;
    \mathbb{E}_{x_0\sim q} \mathbb{E}_{q(x_t|x_0)} \left[ \Vert s_p(x_t, t) - \nabla \log q(x_t, t) \Vert^2_2 \right] + \text{constant}.
\end{equation}

The first term on the RHS of~\eqref{eq: score matching} is the standard score matching objective with samples from the distribution $q$. The Lagrangian then becomes:
\begin{equation}
    \hat{L}_{\text{fw}}(s,\lambda) = \mathbb{E}_{x_0\sim q} \mathbb{E}_{q(x_t|x_0)} \left[ \Vert s_p(x_t, t) - \nabla \log q(x_t, t) \Vert^2_2 \right]
   -\sum_i \lambda_i \left( \mathbb{E}_{x_0\sim q_i} \mathbb{E}_{q(x_t|x_0)} \left[ \Vert s_p(x_t, t) - \nabla \log q^i(x_t, t) \Vert^2_2 \right]\; - b_i \right)
\end{equation}
which we can minimize using standard gradient-based optimization techniques using samples from $q$ and $q^i$.

\textbf{Dual step:} In the dual step, similar to the reverse KL algorithm we simply update the dual multipliers dependent on the constraint violations:
\begin{equation}
    \lambda^+_i = \left(\lambda_i - \eta_d \left( D_{\text{KL}}(q^i_{0:T}\,\Vert\, p_{0:T}(\cdot; s_p)) \; - b_i \right) \right)_+.
\end{equation}\label{eq: dual update reverse kl}

\subsection{Likelihood-Constrained Unlearning Algorithm}

The algorithm similarly alternates between primal and dual steps.

\textbf{Primal step:} For the primal step we need to minimize the Lagrangian:

\begin{equation}\label{eq: lagrangian revl}
    \minimize_{s_p} \;\hat{L}_{\text{revl}}(s_p,\lambda) 
    \; \DefinedAs \; 
    D_{\text{KL}}(p_{0:T}(\cdot; s_p)\,\Vert\, q_{0:T}(\cdot; s_q))
   +\sum_i \lambda_i \left( \mathbb{E}_{p}
\bigl[ q^i(x_0) \bigr]\; - \epsilon_i \right).
\end{equation}

We can treat the likelihood of the unlearning distribution $q^i(x_0)$ as a reward function that can be evaluated for any given sample using a diffusion model that samples from the distribution $q^i(\cdot)$. We discuss this in detail in Lemma~\ref{lem: likelihood estimation} the result of which we summarize here:
\[
        \log q^i_{0}(x_0) - \log q_{0}(x_0)
        \; \approx \;
        \frac{1}{2}
        \sum_{t\,=\,0}^T 
        \left(
\mathbb{E}_{q(x_t|x_0)} \left[ \norm{\nabla \log q(x_t) - s_q(x_t, t)}_2^2\right] - \mathbb{E}_{q(x_t|x_0)} \left[ \norm{\nabla \log q^i(x_t) - s_{q^i}(x_t,t)}_2^2\right] \right)
\]
Note we have further approximated the integral in Lemma~\ref{lem: likelihood estimation} with a sum over discrete diffusion time steps.

Now that we have a way to estimate the likelihood, the primal minimization becomes the well-known problem of aligning a diffusion model with a penalty function with the reward being the unlearning likelihood $q(x_0)$ in this case. We can use the following result from~\cite{fan2023dpokreinforcementlearningfinetuning} to compute the gradient of the expected likelihood:
\begin{equation}
 \nabla_\theta \mathbb{E}_{p_\theta(x_0)}
\bigl[ q(x_0) \bigr] 
\;=\;
\mathbb{E}_{p_\theta(x_{0:T})}
\left[
q(x_0)
\sum_{t=1}^{T}
\nabla_\theta \log p_\theta(x_{t-1} \mid x_t)
\right]
\end{equation}
which allows us to minimize the Lagrangian in~\eqref{eq: lagrangian revl} using gradient-based optimization.

\textbf{Dual step:} For the dual steps we update the dual variables via:
\begin{equation}
    \lambda_i 
    \;=\; 
    \left(\lambda_i + \eta_d \left( \mathbb{E}_{p}
\bigl[ q^i(x_0) \bigr] - \epsilon_i \right)\right)_+.
\end{equation}

\newpage
\section{Additional Implementation Details}
Here we provide some additional implementation details regarding the experiments discussed in Section~\ref{sec: experiments}. All experiments were run on an NVIDIA B200 GPU.

\subsection{Likelihood-Constrained Unlearning}
\subsubsection{Gaussian Mixture Unlearning} 
For the Gaussian mixture unlearning experiments we trained diffusion models with ResNet~\cite{he2015deepresiduallearningimage} networks as denoisers. We used a DDPM~\cite{ho2020denoisingdiffusionprobabilisticmodels} scheduler with $T = 100$ diffusion time steps. For the reverse KL baseline, we set the dual update learning rate to $0$ to essentially minimize the Lagrangian with a fixed multiplier $\lambda$, and we ran for different values of $\lambda$ ranging from $0.01$ to $1.0$. For the Likelihood constrained approach we ran for different constraint thresholds ranging from $0.05$ to $1.0$. The likelihood is normalized such that the expected likelihood of the unlearn concept is approximately $1.0$ when evaluated over samples of the initial pretrained model.

The constraint thresholds are chosen by multiplying the expected likelihood of the unlearning evaluated on samples from the pretrained model, by a factor.

We also note that since Lemma~\ref{lem: likelihood estimation} allows the estimation of ratio of likelihoods, we consider the likelihood of the distribution that is proportional to $q^i(x)/q_{pre}(x)$, where $q_{pre}(x)$ is the pretrained model.

The full list of important hyperparameters is listed in Table~\ref{tab: hyper gmm}.

\begin{table}[h]
\centering
\caption{Hyperparameters for constrained likelihood vs fixed weight reverse KL diffusion unlearning experiments. Used for producing the results in Figure~\ref{fig: gaussians_likelihood_vs_reverse}.}
\label{tab: hyper gmm}
\begin{tabular}{lll}
\hline
\textbf{Category} & \textbf{Hyperparameter} & \textbf{Value} \\
\hline
Training 
& Number of epochs ($N_{\text{epochs}}$) & $100$ \\
& Batches per epoch & $1$ \\
& Primal learning rate ($\eta_p$) & $2 \times 10^{-4}$ \\
& Dual learning rate ($\eta_d$) & $0$/$0.1$ \\
& Primal steps per dual step & $5$ \\
& Primal batch size & $128$ \\
& Gradient accumulation steps & $1$ \\
& Optimizer & AdamW \\
& LR schedule & Cosine w/ warmup \\
\hline
Diffusion
& Number of diffusion steps ($T$) & $100$ \\
& Inference steps & $10$ \\
& Noise scheduler & DDPM \\
& noise scheduler $\beta_{\text{start}}$ & $1 \times 10^{-4}$ \\
& noise scheduler $\beta_{\text{end}}$ & $2 \times 10^{-2}$ \\
\hline
Model architecture
& Model & ResNet Diffusion Model \\
& Input dimension ($x_{\text{dim}}$) & $2$ \\
& Hidden dimension ($h_{\text{dim}}$) & $128$ \\
& Embedding dimension & $32$ \\
& Number of layers & $4$ \\
& Widen factor & $2$ \\
& Time embedding & Learned \\
\hline
\end{tabular}
\end{table}

\subsubsection{Concept Unlearning with Stable Diffusion}
We use a distilled version of Stable Diffusion v1.4 for these experiments. We fine-tune the parameters of LoRA adapters~\cite{hu2021loralowrankadaptationlarge} added to the cross-attention layers in the diffusion U-Net following the esd-x approach from~\cite{gandikota2023erasing}.

For computing KID to the retain distribution, we construct a retain set by sampling images from the pretrained model and keeping those whose likelihood/CLIP score for the unlearn concept is lower than a certain threshold i.e., samples in which the unlearn concept does not appear. The threshold is the expected likelihood/CLIP over samples from the pretrained model multiplied by a factor.

Similarly the likelihood constraint thresholds were chosen by multiplying the expected likelihood of the unlearn concept for the pretrained model ($\bar \ell$) by factors smaller than $1$. We highlight that since Lemma~\ref{lem: likelihood estimation} allows the estimation of ratio of likelihoods, we consider the likelihood of the distribution that is proportional to $q_{pre}(x, c_i)/q_{pre}(x)$, where $q_{pre}(x)$ is the unconditional pretrained distribution of Stable Diffusion and $q_{pre}(x, c_i)$ is conditioned on the unlearning concept $c_i$.

We also note that, unlike the Gaussian case where the likelihood constraint is reliably satisfied regardless of threshold level, for Stable Diffusion, optimizing the expected likelihood  is challenging as the loss sometimes diverges. In Figure~\ref{fig: cowboys_likelihood_vs_reverse} we report only the convergent runs.

The full list of important hyperparameters is listed in Table~\ref{tab:unlearning_sweep_hyperparams}.

\begin{table}[t]
\centering
\caption{Hyperparameters for fixed weight concept erasure baseline (unconstrained) and Likelihood-constrained unlearning sweeps. Used for producing the results in Figure~\ref{fig: cowboys_likelihood_vs_reverse}.}
\label{tab:unlearning_sweep_hyperparams}
\begin{tabular}{lll}
\hline
\textbf{Category} & \textbf{Hyperparameter} & \textbf{Value} \\
\hline
General
& Base model & \texttt{huggingface/nota-ai/bk-sdm-tiny} \\
& Mixed precision & FP16 \\
& Number of epochs & $30$ \\
& Training batch size & $8$ \\
& Gradient accumulation steps & $1$ \\
& Learning rate & $5 \times 10^{-4}$ \\
& Optimizer & Adam \\
& LoRA rank ($r$) & $8$ \\
& LoRA scaling ($\alpha$) & $8$ \\
& Primal steps per dual step & $2$ \\
\hline
Diffusion / sampling
& Inference steps (train) & $20$ \\
& Noise samples per timestep for likelihood estimation & $8$ \\
& Expected Likelihood Evaluation batch size & $16$ \\
\hline
Retain set generation
& Number of retain images & $1000$ \\
& clip threshold factor & $0.95$ \\
& likelihood threshold factor & $1.0$ \\
\hline
Unconstrained Baseline
& Dual learning rate ($\eta_d$) & $0$ \\
& Fixed multiplier sweep values & 
$\{0.01, 0.1, 0.2, 0.3, 0.5, 0.75,$ \\
& & $1.0, 1.5, 2.0, 2.5, 3.0, 5.0\}$ \\
\hline
Likelihood-constrained
& Dual learning rate ($\eta_d$) & $10^{-3}$ \\
& Constraint threshold ($\epsilon$) & $\bar{\ell} \times \beta$ \\
& threshold scaling factors ($\beta$) &
$\{0.70, 0.71, \dots, 0.95\}$ \\
\hline
\end{tabular}
\end{table}

\subsection{Forward KL-Constrained Unlearning}
For forward KL constrained unlearning we unlearn specific samples from a model pretrained on the \texttt{CelebA-HQ} dataset. Our implementation is based on modifying that of~\cite{alberti2025data}, namely, SISS without importance sampling, to incorporate dual updates.

For the unconstrained baseline, we use the same fixed multiplier $\lambda$ for all the unlearning samples. We sweep $\lambda$ over a range from $0.05$ to $0.5$ to observe SSCD values from $1.0$ (almost no unlearning) to around $0.3$ which by visual inspection represents complete unlearning of the memorized sample. 

We compute the SSCD score between two images: the original image sample we wish to unlearn from the model, and the image obtained by adding noise to the unlearning sample up to time step $t = 200$ and then denoising using the model to see how closely it can recover the original sample. This comparison is visually represented in Figure~\ref{fig: forward_images}, Left.

The full list of important hyperparameters is listed in Table~\ref{tab:siss_constrained_hyperparams}.

\begin{table}[h]
\centering
\caption{Hyperparameters for forward KL constrained data unlearning experiments on CelebA-HQ data samples. Used to obtain the results in Figures~\ref{fig: forward_images},~\ref{fig: forward pareto}.}
\label{tab:siss_constrained_hyperparams}
\begin{tabular}{lll}
\hline
\textbf{Category} & \textbf{Hyperparameter} & \textbf{Value} \\
\hline
General
& Dataset & CelebA-HQ ($256 \times 256$) \\
& Pretrained model & \texttt{google/ddpm-celebahq-256}\\
& Diffusion time steps $T$ & 1000 \\
\hline
Training
& Epochs & $200$ \\
& Train batch size & $2$ \\
& Gradient accumulation steps & $16$ \\
& Effective batch size & $32$ \\
& Learning rate ($\eta_p$) & $5 \times 10^{-6}$ \\
& Dual learning rate ($\eta_d$) & $5 \times 10^{-2}$ \\
& Optimizer & Adam \\
\hline
Constraints
& Constraint type & Forward KL divergence \\
& Threshold values & $\{0.0025,\;0.005,\;0.0075,\;0.010,\;0.015,\;0.02,\;0.03\}$ \\
\hline
Unconstrained
&fixed multiplier sweep values& $\{0.01,\;0.015,\;0.02,\;0.025,\;0.030,\;0.035,\;0.04,\;0.045,\;0.05\}$\\
\hline
\end{tabular}
\end{table}

\subsection{Reverse KL-Constrained Unlearning}
\textbf{Dual-Only Algorithm. }For the reverse KL experiments we modify the algorithm slightly. We can skip the primal step in the algorithm discussed in Appendix~\ref{app: algorithm reverse KL}, by observing the fact that we already know from section~\ref{subsec: concept unlearning} that for a given $\lambda$ the optimal primal minimizer is given by
\[
q_{\text{\normalfont u}}^{\dagger} (\cdot;\lambda)
    \; = \;
    \frac{1}{ Z^\dagger_{\text{\normalfont 
    u}}(\lambda)}
    \frac{(q(\cdot))^{1/(1-\one^\top\lambda)}}{\prod_{i\,=\,1}^m(q_{\text{\normalfont u}}^i(\cdot))^{\hat\lambda_i}}.
    \]
Since we have access to the score functions of $q(\cdot)$, and the scores of the individual unlearning distributions $q^i(\cdot)$, we can bypass the primal step and directly do the dual updates in~\eqref{eq: dual update reverse kl 123}.

This dual-only algorithm allows us to find the optimal multipliers $\lambda^\star$ that satisfy the unlearning reverse KL constraints. If needed, the model weights can then be fine-tuned so that the distribution of generated samples fits $q_{\text{\normalfont u}}^{\dagger} (\cdot;\lambda^\star)$.~\citet{khalafi2025composition} propose a similar dual-only algorithm to find the optimal weights for diffusion model composition.

To choose the constraint thresholds, we first estimate the reverse KL between the pretrained model, and each of the individual unlearning concepts. Then we set the threshold as the average of these KL divergences, multiplied by a factor.

For the fixed equal weight, unconstrained baseline, we choose all dual multipliers to be equal, and across runs we sweep the multiplier values so that the sum of the dual multipliers, $\lambda^\top 1$, goes from $0$ to $0.95$. Recall from Section~\ref{subsec: concept unlearning} that to avoid degenerate solutions we need $\lambda^\top 1< 1$.

\begin{table}[h]
\centering
\caption{Hyperparameters for reverse KL constrained unlearning experiments with dual-only algorithm. Used to obtain the results in Figure~\ref{fig: reverseKL_pareto}.}
\label{tab:dual_only_hyperparams}
\begin{tabular}{lll}
\hline
\textbf{Category} & \textbf{Hyperparameter} & \textbf{Value} \\
\hline
General
& Base model & Stable Diffusion v1.4 \\
& Prompt structure & 1 retain prompt + multiple unlearn prompts \\
& Example prompt &
\parbox[t]{0.55\linewidth}{\small
``a photo of a dog, a golden retriever, a german shepherd, a brown dog''} \\
\hline
Training
& Epochs (heuristic $b$ sweep) & $75$ \\
& Training batch size & $8$ \\
& Primal learning rate & $10^{-4}$ \\
& Dual learning rate ($\eta_d$) & $0.02$ \\
\hline
Diffusion / sampling
& Guidance scale & $7.5$ \\
& Inference steps & $10$ \\
\hline
Dual / constraint settings
& KL type & Pathwise \\
& Multiplicative threshold factor &
$\{2,3,4,5,6,7,8,9,10,12,15\}$ \\
\hline
\end{tabular}
\end{table}